\documentclass[twoside,11pt]{article}
\usepackage{jmlr2e}
\usepackage{amsmath}
\usepackage{tikz}
\usepackage{subfigure,epic}
\usepackage{multirow}
\usepackage{array}
\usepackage{booktabs}
\usepackage{color}
\usepackage{verbatim}
\usepackage{algorithmic}
\usepackage[vlined,english,ruled,linesnumbered]{algorithm2e}

\usetikzlibrary{arrows,shapes,chains}

\newtheorem{assumption}{Assumption}
\newtheorem{property}{Property}

\begin{document}

\title{Nystr\"{o}m Regularization for Time Series Forecasting\thanks{Z. Sun and M. Dai contribute  equally and are co-first authors of the paper}}

\author{
\name Zirui Sun \email sunzirui@stu.xjtu.edu.cn\\
\addr Center of Intelligent Decision-Making
        and Machine Learning\\
        School of Management\\
       Xi'an Jiaotong University\\
       Xi'an, China\\
\name Mingwei Dai \email daimw@swufe.edu.cn \\
\addr Center of Statistical Research and School of Statistics\\
     Southwestern University of Finance and Economics\\
       Chengdu, China\\
\name Yao Wang \email yao.s.wang@gmail.com  \\
       \addr Center of Intelligent Decision-Making
        and Machine Learning\\
        School of Management\\
       Xi'an Jiaotong University\\
       Xi'an, China\\
\name Shao-Bo Lin\thanks{Corresponding author}  \email sblin1983@gmail.com \\
       \addr Center of Intelligent Decision-Making
        and Machine Learning\\
        School of Management\\
       Xi'an Jiaotong University\\
       Xi'an, China
}

\editor{???}

\maketitle

\begin{abstract}
This paper focuses on learning rate analysis of Nystr\"{o}m regularization with sequential sub-sampling for $\tau$-mixing time series. Using a recently developed Banach-valued Bernstein inequality for $\tau$-mixing sequences and an integral operator approach based on second-order decomposition, we succeed in deriving almost optimal learning rates of Nystr\"{o}m regularization with sequential sub-sampling for $\tau$-mixing time series. A series of numerical experiments are carried out to verify our theoretical results, showing the  excellent learning performance of Nystr\"{o}m regularization with sequential sub-sampling in learning massive  time series data.
All these results extend the applicable range of  Nystr\"{o}m regularization  from i.i.d. samples to non-i.i.d. sequences.
\end{abstract}


\section{Introduction}
\qquad Time series is one of the most common data types abounding in almost every aspect of human life  \citep{Fu2011}, including clinical medicine, finance data, speech recognition, motion capture data, traffic data, music recognition and video data. Besides the time-dependent nature, time series data in recent years exhibit additional massiveness  and hard-to-model  properties,  making  the existing  models such as autoregressive models, linear dynamical systems and hidden Markov models be no more efficient    \citep{Rakthanmanon2012}. It is thus highly desired to develop scalable learning algorithms of high quality to tackle massive     time series data  that are not generated by simple parametric models.

One of the most productive solutions to tackle hard-to-model  time series data is the nonparametric approach \citep{Meir2000}, which aims to find  the intrinsic nature of time series
without  imposing any structural restrictions on the models. Neural networks \citep{Hagan1997}
and kernel methods  \citep{ShaweTaylor2004} are two popular nonparametric schemes in time series forecasting. In particular, it was shown in \citep{Modha1996,Modha1998,Xu2008,Steinwart2009} that neural networks  and kernel methods  are memory-universal and possess  good generalization performance, provided the time series satisfy the well known  $\alpha$-mixing condition \citep{Doukhan1994}. The problem is, however, that the existing methods require huge computations, especially when the data size is large.

This paper aims to derive a  scalable kernel-based learning algorithm  to tackle massive time series data.  Our study is motivated by three important observations. At first,   the dependent nature of time series data enhances the dependence among columns of the kernel matrix, making its effective rank\footnote{The effective rank in this paper denotes the number of eigenvalues that are larger than a specific threshold.}
much smaller than that of independent and identical (i.i.d.) data (see Figure 2 below for detailed description). Secondly,  numerous studies showed that the  dependence of time series data can be maintained via kernelization, i.e.,   columns of the kernel matrix and input data possess the same mixing property \citep{Bradley2005}. For example, it can be found in \citep{Bradley2005,Sun2021} that the so-called $\alpha$-mixing condition \citep{Rosenblatt1956} is unchanged via kernelization. Finally, Nystr\"{o}m regularization \citep{Williams2000}, a special type of learning with sub-sampling that randomly sketches a few columns from the kernel matrix to build up the estimator in the premise that  such a computation-reduction scheme does not sacrifice the learning performance, has been widely used for i.i.d. data. Nystr\"{o}m regularization \citep{Williams2000} successfully reduces the computational burden of kernel methods without losing the generalization performance  \citep{Rudi2015,Kriukova2017}, provided the number of sketched columns is  larger than the effective rank of the kernel matrix.
Taking   these interesting observations into accounts, we  then  devote to
utilizing Nystr\"{o}m regularization for kernel ridge regression (KRR) \citep{Gittens2016} to yield a novel scalable learning strategy for time series data.

Different from the classical Nystr\"{o}m regularization for i.i.d. data, Nystr\"{o}m regularization for time series  requires  strict orders of selected columns  to reflect the time-dependent nature, making the widely used sub-sampling schemes including the plan Nystr\"{o}m \citep{Williams2000}, leverage score approach \citep{Gittens2016} and random sketching \citep{Yang2017} be no more available.  
Therefore, novel sub-sampling strategies are needed to equip Nystr\"{o}m regularization to tackle time series data.  This raises two challenges: (i) designing an exclusive sub-sampling mechanism available to time series data and (ii)  providing theoretical guarantees for  corresponding Nystr\"{o}m regularization approach.

As kernelization  maintains the mixing property of time series data,
we employ a simple but effective sub-sampling strategy for the first challenge, via selecting  continuous columns in the kernel matrix to guarantee the order. As a result, the selected columns possess similar mixing property as the time series data.
We call such a sub-sampling strategy as Nystr\"{o}m regularization with sequential sub-sampling. There are mainly two advantages of sequential sub-sampling. One is its user-friendly nature, making the sub-sampling be easy to be implemented. The other is that the mixing property of selected columns plays a crucial role in providing theoretical guarantees for the corresponding Nystr\"{o}m regularization.

For the second challenge, it should be noted that
previous approaches for   theoretical guarantees for learning with time series data require quantifying the dependence of time series via  mixing conditions, including  the $\alpha$-mixing \citep{Modha1996}, $\beta$-mixing \citep{Yu1994} and  $\phi$-mixing  \citep{Billingsley1968}. However, learning rates established in \citep{Yu1994,Xu2008,Steinwart2009,Sun2010,Alquier2012,Alquier2013,Hang2017} concerning  the corresponding  mixing data are sub-optimal, since the dependence among data reduces the effective samples and then makes the Bernstein-type inequality established in \citep{Yu1994,Modha1996} for dependent data be a little bit worse than the classical Bernstein inequality for i.i.d. data. Therefore, to provide optimal theoretical guarantees for Nystr\"{o}m regularization with sequential sub-sampling, it is necessary to develop a novel analysis approach such as the integral operator approach in \citep{Sun2010,Sun2021}. Fortunately, for the well known $\tau$-mixing sequences,  \citet{Blanchard2019} have derived almost optimal learning rates for KRR via establishing a novel integral operator approach. Noting further that  $\tau$-mixing is somewhat weaker than $\alpha$-mixing \citep{Dedecker2004}, we borrow the idea from   \citep{Blanchard2019} to derive almost optimal learning rates of Nystr\"{o}m regularization for  $\tau$-mixing time series.

Our main  contributions can be summarized in the following three aspects:

$\bullet$ {\it Methodology:} To tackle massive and hard-to-model time series data, we propose a novel
Nystr\"{o}m regularization with sequential sub-sampling based on kernel methods. The sequential sub-sampling mechanism succeeds in maintaining the mixing property of time series and the Nystr\"{o}m regularization significantly reduces the computational burden of kernel methods.
Meanwhile, compared with i.i.d. data, the dependence nature of time series results in smaller effective rank of kernel matrix and thus requires smaller sub-sampling ratio\footnote{The sub-sampling ratio in this paper means the ratio between the number of selected  columns in kernel matrix and   size of data.} in Nystr\"{o}m regularization.

$\bullet$ {\it Theory:} Utilizing a recently developed Banach-valued Bernstein inequality for $\tau$-mixing sequences \citep{Blanchard2019} and the second-order decomposition approach for integral operator \citep{Guo2017}, we derive almost optimal learning rates of Nystr\"{o}m regularization with sequential sub-sampling for $\tau$-mixing time series. In particular, we show that, with a small sub-sampling ratio, Nystr\"{o}m regularization  with sequential sub-sampling performs the same as KRR, provided the $\tau$-mixing coefficients of time series decay exponentially. This is the first result, to the best of our knowledge, to show almost optimal learning rates for learning non-i.i.d. data with the sub-sampling strategy.

$\bullet$ {\it Experiments:} Our theoretical assertions are verified by numerous toy simulations and two real data experiments including the BITCOIN (BTC) data and  Western Texas Intermediate  (WTI) data.   Our experimental results show that     Nystr\"{o}m regularization with sequential sub-sampling is an effective and efficient approach to reduce the computational burden of KRR without scarifying its excellent learning performance, provided the sampling ratio is larger than a specific value. Furthermore, we find that  Nystr\"{o}m regularization with sequential sub-sampling is a feasible noise-extractor  in the sense that it can quantify the random noise in time series. All these results show that Nystr\"{o}m regularization with sequential sub-sampling is a scalable and feasible strategy  to tackle massive and hard-to-model time series.

The remainder of this paper is organized as follows. In Section 2,
 we introduce some basic properties of time series and $\tau$-mixing sequences, and then propose the
Nystr\"{o}m regularization with sequential sub-sampling for time series. In Section 3, we study  theoretical behaviors of Nystr\"{o}m regularization via presenting its almost optimal learning rates.
  In Section 4, we compare our results with some
related literature and present some discussions. In Section 5,
extensive experimental studies are carried out to verify our
theoretical assertions.  In the last section, we prove our
 main results.

\section{Time Series Forecasting via Nystr\"{o}m Regularization}
\qquad In this section, we introduce  time series forecasting problem   and then propose the   Nystr\"{o}m regularization  with sequential sub-sampling.

\subsection{Time series forecasting}

\qquad We are interested in a standard time series forecasting setting where the learner receives   data of the form $D:=D_n:=\{z_t\}_{t=1}^n=\{(x_t,y_t)\}_{t=1}^n$ with $x_t\in\mathcal X$, $y_t\in\mathcal Y$ and $z_t\in\mathcal Z:=\mathcal X\times\mathcal Y$.  The aim is to learn a function $f_n:\mathcal X\rightarrow\mathcal Y$ such that $f_n(x_{n+1})$ can predict $y_{n+1}$ well.
Without loss of generality, we assume $\mathcal Y=[-M,M]$ and $\mathcal X=[-M,M]^d$. The following are three widely used time series forecasting models.

$\bullet$ {\it Non-parametric auto-regression:}
Let $d\in\mathbb N$ be the memory size. Assume $\mathcal X=\mathcal Y^d$ and  there is an $f_0:\mathcal X \rightarrow \mathcal Y$ such that
\begin{equation}\label{AR-model}
      x_t=f_0(x_{t-1},\dots,x_{t-d})+\varepsilon_t,
\end{equation}
where $\{\varepsilon_t\}_{t=1}^n$ are independent of $x_0$.

$\bullet$ {\it Non-parametric ARX:}
Let $d\in\mathbb N$ be the memory size. Assume $\mathcal X=\mathcal Y^{d+d'}$ for some $d'\in\mathbb N$.
Let $\{\xi_t\}_{t=1}^n$ with $\xi_t\in\mathcal Y$ be a set of auxiliary variants.
Assume  that there is an $f_0:\mathcal X\rightarrow\mathcal Y$ such that  $x_t$, $t=1,\dots,n$ are generated via
\begin{equation}\label{ARKT-model}
      x_t=f_0(x_{t-1},\dots,x_{t-d},\xi_t,\dots,\xi_{t-d'})+\varepsilon_t,
\end{equation}
where $\{\varepsilon_t\}_{t=1}^n$ are independent of $x_0$.
%

$\bullet$ {\it Nonlinear processes:} Let $d\in\mathbb N$ be the memory size. Assume $\mathcal X=\mathcal Y^d$ and there is an $f_0:\mathcal X \rightarrow \mathcal Y$ such that
$$
      x_t=f_0(x_{t-1},\dots,x_{t-d};\varepsilon_t),
$$
where $\{\varepsilon_t\}_{t=1}^n$ are independent of $x_0$.

\begin{figure}\label{Fig:timeseries}
\centering
\includegraphics[scale=0.5]{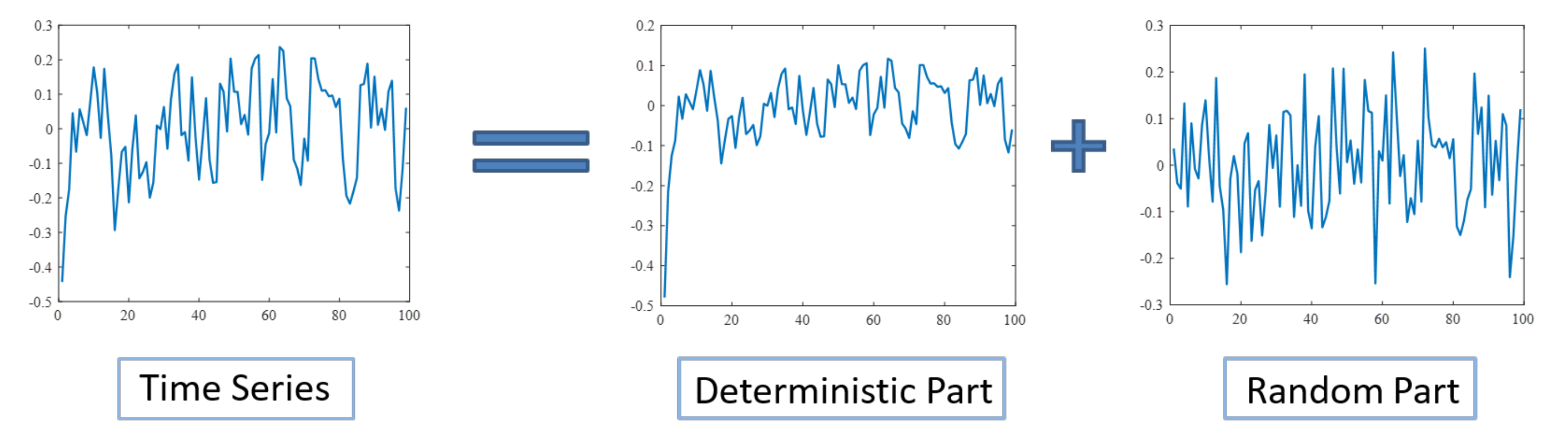}
\caption{The decomposition of time series.}\label{decomp}
\end{figure}

Due to the existence of noise  $\{\varepsilon_t\}_{t=1}^n$, it is not a good choice  to formulate time series forecasting problem into the standard regression setting for i.i.d. data \citep{Gyorfi2002}. Instead, as shown in Figure  1, a general time series forecasting problem can be divided into searching a deterministic relation between input and output and describing a random part that is mainly caused by the random noise. Taking non-parametric auto-regression regression for example, the deterministic part refers to find a good estimate, $f_D$, of $f_0$ in (\ref{AR-model}) and the random part concerns the distribution of the noise $\{\varepsilon_t\}_{t=1}^n$. Once  $f_D$  is good enough, then it can be regarded as a noise-extractor in the sense that $y_t-f_0(x_t)$ is near to $\varepsilon_t$ and then the distribution of noise of time series data can be revealed.

%
%
%

\subsection{$\tau$-mixing sequences}
\qquad It is well known that some restriction on the dependence is necessary to establish satisfactory  generalization error bounds for time series data. Strong mixing condition \citep{Modha1996} is one of most popular restriction  to quantify the dependence among time series and  is much weaker than the so-called $\beta$-mixing condition \citep{Yu1994} and $\phi$-mixing condition \citep{Billingsley1968}.
For two $\sigma$-fields $\mathcal J$ and $\mathcal K$, define the $\alpha$-mixing (or strong mixing) coefficient as
\begin{equation}\label{alpha}
     \alpha(\mathcal J,\mathcal K):=\sup_{A\in\mathcal J, B\in\mathcal K}|P(A\cap B)-P(A)P(B)|.
\end{equation}
Denote by $\mathcal M_{i,j}$ the $\sigma$-filed generated by random variables $z_{i:j}:=(z_i,z_{i+1},\dots,z_j)$.
The strong  mixing condition is defined  as follows.
\begin{definition}\label{Def:alpha}
A set of random sequence $\{z_i\}_{i=1}^\infty$ is said to satisfy a strong  mixing condition (or $\alpha$-mixing condition) if
\begin{equation}\label{def-alpha}
       \alpha_j:=\sup_{k\geq 1}\alpha(\mathcal M_{1,k},\mathcal M_{k+j,\infty})\rightarrow0, \qquad\mbox{as}\ j\rightarrow \infty.
\end{equation}
\end{definition}

 We refer the readers to \citep{Doukhan1994} for more details and  examples for $\alpha$-mixing sequences.
Unfortunately,  the $\alpha$-mixing condition presented in Definition \ref{Def:alpha} is still a little bit strong,   excluding
some simple Markov chains and causal linear processes.

$\bullet$ Markov chains with Bernoulli distribution: Let
$$
     x_t=\frac12(x_{t-1}+\varepsilon_t)
$$
where $\{\varepsilon_t\}_{t=1}^n$ are i.i.d. drawn with the Bernoulli distribution $\mathcal B(1/2)$ and are independent of $x_0$. It was proved in \citep{Andrews1984} that $\alpha_j=1/2$ for any $j$.

$\bullet$ Causal linear process: Let $(\xi_j)_{j\in\mathbb Z}$ be a sequence of i.i.d. random variables with values in $\mathbb R$. Define the time series $\{x_t\}_{t\in\mathbb N}$ with
$$
       x_t=\sum_{j=0}^\infty a_j\xi_{t-j},
$$
where  $a_j=2^{-j-1}$. If $\xi_0$ is drawn from $\mathcal B(1/2)$, then it can be found in \citep[P.871]{Dedecker2004} that $\alpha_j=1/4$ for   any $j$.

Noting this dilemma,  \citep{Dedecker2004,Dedecker2005} proposed a slightly weaker $\tau$-mixing condition. Let $\mathcal C_{Lip}$ be the set of bounded Lipschitz functions over $\mathcal X$. Consider
$$
     C_{Lip}(f):=\|f\|_{Lip(\mathcal X)}=\sup\left\{\frac{|f(x)-f(x')|}{\|x-x'\|}\big|x,x'\in\mathcal X,x\neq x'\right\}.
$$
It is easy to see that $C_{Lip}$ is a semi-norm.
We consider a norm of $\mathcal C_{Lip}$ of the form:
$$
     \|g\|_{\mathcal C_{Lip}}:=\|g\|_\infty+C_{Lip}(g),
$$
where  $\|\cdot\|_\infty$ is the sup-norm on $\mathcal C_{Lip}$. Let $\mathcal C_1$ be the ``semi-ball'' of functions $g\in\mathcal C_{Lip}$ such that $C_{Lip}(g)\leq 1$. Let $\mathcal M_i$ be the $\sigma$-algebra generated by $z_1,\dots,z_t$.

\begin{definition}\label{Def:C-mixing}
The $\tau$-mixing coefficients are defined by
\begin{eqnarray}\label{C-mix-coeff}
   \tau_j&:=&\sup\{ E(\eta g(z_{i+j}))- E(\eta) E(g(z_{i+j}))|i\in\mathbb N, \nonumber\\
   && \eta\ \mbox{is $\mathcal M_i$-measurable and}\ \|\eta\|_1\leq 1, g\in\mathcal C_1\}.
\end{eqnarray}
\end{definition}
 If  there are some constants $b_0>0,c_0\geq0,\gamma_0 >0$ such that
\begin{equation}\label{def-Galpha}
           \tau_j\leq c_0\exp( {-(b_0j)}^{\gamma_0}),\qquad \forall \ j\geq 1,
\end{equation}
then   $\{z_t\}_{t=1}^\infty$ is said    to be geometrical $\tau$-mixing.
 If  there are some constants $c_1>0, \gamma_1 >0$ such that
\begin{equation}\label{def-Galpha11}
           \tau_j\leq c_1j^{-\gamma_1},\qquad \forall \ j\geq 1,
\end{equation}
then   $\{z_t\}_{t=1}^\infty$ is said    to be algebraic $\tau$-mixing.


It can be found in \citep{Dedecker2004} that the above Markov chains with Bernoulli distribution and causal linear process are  geometrical $\tau$-mixing, showing that $\tau$-mixing   is essentially different from $\alpha$-mixing.
Some basic properties of $\tau$-mixing sequences were derived in \citep{Dedecker2004,Dedecker2005}, among which the following are important for our analysis.

\begin{property}\label{Property:1}
 If $\{z_t\}_{t=1}^\infty$ is $\tau$-mixing with coefficient $\tau_j$, then for arbitrary $i,k\in\mathbb N$,  $\{z_{t}\}_{t=i}^{i+k}$  is $\tau$-mixing with coefficient $\tau_j$.
\end{property}

\begin{property}\label{Property:2}
 If $\{z_t\}_{t=1}^\infty$ is $\tau$-mixing with coefficient $\tau_j$ and $h\in \mathcal C_{Lip}$ with Liptchiz constant $C_h$, then $\{h(z_t)\}_{t=1}^\infty$ is $\tau$-mixing with coefficient $C_h\tau_j$.
\end{property}

\begin{property}\label{Property:3}
 If $\{z_t\}_{t=1}^\infty$ is $\alpha$-mixing, then $\{z_t\}_{t=1}^\infty$ is $\tau$-mixing.
\end{property}

Property \ref{Property:1} and Property \ref{Property:2} can be deduced from the definition directly and play crucial roles in developing Nystr\"{o}m Regularization for $\tau$-mixing series. Property 3 that can be found in \citep[Lemma 7]{Dedecker2004} shows that the $\tau$-mixing condition is   weaker than the $\alpha$-mixing condition.
We highlight that Property \ref{Property:3} only  implies $\tau_j\leq g(\alpha_j)$ for some monotonously increasing function $g$ rather than  $\tau_j\leq \alpha_j$.

\subsection{Nystr\"{o}m regularization}
\qquad Let $D:=D_n:=\{z_t\}_{t=1}^n=\{(x_t,y_t)\}_{t=1}^n$ with $x_t\in\mathcal X=[-M,M]^d$ and $y_t\in\mathcal Y=[-M,M]$ and $K(\cdot,\cdot)$ be  a Mercer kernel and $({\mathcal H}_K, \|\cdot\|_K)$ be
the corresponding reproduced kernel Hilbert space (RKHS).  Kernel ridge regression (KRR) \citep{Evgeniou2000}, given a regularization parameter $\lambda>0$,   is defined by
\begin{equation}\label{KRR}
    f_{D,\lambda} =\arg\min_{f\in \mathcal{H}_{K}}
    \left\{\frac{1}{n}\sum_{t=1}^n(f(x_t)-y_t)^2+\lambda\|f\|^2_{K}\right\}.
\end{equation}
It is easy to check that the complexities in storage and training of KRR are $\mathcal O(n^2)$ and $\mathcal O(n^3)$, respectively.  As a result,  KRR is difficult to tackle massive time series, although it is one of the most popular learning algorithms in the past two decades. Sub-sampling \citep{Gittens2016} is a preferable way to reduce the computational burden of KRR.

For any subset  {$D_j:=D_{j,m}:=\{(\tilde{x}_i, \tilde{y}_i)\}_{i=1}^m$} of $D$, define
\begin{equation}\label{hypothesis-space-sub}
      \mathcal H_{D_j}:=\left\{\sum_{i=1}^ma_iK_{\tilde{x}_i}:a_i\in\mathbb R\right\},
\end{equation}
where  $K_x:=K(x,\cdot)$.  Nystr\"{o}m regularization with sub-samples $D_j$ is  then defined  by
\begin{equation}\label{Nystrom}
       f_{D,D_j,\lambda}:=\arg\min_{f\in\mathcal
       H_{D_j}}\frac1{n}\sum_{t=1}^n(f(x_t)-y_t)^2+\lambda\|f\|_K^2.
\end{equation}
Direct computation \citep{Rudi2015} yields
\begin{equation}\label{Analytic-solution}
   f_{D,D_j,\lambda}(\cdot)=\sum_{i=1}^m \alpha_iK_{\tilde{x}_i}(\cdot),
\end{equation}
where
\begin{equation}\label{Analytic-solution1}
   \alpha=(\alpha_1,\dots,\alpha_m)^T= (\mathbb K_{nm}^T\mathbb K_{nm}+\lambda|D|\mathbb K_{mm})^\dagger \mathbb K_{nm}^Ty_D,
\end{equation}
$\mathbb A^\dagger$ and $\mathbb A^T$ denote  the Moore-Penrose pseudo-inverse and transpose  of a matrix $\mathbb A$ respectively, and $(\mathbb K_{n,m})_{t,i}=K(x_t,\tilde{x}_i)$, $(\mathbb K_{mm})_{k,i}=K(\tilde{x}_k,\tilde{x}_i)$ and $y_D=(y_1,\dots,y_{|D|})^T$. In this way, it requires $\mathcal O(nm)$ and $\mathcal O(nm^2)$ complexities in storage and training respectively to derive a Nystr\"{o}m regularization  estimator. If the sub-sampling ratio, i.e., $m/n$ is small, then Nystr\"{o}m regularization  significantly reduces the computational burden of KRR.

\begin{figure}
\centering
\subfigure{\includegraphics[scale=0.4]{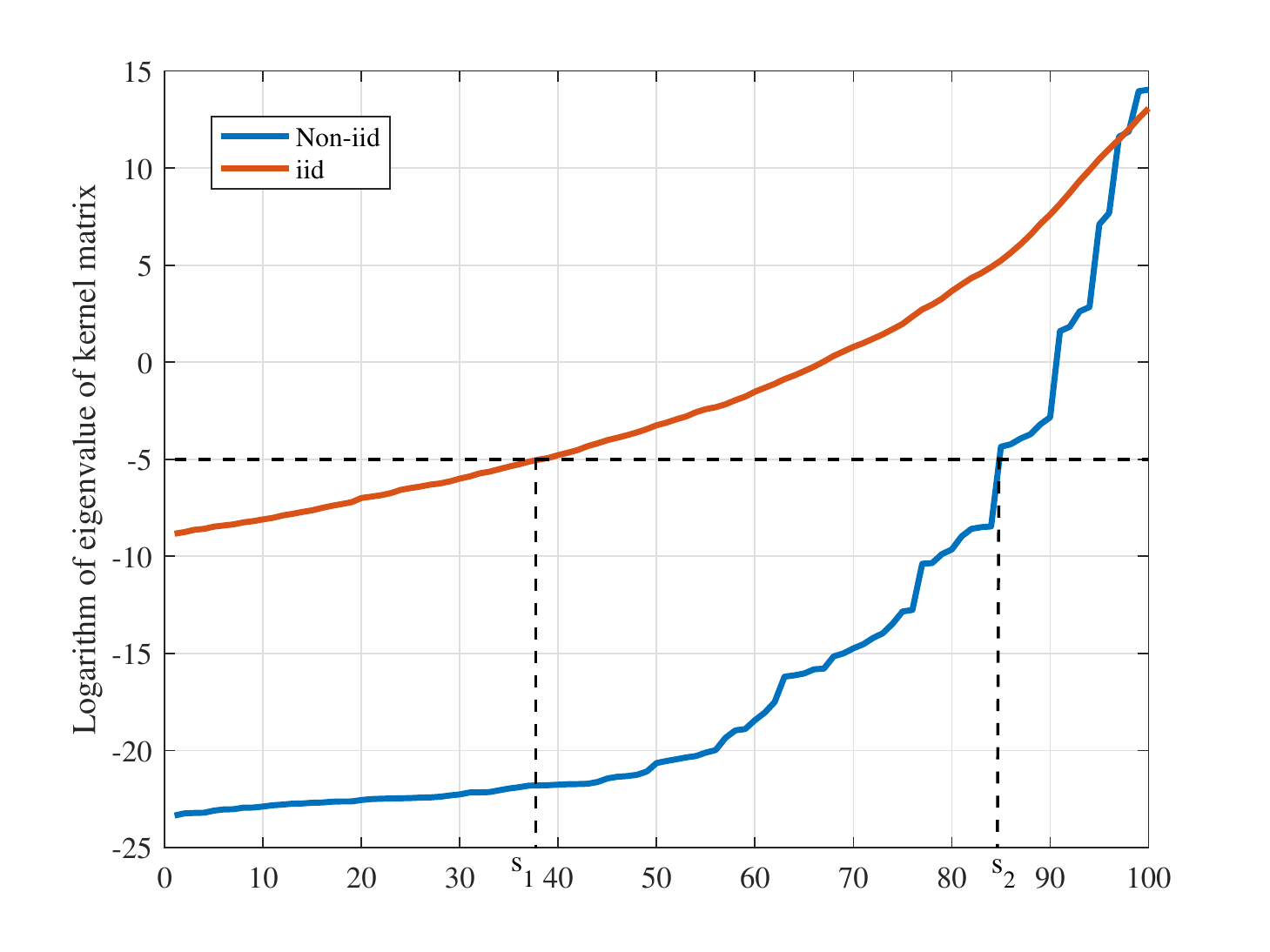}}
\subfigure{\includegraphics[scale=0.4]{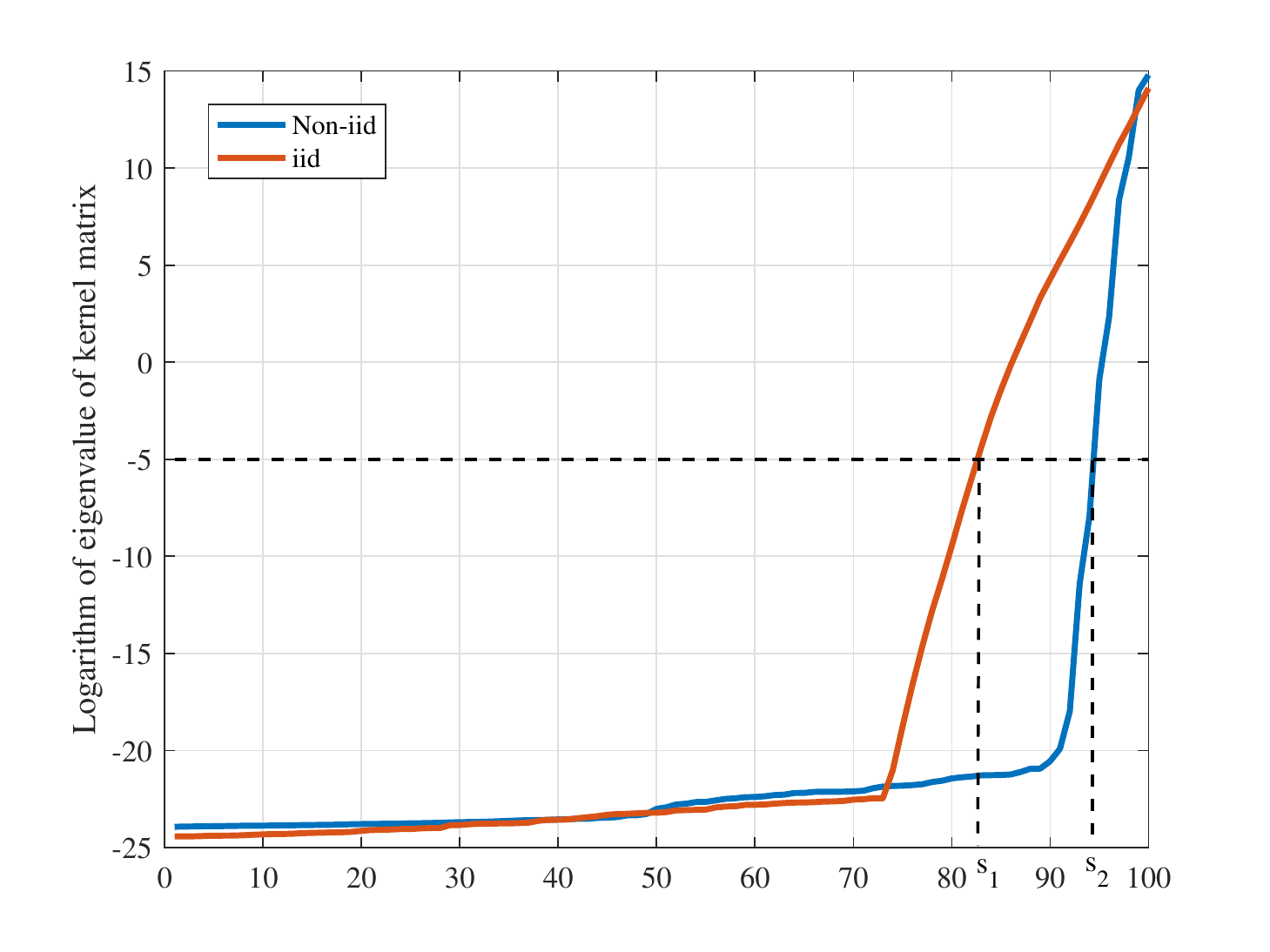}}
\caption{Eigenvalues of kernel matrix. Data are generated with $ x_{t} = 0.5\sin(x_{t-1})+\varepsilon_t$ (the blue line) and $ f(x) = 0.5\sin(x)+\varepsilon$ (the red line) where $\varepsilon_t,\varepsilon \sim \mathcal B(1/2)$. The number of samples is 3000 and we draw the largest 100 eigenvalues. The left figure and right figure are the eigenvalues under different kernel functions, which are Wendland kernel (\ref{kernel}) below and Gaussian kernel with $\sigma = 0.5$ respectively.}
\end{figure}\label{Fig:eigenvalue}

Though theoretical behaviors of  Nystr\"{o}m regularization have been rigorously verified in  \citep{Rudi2015}, it remains open whether Nystr\"{o}m regularization is available to time series. Noting further that the basic idea of Nystr\"{o}m regularization \citep{Williams2000} is that the effective rank
of kernel matrix is much smaller than the size of data and sub-sampling aims to find principal component of the kernel matrix, we show the applicability of  sub-sampling for time series by comparing eigenvalues of the kernel matrices generated from  time series and  i.i.d. data respectively. In order to avoid the influence of the   kernel functions on the calculation of eigenvalues, we choose two different types of kernel functions. As shown in Figure~2, we see that $s_1$ is always less than $s_2$, implying that   it is easier to retain the main information  after sub-sampling   time series data.
That is,  time series data allow      lower sub-sampling ratio to  maintain the spectrum  information of kernel matrix than i.i.d. data. Therefore, we may claim that time series data may be more suitable than i.i.d. data for sub-sampling methods.

\subsection{Nystr\"{o}m regularization with sequential sub-sampling}
\qquad Different from i.i.d. data, time series  exhibit additional difficulty in designing the sub-sampling strategy due to their non-i.i.d. nature. In particular, it is unknown whether the selected columns satisfy   certain mixing conditions. For this purpose, we propose a concept of  sequential sub-sampling to guarantee the $\tau$ mixing property of sub-sampling.
\begin{definition}\label{Def:seq-subsampling}
Let $m\in\mathbb N$ and $m\leq n$. Randomly select $j\in[1,n-m+1]$ according to the uniform distribution. The set $D_j^\ast:=D_{j,m}:=\{z_j,\dots,z_{j+m-1}\}:=\{\tilde{z}_i\}_{i=1}^m$ is then defined to be a sequential sub-sampling of size $m$  for $D$.
\end{definition}

Due to Property \ref{Property:1},  the  $\tau$-mixing property of $D$ implies the $\tau$ mixing property of $D_j^\ast$. To facilitate the analysis, we also need the following assumption on the kernel.

\begin{assumption}\label{Assumption:kernel}
Assume $\kappa:=\sup_{x\in\mathcal X}\sqrt{K(x,x)}\leq 1$ and there exists a $\mathcal K>0$ such that $\max\left\{\left|\frac{\partial K(x,x')}{\partial x}\right|,\left|\frac{\partial K(x,x')}{\partial x'}\right|\right\}\leq \mathcal K.$
\end{assumption}

Assumption \ref{Assumption:kernel} is mild. In particular,  the boundedness assumption  $\kappa\leq 1$ is satisfied for any continuous kernels with  scaling. The restriction on the partial derivatives can also be satisfied for any smooth kernels.  It is easy to check that almost all widely used kernels such as the Gaussian kernel, Wendland kernel, Sobolev kernel and multi-quadratic  kernel \citep{Steinwart2008} satisfy Assumption \ref{Assumption:kernel}. Based on Assumption \ref{Assumption:kernel} and Property \ref{Property:2}, we  obtain the following property directly \citep{Blanchard2019}.

\begin{property}\label{Property:mixing}
Under Assumption \ref{Assumption:kernel}, if $\{z_t\}_{t=1}^n$ is $\tau$-mixing and  $\{\tilde{z}_i\}_{i=1}^m$ is a sequential sub-sampling of $\{z_t\}_{t=1}^n$, then $\{K_{\tilde{z}_i}\}_{i=1}^m$ is Banach-valued  $\tau$-mixing.
\end{property}

Property \ref{Property:mixing} plays an important role in guaranteeing the feasibility of sequential sub-sampling. With the proposed sequential sub-sampling scheme, we can develop an exclusive Nystr\"{o}m regularization for time series by using $D^*_j$ in Definition \ref{Def:seq-subsampling} to take place of $D_j$ in (\ref{Nystrom}). The framework of Nystr\"{o}m regularization with sequential sub-sampling is given in Algorithm~\ref{alg:seq sub}.

%
%

\begin{algorithm}\label{alg:seq sub}
\caption{Nystr\"{o}m regularization with sequential sub-sampling}
\LinesNumbered
\KwIn{Given $D:=\{z_t\}_{t=1}^n=\{(x_t,y_t)\}_{t=1}^n$, Kernel $K(\cdot,\cdot)$, regularization parameter $\lambda$, sub-sampling number $m$.}
\KwOut{ $\widehat{f}_{D,D_j^\ast,\lambda}(x)$.}
// Sequential sub-sampling\\
${D_j^\ast:=\{z_j,\dots,z_{j+m-1}\}:=\{\tilde{z}_i\}_{i=1}^m} \leftarrow$\textbf{Sub-sampling} ($m$, $\{z_i\}_{i=1}^n$);\label{sam}\\
// Calculate kernel $K_{nm}$, $K_{mm}$\\
$(K_{nm})_{ij}\leftarrow K(x_i,\widetilde{x}_j)$ for all $i=\{1,\cdots, n\}, j = \{1,\cdots,m\}$;\label{knm}\\
$(K_{mm})_{ij}\leftarrow K(\widetilde{x}_i,\widetilde{x}_j)$ for all $i=\{1,\cdots, m\}, j = \{1,\cdots,m\}$;\label{kmm}\\
// Calculate $\alpha$\\
$\alpha \leftarrow (K_{nm}^{\intercal}K_{nm}+\lambda \cdot n K_{mm} )^{\dagger}K_{nm}^{\intercal}y$;\label{alpha}\\
\Return $\widehat{f}_{D,D_j^\ast,\lambda}(x)\leftarrow \sum_{i=1}^m \alpha_i \cdot K(\widetilde{x}_i,x)$.
\end{algorithm}

As shown in Algorithm \ref{alg:seq sub}, there are two parameters, $m$ and $\lambda$, to be tuned. It should be mentioned that $\lambda$ is a model parameter that balances the bias and variance of the Nystr\"{o}m regularization, while $m$ is an algorithmic parameter that reflects the trade-off between computation burden and prediction accuracy.  Our analysis below will show that for appropriately selected $\lambda$, the learning performance of Nystr\"{o}m regularization is not sensitive to $m$, provided $m$ is not extremely small. This shows the feasibility and efficiency of utilizing sub-sampling to tackle  time series data. Practically, to reduce the computational burden, it is preferable to set $m=\sqrt{n}$ and then choose $\lambda$ according to the well-known cross-validation. Besides these two explicit parameters, there are also two hidden parameters in  Algorithm \ref{alg:seq sub},  the sub-sampling location ($j$ in Definition \ref{Def:seq-subsampling}) and  sub-sampling interval (the margin between two selected column which is $1$ in Definition \ref{Def:seq-subsampling})   in Algorithm \ref{alg:seq sub}. Our theoretical analysis and numerical results below will  show that different sub-sampling strategies do not affect the learning performance of Nystr\"{o}m regularization with sequential sub-sampling very much, provided the $\tau$-mixing property of the selected columns is guaranteed.

\section{Theoretical Behaviors}\label{Sec.Mainresult}
\qquad In this section, we present our main results on analyzing learning performances of  Nystr\"{o}m regularization with sequential sub-sampling for $\tau$-mixing time series.

\subsection{ Setup and Assumptions}
\qquad Our analysis is carried out in a standard random setting \citep{Alquier2012,Alquier2013}, where
$D=\{z_t\}_{t=1}^n$ is assumed to be drawn according to a  joint distribution
$$
   \rho_{1:n}:=\rho_1(\cdot)\times\rho_2(\cdot|z_1)\times\dots\times\rho_n(\cdot|z_{1:n-1})
$$
and $z_{1:k}=\{z_t\}_{t=1}^k$. Our first assumption in this section is on the distribution $\rho_{1:n}$.

\begin{assumption}\label{Assumption:marginal distributions}
  $\rho_1(\cdot)=\rho_2(\cdot|z_1)=\cdots=\rho_n(\cdot|z_{1:n-1})$, that is, the marginal distribution $\rho_t$ is independent of $t$.
\end{assumption}

It should be mentioned that  Assumption \ref{Assumption:marginal distributions} can be regarded as  the same distribution restriction of $\rho_{1:n}$ and  is essentially weaker than the strict stationarity \citep{Alquier2013}. In particular, it can be found in \citep{Sun2021} that there exist a joint distribution  $\rho^*_{1:n}$ which satisfies Assumption \ref{Assumption:marginal distributions} but is not stationary.  With the help of Assumption \ref{Assumption:marginal distributions}, we denote
$$
       \rho=\rho_1(\cdot)=\rho_2(\cdot|z_1)=\cdots=\rho_n(\cdot|z_{1:n-1}).
$$
Noting further $ z_t=(x_t,y_t)$, we can write $\rho=\rho_X\times\rho(y|x)$. Our aim is then to find a function $f$ to minimize the generalization error $\int_{Z}(f(x)-y)^2d\rho$, which is minimized by the well-known regression function
$$
         f_\rho(x)=\int_{\mathcal Y} y d\rho(y|x), \qquad x\in\mathcal X.
$$
Let $L^2_{\rho_{_X}}$ be the Hilbert space of $\rho_X$ square
integrable functions on $\mathcal X$, with norm denoted by
$\|\cdot\|_\rho$.
Our purpose is then to bound
$$
   \int_{Z}(f(x)-y)^2d\rho -  \int_{Z}(f_\rho(x)-y)^2d\rho=\|f-f_\rho\|_\rho^2,
$$
which makes the time series forecasting problem be similar as the standard least-squares regression setting in \citep{Gyorfi2002}. The only difference is that the independent assumption of samples is replaced by the following $\tau$-mixing assumption.

\begin{assumption}\label{Assumption:mixing}
 $D=\{z_t\}_{t=1}^n$ is a $\tau$-mixing sequence with mixing coefficient $\tau_j$.
\end{assumption}

The mixing condition is a standard assumption to describe the dependence of time series \citep{Doukhan1994,Bradley2005,Alquier2012,Alquier2013}, among which $\beta$-mixing \citep{Yu1994} and $\alpha$-mixing \citep{Modha1996} are widely used. From Property \ref{Property:3}, we obtain that $\tau$-mixing is  weaker than $\alpha$-mixing and therefore $\beta$-mixing, which implies that our results also hold for $\alpha$-mixing time series with a slight change of $\alpha$-mixing or $\beta$-mixing coefficients.  In this way, there are numerous time series  satisfying the $\tau$-mixing condition, including  the causal linear processes, functional autoregressive processes, and Markov kernel associated to expanding maps \citep{Dedecker2004,Dedecker2005}.

Our next assumption  is to quantify the regularity of $f_\rho$. For this purpose, we introduce the well known integral operator associated with the Mercer kernel $K$. Define
$L_K$ on ${\mathcal H}_K$  (or $L_{\rho_X}^2$) by
$$
         L_K(f) =\int_{\mathcal X} K_x f(x)d\rho_X, \qquad f\in {\mathcal
          H}_K \quad (\mbox{or}\ f\in L_{\rho_X}^2).
$$
It is easy to check that $L_K$ is a compact and positive operator.  Denote further
 $L_K^r$ by the  $r$-th power of $L_K: L_{\rho_X}^2 \to
L_{\rho_X}^2$. We can get the following assumption.

\begin{assumption}\label{Assumption:regularity}
There exists an $r>0$ such that
\begin{equation}\label{regularitycondition}
         f_\rho=L_K^r (h_\rho),~~{\rm for~some}~  h_\rho\in L_{\rho_X}^2.
\end{equation}
\end{assumption}

Let $\{(\sigma_\ell,\phi_\ell)\}_{\ell=1}^\infty$ be the normalized eigen-pairs of $L_{K}$ with
$\sigma_1\geq\sigma_2\geq\dots\geq0$.
The Mercer expansion \citep{Aronszajn1950} shows
$$
    K(x,x')=\sum_{\ell=1}^\infty \phi_\ell(x)\phi_\ell(x')
           =\sum_{\ell=1}^\infty \sigma_\ell\frac{\phi_\ell(x)}{\sqrt{\sigma_\ell}}\frac{\phi_\ell(x')}{\sqrt{\sigma_\ell}}
$$
Then
 (\ref{regularitycondition}) is equivalent to
\begin{equation}\label{Regularity-1111}
      f_\rho=\mathcal L_K^rh_\rho=\sum_{\ell=1}^\infty \sigma^{r-1/2}_\ell\langle  h_\rho,\phi_\ell/\sqrt{\sigma_\ell}\rangle_\rho  \phi_\ell.
\end{equation}
That is,  the parameter $r$  in Assumption \ref{Assumption:regularity} determines the smoothness of the target functions. In particular, (\ref{regularitycondition}) with $r=1/2$ implies $f_\rho\in\mathcal H_K$ and  (\ref{regularitycondition}) with $r=0$ yields $f_\rho\in L_{\rho_X}^2$.
 Generally speaking, the larger value of $r$ is, the smoother the function $f_\rho$ is.

Our final assumption is to quantify the capacity of the RKHS $\mathcal H_K$ via the
  effective dimension $\mathcal{N}(\lambda)$ \citep{Zhang2005},  which is
defined to be the trace of the operator $( L_K+\lambda I)^{-1}L_K,$
that is,
$$
        \mathcal{N}(\lambda)={\rm Tr}((\lambda I+L_K)^{-1}L_K),  \qquad \lambda>0.
$$
To obtain explicit learning rates for algorithm, we give an assumption on the decaying rate of the effective dimension as follow.
\begin{assumption}\label{Assumption:capacity}
 There exists some $s\in(0,1]$ such that
\begin{equation}\label{assumption on effect}
      \mathcal N(\lambda)\leq C_0\lambda^{-s},
\end{equation}
where $C_0\geq 1$ is  a constant independent of $\lambda$.
\end{assumption}

Condition (\ref{assumption on effect}) with $s=1$ is always satisfied by taking $C_0=\mbox{Tr}(L_K)\leq\kappa^2$.
For $0<s<1$, let
\[
        L_K=\sum_{\ell=1}^\infty \lambda_\ell\langle\cdot,\phi_{\ell}\rangle_K \phi_{\ell}
\]
be the spectral decomposition. If $\sigma_\ell\leq c_0n^{-1/s}$
for some   $c_0\geq 1$, then
\begin{eqnarray*}
       \mathcal N(\lambda)
       &=&
       \sum_{\ell=1}^\infty\frac{\sigma_\ell}{\lambda+\sigma_\ell}
       \leq
       \sum_{\ell=1}^\infty\frac{c_0\ell^{-1/s}}{\lambda+c_0\ell^{-1/s}}
       =
       \sum_{\ell=1}^\infty\frac{c_0}{c_0+\lambda\ell^{1/s}}\\
       &\leq&
     \int_0^\infty\frac{c_0}{c_0+\lambda t^{1/s} }dt
       =
       \mathcal O(\lambda^{-s}),
\end{eqnarray*}
which implies that Assumption \ref{Assumption:capacity} is more general than the widely used eigenvalue decaying assumption in the literature \citep{Caponnetto2007,Zhang2015}.

In summary, there are totally five assumptions in our analysis, among which Assumptions \ref{Assumption:kernel} and \ref{Assumption:capacity} concern  properties of the kernel $K$, Assumptions \ref{Assumption:marginal distributions} and \ref{Assumption:regularity} refer to   properties of the  joint distribution $\rho$ and Assumption \ref{Assumption:mixing} describes the mixing property of time series. There are numerous time series satisfying the above assumptions. Taking the non-parametric auto-regression for example. Assume
$$
      z_t=f_0(z_{t-1},\dots,z_{t-d})+\varepsilon_t,
$$
where $\rho_1(\cdot)=\rho_2(\cdot|z_1)=\cdots=\rho_n(\cdot|z_{1:n-1})$ are the uniform distribution,
$\varepsilon_t$ is independent of $z_0$ and $ E[\varepsilon_t]=0$, and $f_0\in \mathcal H_K$ with $K$ being a Sobolev kernel. Then it can be found in \cite[p.102]{Doukhan1994} (see also \cite[Proposition 1]{Chen1998}) that Assumptions 1-5 hold
 under certain additional restrictions on $\{\varepsilon_t\}$ and  $f_0$.

\subsection{Learning rate analysis}

\qquad By the aid of above assumptions, we are in a position to present our main results.
Our first main result is the Nystr\"{o}m regularization with sequential sub-sampling given in Algorithm \ref{alg:seq sub} for geometrical $\tau$-mixing time series.

\begin{theorem}\label{Theorem:error-for-exp}
Let $0<\delta\leq 1/2$ and $D_j^*$ be an arbitrary sequential sub-sampling of size $m$ for $D$.
Under Assumptions 1-5 with $\frac12\leq r\leq1$ and $0<s\leq 1$,  if (\ref{def-Galpha}) holds, $\lambda\sim\left(\frac{n}{ (\log n)^{1/\gamma_0}}\right)^{1/(2r+s)}$
  and
\begin{equation}\label{bound-on-m}
      m\geq n^\frac{s+1}{2r+s}(\log m)^{1/\gamma_0}(\log n)^{-\frac{s+1}{(2r+s)\gamma_0}},
\end{equation}
 then with confidence $1-\delta$, there holds
\begin{eqnarray}\label{Error-analysis-2-exp}
    \|f_{D,D_j^*,\lambda}-f_\rho\|_\rho
    \leq
     C^* n^{-\frac{r}{2r+s} }( {1+\log (n)})^{\frac{r}{(2r+s)\gamma_0}} {\log^4\frac2\delta},\qquad \forall j\in[1,n-m+1],
\end{eqnarray}
 where
$C^*$ is a constant independent of $m,n,j$ or $\delta$.
\end{theorem}

If the samples in $D$ are i.i.d. drawn, the optimal learning rates for KRR $f_{D,\lambda}$, defined by (\ref{KRR}),  have been established in \citep{Caponnetto2007} in the sense that there exists a distribution $\rho^*$ satisfying Assumptions \ref{Assumption:kernel}, \ref{Assumption:regularity} and \ref{Assumption:capacity} such that with high probability,
$$
       \|f_{\rho^*} -f_{D,\lambda}\|_{\rho^*}\geq C_1^*n^{-\frac{r}{2r+s}}
$$
for some constant $C_1^*>0$. Noting that i.i.d. samples always satisfy Assumptions \ref{Assumption:marginal distributions} and \ref{Assumption:mixing}, the
derived error estimate in (\ref{bound-on-m}) is  optimal up to a logarithmic factor. Therefore, the derived learning rate cannot be essentially improved further.  Furthermore, it should be mentioned that  Theorem \ref{Theorem:error-for-exp} is
an   extension of  \citep[Theorem 1]{Rudi2015}, where optimal learning rates of plain Nystr\"{o}m  regularization for i.i.d. samples are deduced, since Theorem \ref{Theorem:error-for-exp} with $\gamma_0\rightarrow\infty$ coincides with \citep[Theorem 1]{Rudi2015}. In particular, setting $\mathcal M_{1:n}$ be the set of distribution $\rho_{1:n}=\{\rho_1,\dots,\rho_n\}$ that satisfying Assumptions 1-5, we can get the following corollary directly.

\begin{corollary}\label{Corollary:optimal-exp}
Let $D_j^*$ be an arbitrary sequential sub-sampling of size $m$ for $D$. If (\ref{def-Galpha}), (\ref{bound-on-m}) hold, and $\lambda\sim\left(\frac{n}{ (\log n)^{1/\gamma_0}}\right)^{1/(2r+s)}$, then for any $j\in[1,n-m+1]$, there holds
$$
    C^*_1n^{-\frac{2r}{2r+s} }
    \leq  \sup_{\rho_{1:n}\in\mathcal M_{1:n}} E[\|f_{D,D_j^*,\lambda}-f_\rho\|_{\rho_n}^2]
    \leq
     C^*_2 n^{-\frac{2r}{2r+s} }( {1+\log (n)})^{\frac{2r}{(2r+s)\gamma_0}},
$$
where $C^*_2$ is  a constant independent of $m,n,j$.
\end{corollary}

There are two tunable  parameters, $m$ and $\lambda$, in Nystr\"{o}m regularization. Theoretically speaking  \citep{Rudi2015},   the regularization parameter $\lambda$ is introduced to balance  the bias and variance,  while the sub-sampling size $m$ is involved to balance computational cost and generalization error. The problem is, however, that such an ideal assertion neglects the interaction between parameters. In fact, for an arbitrary fixed $\lambda\geq 0$, the sub-sampling size $m$ can also be utilized to balance the bias and variance, just as \citep{Lin2021} did. In our theorem, we do not consider the interaction between $\lambda $ and $m$, and only use $\lambda$  for the bias-variance trade-off purpose.  Under this circumstance, our derived generalization error in (\ref{Error-analysis-2-exp}) is non-increasing with respect to $m$.

Finally, we should highlight that the restriction on $m$ in (\ref{bound-on-m}) is a bit strict, which makes Theorem \ref{Theorem:error-for-exp} trivial for $r=1/2$ in the sense that $m$ should be not smaller  than $n$. The main reason for this is due to that we adopt the integral algorithm operator based on second order decomposition for operator difference  \citep{Lin2017} in our proof. We believe it can be relaxed by using the approach developed by  \citet{Lin2020} that estimated the operator difference by using some operator concentration inequality. Furthermore, compared with our numerical results in Section \ref{Sec.experiment} in which an extremely small $m$ (such as $m=6$ for $n=4000$) is sufficient for yielding a comparable generalization error with KRR, our theoretical result  in (\ref{bound-on-m}) is too pessimistic. The main reason for such an inconsistency is that our theory holds for all distributions satisfying Assumptions 1-5, which is indeed  a worst case analysis, while our numerical results only focus on some specific distributions.

Theorem \ref{Theorem:error-for-exp} established learning rate analysis for geometrical $\tau$-mixing time series, which is an extension of analysis for the classical i.i.d. samples \citep{Rudi2015}.
In our next theorem, we prove that similar results also hold for Nystr\"{o}m regularization with sequential sub-sampling  for algebraic $\tau$-mixing time series.

\begin{theorem}\label{Theorem:error-for-alg}
Let $0<\delta\leq 1/2$ and $D_j^*$ be an arbitrary sequential sub-sampling of size $m$ for $D$.
Under Assumptions 1-5 with $\frac12\leq r\leq1$ and $0<s\leq 1$,
   if (\ref{def-Galpha11}) holds, $\lambda=n^{-\frac{2\gamma_1}{2\gamma_1(2r+s)+2r+1}}$ and
\begin{equation}\label{bound-on-m-alg}
      m\geq n^{\frac{2\gamma_1(s+1)+2}{2\gamma_1(2r+s)+2r+1}},
\end{equation}
 then with confidence $1-\delta$, there holds
\begin{eqnarray}\label{Error-analysis-2-exp}
    \|f_{D,D_j^*,\lambda}-f_\rho\|_\rho
    \leq
     \hat{C}n^{-\frac{2\gamma_1r}{2\gamma_1(2r+s)+2r+1}} {\log^4\frac2\delta}, \qquad \forall j\in[1,n-m+1],
\end{eqnarray}
 where
$\hat{C}$ is a constant independent of $m,n,j $ or $\delta$.
\end{theorem}
As ${\gamma_1} \rightarrow {\infty}$, meaning that $D$ consists i.i.d. samples,  (\ref{Error-analysis-2-exp}) becomes
$$
    \|f_{D,D_j^*,\lambda}-f_\rho\|_\rho
    \leq
     \hat{C}n^{-\frac{r}{2r+s}} {\log^4\frac2\delta}, \qquad \forall j\in[1,n-m+1],
$$
which is the optimal learning rates for KRR. This shows that (\ref{Theorem:error-for-alg}) is a reasonable extension of the classical results for KRR \citep{Caponnetto2007} and Nystr\"{o}m regularization \citep{Rudi2015} for i.i.d. samples. It should be mentioned that for $\gamma_1<\infty$, the derived learning rate is always worse than that for i.i.d. samples. The main reason is that the dependence nature of samples sometimes reduces  the effective samples, just as \citep{Yu1994} did for $\beta$-mixing sequences and \citep{Modha1996} did for $\alpha$-mixing sequences.

\section{Related Work}\label{Sec.Related Work}
\qquad In this section, we present some related work to highlight the novelty of our results. For this purpose, we divide our presentation into three parts: related work on scalable kernel methods for i.i.d. samples, related work on learning rates analysis of kernel methods for non-i.i.d. samples and related work on scalable kernel methods for non-i.i.d. samples.

\subsection{Scalable kernel methods for massive data}
\qquad Along with the development of data mining, data of massive size are  usually collected for  learning purpose. Scalable learning algorithms that can tackle these massive data are  highly desired in the community of machine learning and thus numerous scalable schemes including the distributed learning \citep{Zhang2015,Shi2019}, localized SVM \citep{Meister2016,Thomann2017} and
learning with sub-sampling \citep{Williams2000,Gittens2016} were developed to equip kernel methods to reduce the computational burden.  For example,
learning with sub-sampling firstly  selects centers
of kernel with small size either in a random manner or by computing some leverage scores in   a data dependent  way, and then deduces  the final estimator based on the selected centers. The feasibility of these scalable variants have been rigorously verified in \citep{Zhang2015,Rudi2015,Meister2016,Yang2017,Lin2017,Wang2021}, provided the samples are i.i.d. drawn.

The most related work  on Nystr\"{o}m regularization for i.i.d. samples is \citep{Rudi2015}, in which learning rates of  Nystr\"{o}m regularization with both plain sub-sampling and leverage scores were derived.
Our results can be regarded as an extension of the interesting work \citep{Rudi2015} from i.i.d. samples to $\tau$-mixing  time series. It should be mentioned that there are mainly three differences between our work and \citep{Rudi2015}, although some important tools (Lemmas \ref{Lemma:Projection general} and \ref{Lemma:operator inequality general} below) for proofs   are borrowed from  \citep{Rudi2015}. At first, we are interested in developing Nystr\"{o}m regularization for time series, which requires totally different sub-sampling mechanism. In particular, we propose  a novel sequential sub-sampling approach to equip Nystr\"{o}m regularization to guarantee the mixing property of the selected columns of kernel matrix. Then, due to the non-i.i.d. nature of time series, the well developed integral operator  approach in \citep{Caponnetto2007,Lin2017,Guo2017} for i.i.d. samples is unavailable. We then turn to utilizing  a Banach-valued Bernstein inequality  established in a recent work \citep{Blanchard2019} for $\tau$-mixing sequences to derive tight bounds for differences between integral operators and their empirical counterparts.  Finally,  besides establishing almost optimal learning rates for the developed Nystr\"{o}m regularization, we numerically  find that learning with sub-sampling is more suitable for time series than i.i.d. samples in the sense that the former requires less sub-sampling ratio due to the dependent nature of samples. In fact, our toy simulations and real data experiments show that a sub-sampling ratio not larger than 0.05 is good enough to maintain the learning performance of KRR for time series, illustrating that Nystr\"{o}m regularization is practically feasible and efficient for massive time series.

\subsection{Learning performance of kernel methods for weak dependent samples}

\qquad It is impossible to quantify the learning rates of kernel methods for time series without presenting any restrictions on the dependence among samples \citep{Chen1998}, an extreme case of which is that all   samples in the data set are identical. Therefore, some mixing properties \citep{Doukhan1994} concerning the weak dependence among samples should be imposed on time series. $\alpha$-mixing \citep{Modha1996}, $\beta$-mixing \citep{Yu1994} and $\tau$-mixing \citep{Dedecker2004} are three most widely used conditions in learning theory.

Based on the well developed Bernstein-type inequality for $\beta$-mixing sequences \citep{Yu1994}, $\alpha$-mixing sequences \citep{Modha1996}, and $\tau$-mixing sequences \citep{Hang2017}, learning rates of kernel methods for $\alpha$-mixing, $\beta$-mixing and $\tau$-mixing sequences have been established in \citep{Alquier2012,Mcdonald2017}, \citep{Xu2008,Steinwart2009} and \citep{Dedecker2005,Hang2017}, respectively.
Unfortunately,  the derived  learning rates   for non-i.i.d data are always worse than those for i.i.d. samples.
This is mainly due to the dependence of mixing sequences, which reduces the number of valid   samples and makes the established Bernstein-type inequality be not so tight as that for i.i.d. samples.

The interesting work \citep{Sun2010}, to the best of our knowledge, is the first result to derive  learning rates for non-i.i.d. data via using the  integral operator approach rather than Bernstein inequality. As a result, the learning rates of KRR for $\alpha$-mixing sequences can achieve the optimal learning rates for i.i.d. samples, provided the $\alpha$-mixing sequences decay sufficiently fast and $s=1$ in Assumption \ref{Assumption:capacity}. This provides a springboard to study the learning performance of kernel methods for non-i.i.d. data, although the approach developed in \citep{Sun2010} cannot be extended for the general case $0<s\leq 1$ directly. Recently, \citep{Blanchard2019} established a Banach-valued Bernstein inequality for $\tau$-mixing sequences and developed a novel integral approach to deduce almost optimal learning rates for kernel-based spectral algorithms, provided the $\tau$-mixing coefficients decay sufficiently fast. Compared with \citep{Blanchard2019}, there are two novelties in our work. From the theoretical side,  we consider scalable variant of KRR  to reduce its   computational burden, while the analysis in \citep{Blanchard2019} is for KRR. As a result, our results require novel proof skills including the projection strategy for sub-sampling. In a word, our proofs skill can be regarded as a combination of \citep{Rudi2015} for projection strategy, \citep{Guo2017} for second-order decomposition of operator differences and \citep{Blanchard2019} for Banach-valued Bernstein inequality. From the numerical side, we conduct both toy simulations and two real time-series forecasting analysis to verify our theoretical findings, but \citep{Blanchard2019}  is only in a theoretical flavor. It should  be highlighted that the numerical experiments in this paper are important in the sense that they reveal the low sub-sampling ratio in Nystr\"{o}m regularization for time series and imply that Nystr\"{o}m regularization can be used as a noise-extractor in practice.

\subsection{Scalable kernel methods for massive time series}
\qquad There are numerous learning approaches developed to tackle massive time series, including scalable  bootstrap \citep{Laptev2012}, sketching \citep{Indyk2000}, and fast approximation correlation \citep{Mueen2010}. Though these approaches have been verified  to be feasible in practice, there lack solid theoretical results to demonstrate the running mechanism and reasonability, making  these methods sensitive to data.

In our recent work \citep{Sun2021}, we propose a distributed kernel ridge regression (DKRR) to handle massive $\alpha$-mixing time series. Using some covariance inequality for $\alpha$-mixing sequences,  \citet{Sun2021} successfully derived optimal learning rates for DKRR under similar assumptions as this paper. There are mainly four difference between \citep{Sun2021} and this paper. Firstly,
we focus on Nystr\"{o}m regularization while \citep{Sun2021} devoted to distributed learning. It should be noted that there are totally different scalable variants of KRR. In particular, our numerical results show that the proposed Nystr\"{o}m regularization for time series admits extremely small sub-sampling ratio while the numerical results in \citep{Sun2021} showed that DKRR is infeasible if the number of samples in local machines is too small.  Secondly, we are interested in $\tau$-mixing time series, which can be regarded as an extension of   $\alpha$-mixing time series in \citep{Sun2021}. The direct consequence is that the covariance inequality for $\alpha$-mixing sequences is unavailable to $\tau$-mixing sequence and novel integral operator approach is required. Thirdly, our results are described in probability while the results in \citep{Sun2021} are in the framework of expectation. It should be mentioned that learning rate analysis in probability is usually stronger than that in expectation in the sense that it is easy to derive an error estimate in expectation based on error in probability, just as our Corollary  \ref{Corollary:optimal-exp} shows, but not vice-verse. At last, our analysis holds for Assumptions \ref{Assumption:regularity} and  \ref{Assumption:capacity} with all $0<r\leq 1$ and $0<s\leq 1$, while the analysis in \citep{Sun2021} imposed an additional $2r+s\geq 1$ in deriving the learning rates.

In summary, Nystr\"{o}m regularization with sequential sub-sampling is a novel scalable learning approach for tackling massive time series data. Different from the widely used distributed learning schemes, our approach admits small sub-sampling ratio, possesses slightly better theoretical behavior and can be used without advanced computing resources.

\section{Simulation Studies}\label{Sec.experiment}
\qquad In this section, we conduct both toy simulations and two real world time series forecasting analysis to verify our theoretical statements  and show the advantages of Nystr\"{o}m regularization with sequential sub-sampling. Our numerical experiments were carried out in Matlab R2016a with Intel(R) Core(TM) i7-4720HQ CPU @2.60GHz, Windows 10.

\subsection{Toy simulations}
\qquad In this part, we carry out four simulations to verify the theoretical statements. The first simulation is to study the relationship between the sub-sampling ratio and test error (measured by RMSE: rooted mean squared error) to demonstrate the power of Nystr\"{o}m regularization for massive time series.
The second simulation focuses on illustrating the effectiveness of the proposed sequential sub-sampling by showing that the learning performance of Nystr\"{o}m regularization is independent of the the sampling position or sampling interval. The third simulation aims to verify Theorem \ref{Theorem:error-for-exp} and Theorem \ref{Theorem:error-for-alg} via showing the relationship between test error and the number of the training samples. The last one is to exhibit the capability of the Nystr\"{o}m regularization with sequential sub-sampling as a noise-extractor.

In all simulations, we consider two time series:  nonlinear model (\ref{f_1})  with $\varepsilon_t$ being  the independent noise satisfying $\varepsilon_t\sim\mathcal{U}(-0.7,0.7)$ (here  $\mathcal{U}(a,b)$ represents the uniform distribution on $(a, b)$), and  Markov chains with Bernoulli distribution (\ref{f_2}), where $\{\varepsilon_t\}_{t=1}^T$ are i.i.d. drawn from the Bernoulli distribution $\mathcal B(1/2)$ and are independent of $x_0$. It should be mentioned that the time series generated by (\ref{f_1}) is an $\alpha$-mixing sequence \citep{Alquier2013} while that generated by (\ref{f_2}) is a $\tau$-mixing sequence but not an $\alpha$-mixing sequence \citep{Dedecker2005}.
\begin{eqnarray}
\mbox{Mechanism 1$(M_1)$}:\qquad x_{t} &=& 0.5\sin(x_{t-1})+\varepsilon_t,  \label{f_1}\\
\mbox{Mechanism 2$(M_2)$}:\qquad x_{t} &=& \frac12(x_{t-1}+\varepsilon_t), \label{f_2}
\end{eqnarray}
We adopt a widely used   Wendland  kernel \citep{Chang2017}  as follow:
\begin{eqnarray}\label{kernel}
   K(x,x')=\left\{
\begin{array}{cc}
 (1-\|x-x'\|_2)^4(4\|x-x'\|_2+1)\ &\text{if}\ 0<\|x-x'\|_2\leq1 \\
 0 &\ \text{if}\ \|x-x'\|_2>1.
\end{array}
\right.
\end{eqnarray}
It is easy to check that $K(t)$ is triple differentiable but not quadratic differentiable.

Let $N$  and $N_{test}$ be the number of training samples and test points respectively.
 We generate ($N+N_{test}+1$) samples via (\ref{f_1}) or (\ref{f_2}). The training samples are: $\{x_{t},x_{t+1}\}_{t=1}^{N}$  with $x_0$ drawn randomly according to $\mathcal U(0,1)$. Meanwhile, we also construct a test set: $ \{x_{t},x_{t+1}-\varepsilon_t\}_{t=N+1}^{N+N_{test}}$. In order to make a more faithful evaluation, we are concerned with one-step prediction, that is, the prediction of  $k$th test sample is built upon $N+(k-1)$ samples, where $k$ is an integer and varies in the range [$1$, $N_{test}$].

{\bf Simulation 1:\rm} In this simulation, we aim at studying relation between the learning performance of the proposed Nystr\"{o}m regularization and sub-sampling ratio.
As shown in our theoretical results in Section \ref{Sec.Mainresult}, the sub-sampling ratio controls not only   the generalization performance of  the proposed algorithm but also  the computational complexities and memory requirements.

The number of training samples $N$ and test samples $N_{test}$ are 2000 and 50, respectively.  We repeat the experiments 5 times to obtain the average RMSE.   The regularization parameters $\lambda$ are selected from $[5\times 10^{-4}:5\times 10^{-4}:0.01]$ (the first value is the lower bound of range, the second value is the step size, the third one is the upper bound of the range) and $[5\times 10^{-4}:5\times 10^{-5}:0.001]$ via cross-validation for $M_1$ and $M_2$ respectively. Our numerical results are reported in Figure \ref{select_m1}.
\begin{figure}
\centering
\subfigure[Relation on $M_1$]{\includegraphics[scale=0.4]{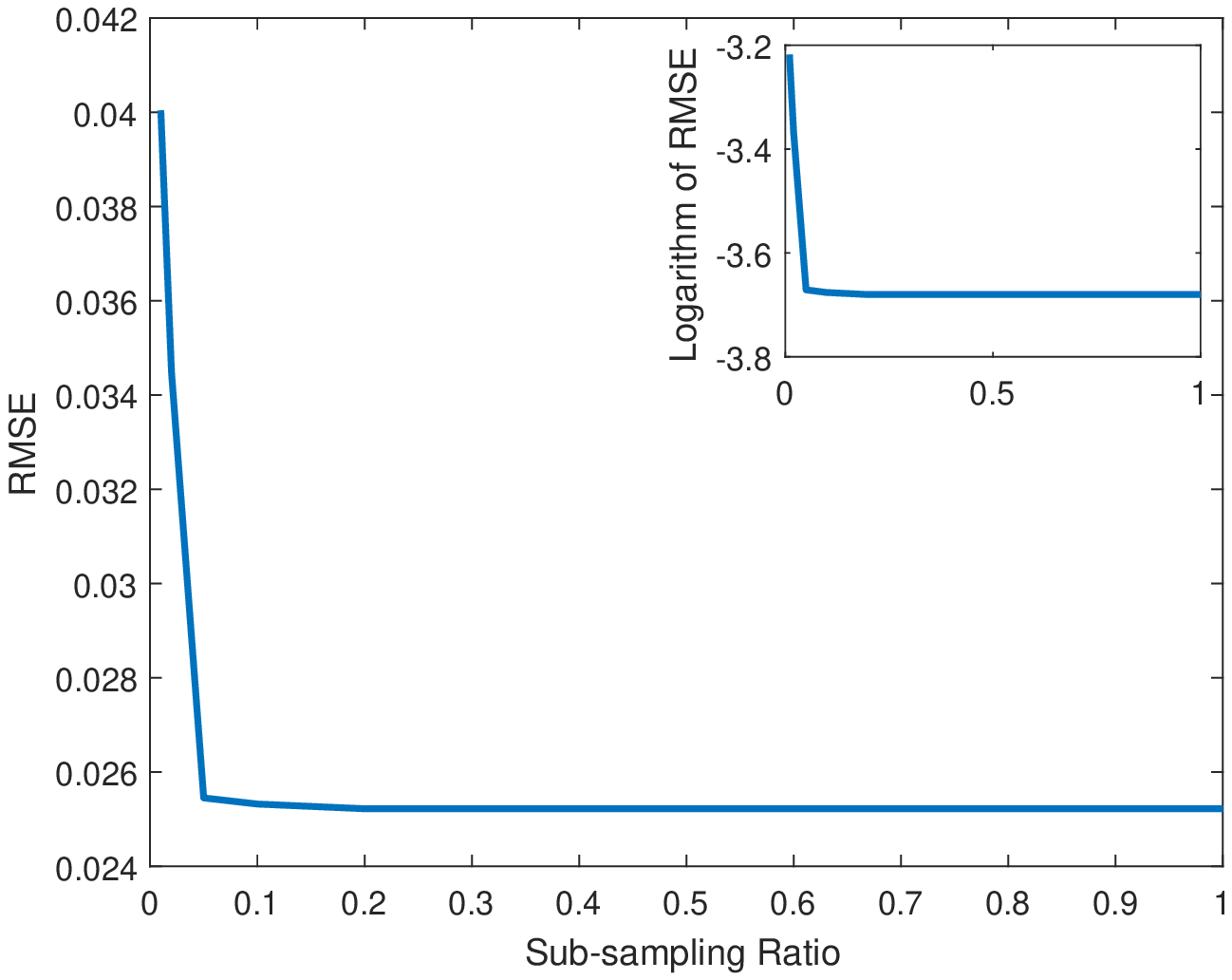}}
\subfigure[Relation on  $M_2$]{\includegraphics[scale=0.4]{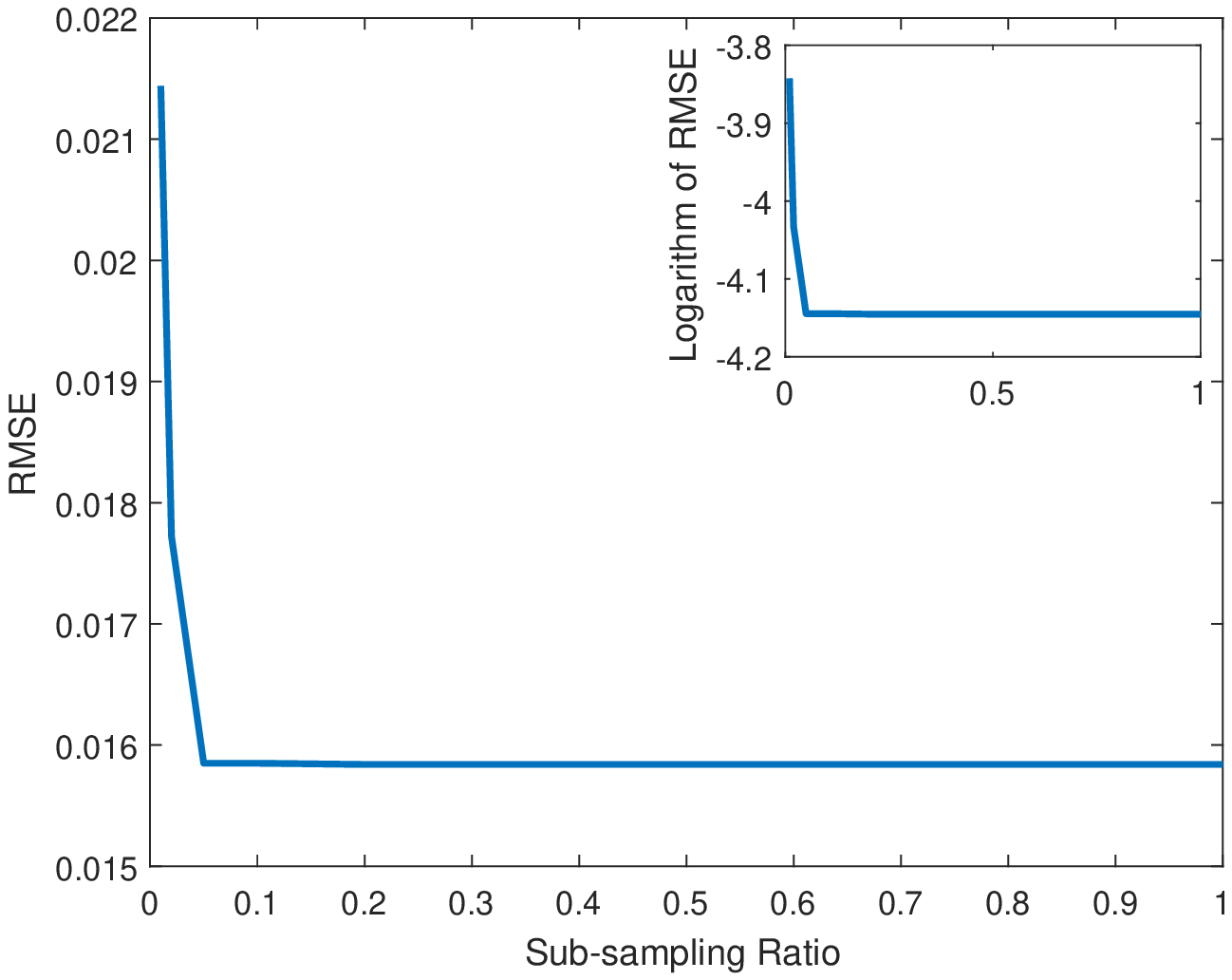}}
\caption{Relation between generalization error and sub-sampling ratio}\label{select_m1}
\end{figure}

From Figure \ref{select_m1}, we find three interesting phenomena: 1) for both $M_1$ and $M_2$, the generalization capability  does not decrease with respect to the sub-sampling ratio, which verifying our theoretical results in Theorem \ref{Theorem:error-for-exp} in the sense that if $m$ is larger than a specific value, then Nystr\"{o}m regularization with sequential sub-sampling reaches the optimal learning rates of KRR; 2) there exists a  lower bound of sub-sampling ratio (e.g., about 0.05 for both $M_1$ and $M_2$), smaller than which,  Nystr\"{o}m regularization with sequential sub-sampling degrades the learning performance of KRR dramatically. Noting that  the complexity of training for Nystr\"{o}m regularization with sequential sub-sampling is $O(nm^2)$, this reflects the dilemma in selecting $m$.  From the computational side, it is desired to select $m$ as small as possible. However, too small $m$ inevitably leads to bad generalization capability; 3) it should be highlighted that the lower bound of  sub-sampling ratio   to guarantee the learning performance of Nystr\"{o}m regularization is extremely small  (about 0.05) with which it is safe to set $m$ to be not so small  ($m=0.1$ for example). In particular, for 2000 samples,  Nystr\"{o}m regularization with $m\geq10$ is good enough to yield an estimator of high quality. It should be mentioned that the effective sub-sampling ratio for time series is  much less than that for i.i.d. sampling \citep{Rudi2015}. The main reason, as shown in Figure 2, is that the dependence nature of samples enhances the dependence among columns in the kernel matrix. As a result, the effective rank of kernel matrix of time series is smaller than that of i.i.d. samples.

{\bf Simulation 2:\rm} In this simulation, we devote to pursuing the role of sub-sampling strategy. As shown in Theorem \ref{Theorem:error-for-exp}, the learning performance of Nystr\"{o}m regularization is independent of the position of sub-sampling. It is thus urgent to verify such an independence. Additionally, we also show the role of sampling intervals in  Nystr\"{o}m regularization via comparing the proposed sequentially  sub-sampling set $D_j^*$ in Definition \ref{Def:seq-subsampling} with $D_{j,k}^*:=\{x_j,x_{j+(k+1)},x_{j+2(k+1)},\dots,x_{j+(m-1)(k+1)}\}$.

The number of training samples $N$ is fixed as $N=2000,10000$, and the sub-sampling ratio is fixed as $0.01$, which means that the sub-sampling size is fixed at $20,100$, respectively. To test on sub-sampling at arbitrary positions, we sub-sample the first 20(or 100) samples, the middle 20(or 100) samples and the last 20(or 100) samples of the training data sequence respectively.  Meanwhile,we choose different sub-sampling intervals with $k$ varying in $\{5,10,15,20,50,100\}$. The algorithm is employed to predict 5 testing samples and the experiments are repeated 20 times. The regularization parameter $\lambda$ is set as:
\begin{itemize}
\item For $M_1$ with 2000 training samples, the $\lambda$ is selected from $[5\times 10^{-4}:5\times 10^{-4}:0.01]$;
\item For $M_1$ with 10000 training samples, the $\lambda$ is selected from $[5\times 10^{-5}:1\times 10^{-4}:0.001]$;
\item For $M_2$ with 2000 training samples, the $\lambda$ is selected from $[5\times 10^{-4}:5\times 10^{-5}:0.001]$;
\item For $M_2$ with 10000 training samples, the $\lambda$ is selected from $[1\times 10^{-4}:2\times 10^{-5}:4\times 10^{-4}]$;
\end{itemize}

\begin{figure}
\subfigure[ $M_1$ with 2000 data]{\includegraphics[scale=0.25]{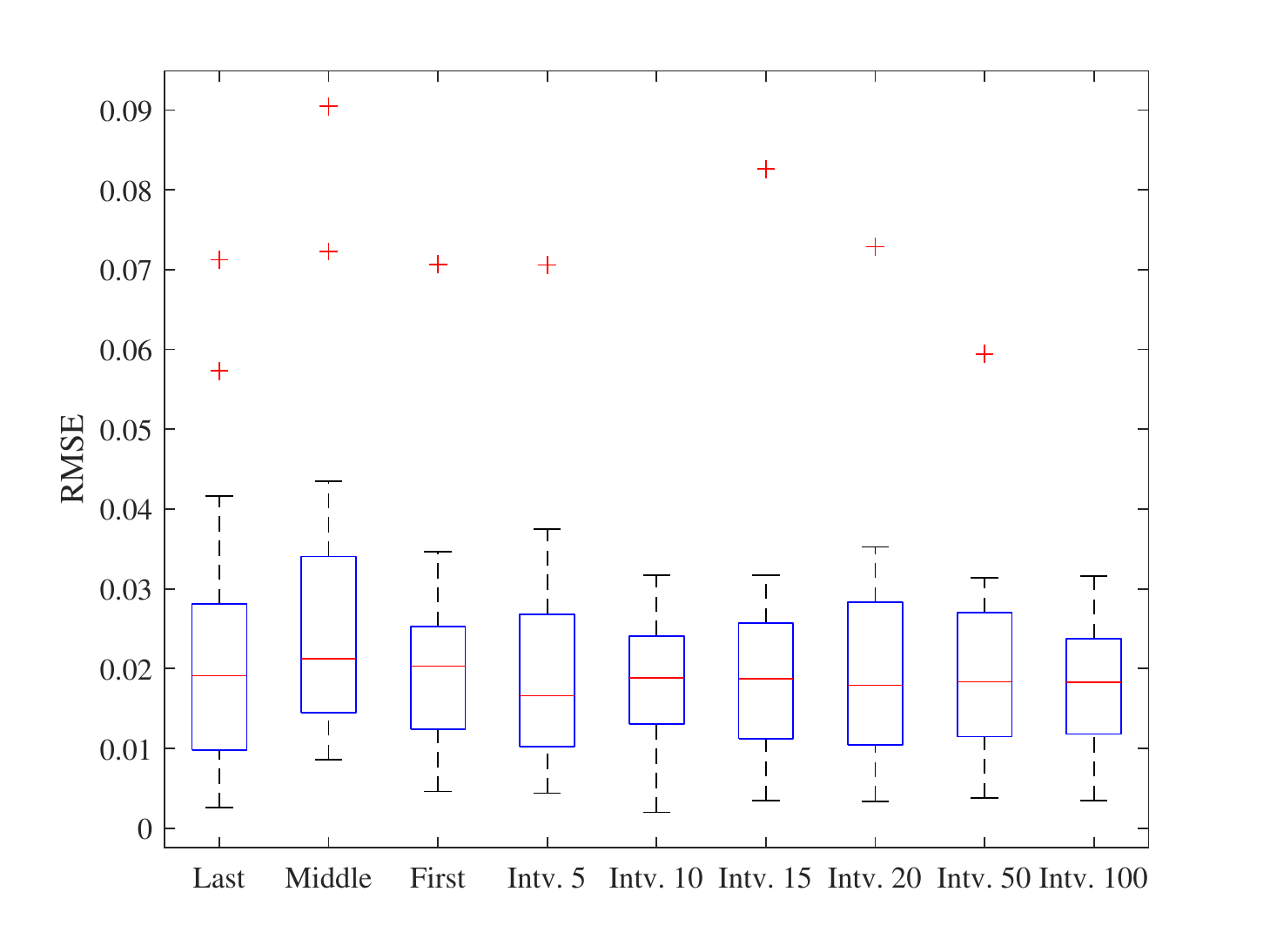}}
\subfigure[ $M_1$ with 10000 data ]{\includegraphics[scale=0.25]{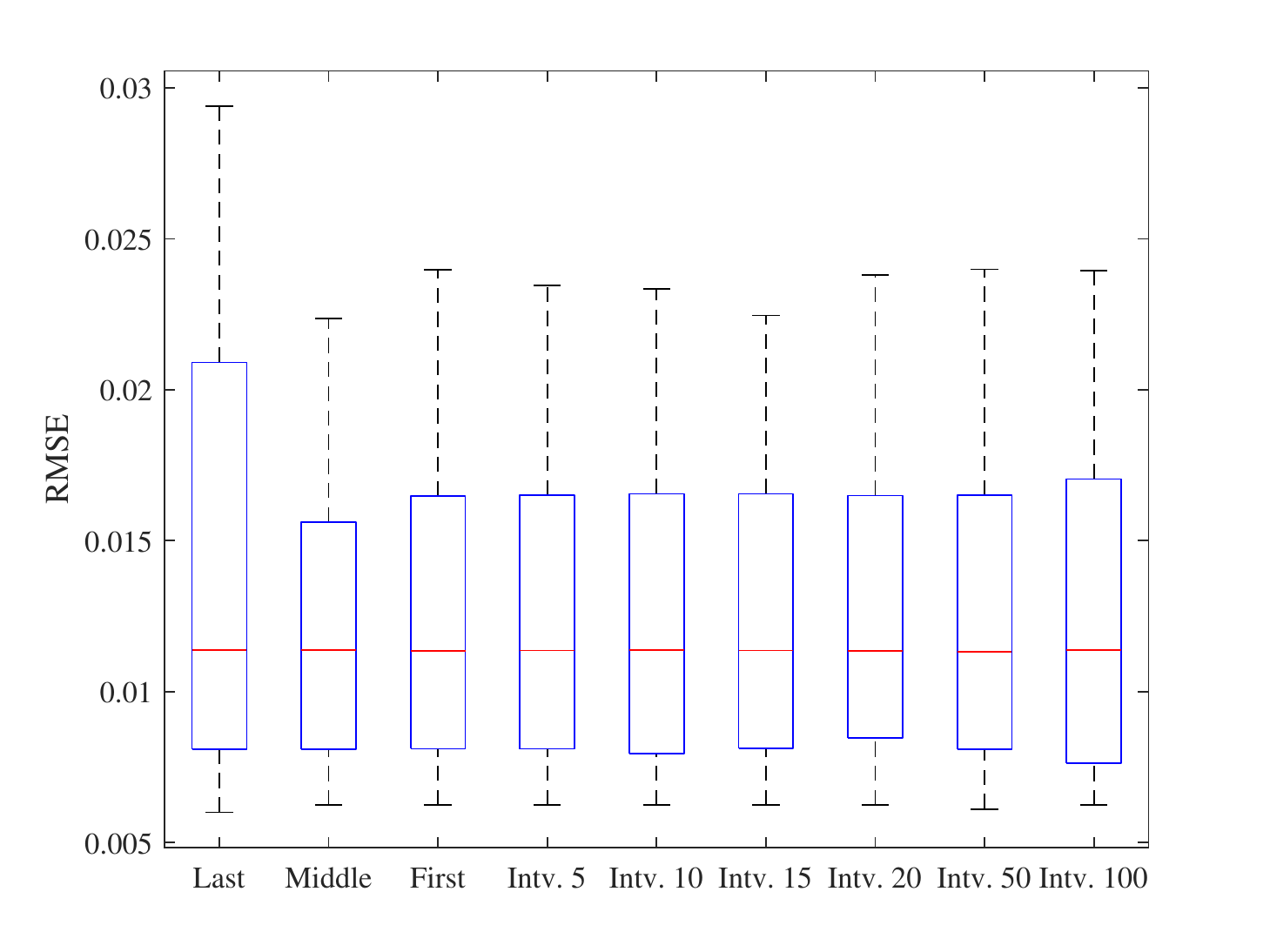}}
\subfigure[ $M_2$ with 2000 data]{\includegraphics[scale=0.25]{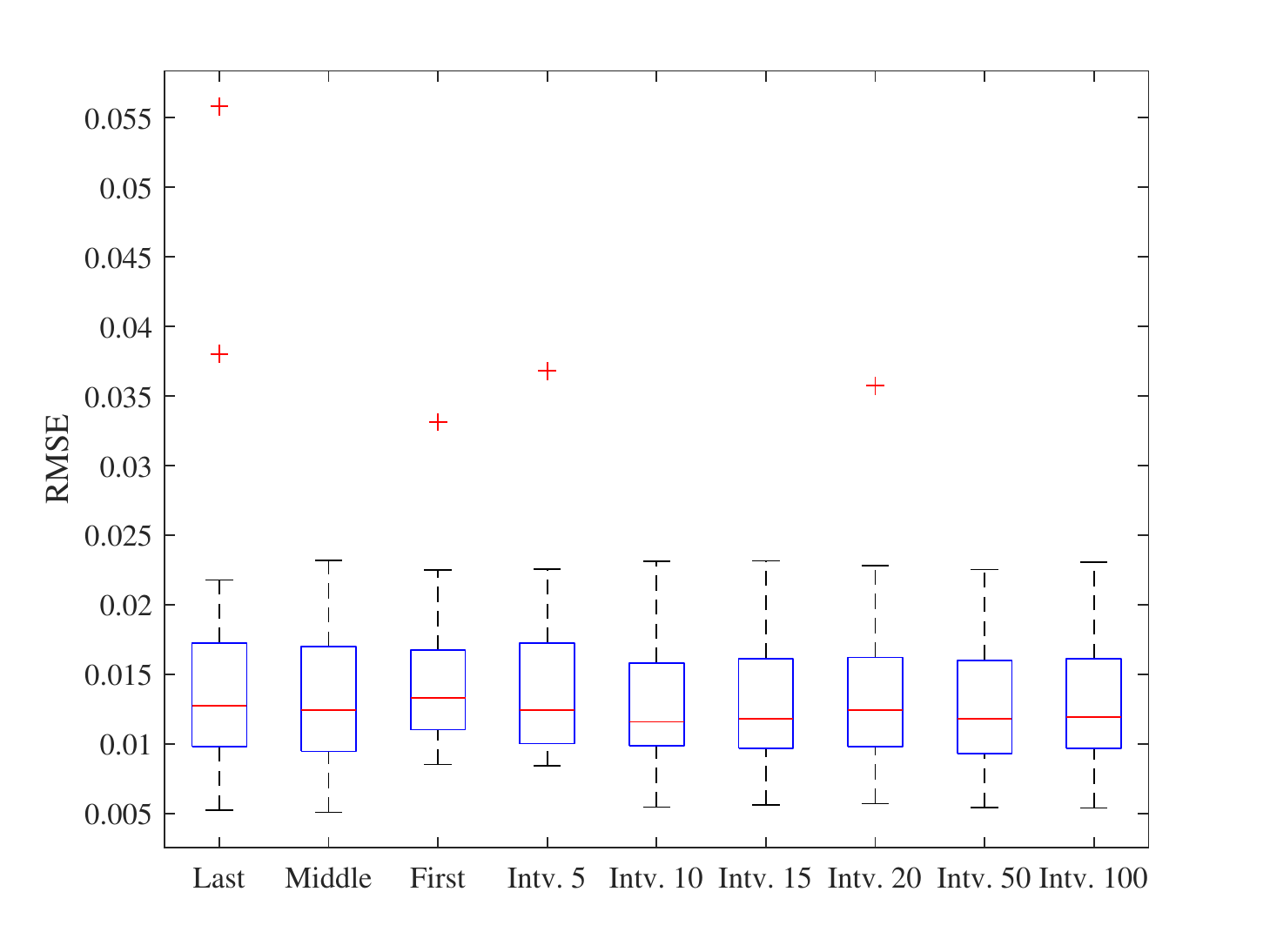}}
\subfigure[ $M_2$ with 10000 data]{\includegraphics[scale=0.25]{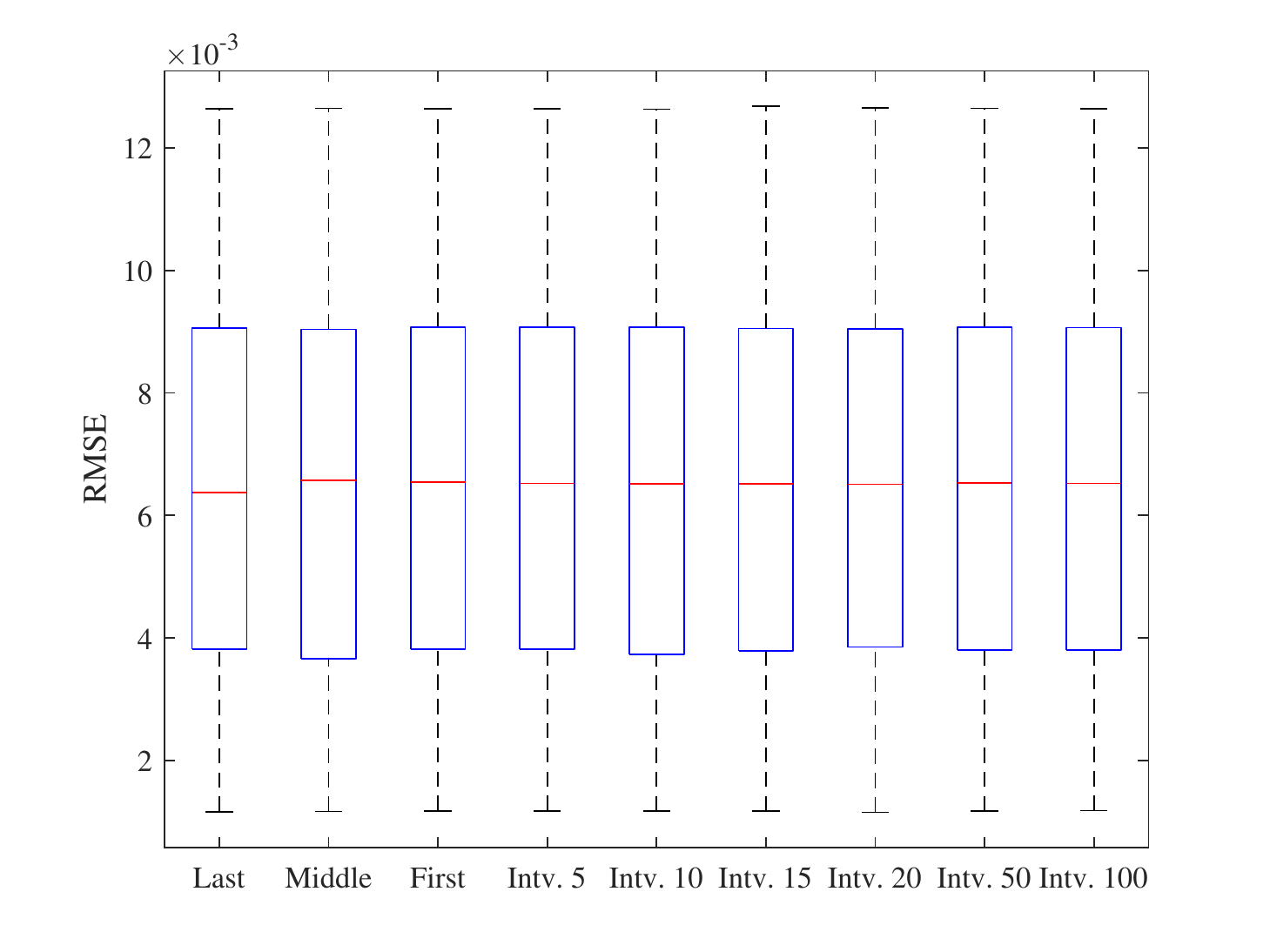}}
\caption{Role of sub-sampling mechanisms in    Nystr\"{o}m regularization}\label{arbitrary}
\end{figure}

Our numerical results are shown in Figure \ref{arbitrary}, where ``Last'', ``Middle'' and ``First''  means $j=N-m+1$, $j=N/2$ and $j=1$ respectively, while ``Intv.a'' means $k=a$ for $a=5,10,15,20,50,100$ in $D_{j,k}^*$.  For $N=2000$, though the mean of RMSE is almost the same for all sub-sampling mechanisms,  as shown in Figure  \ref{arbitrary} (a) and (c), the generalization capability  of Nystr\"{o}m regularization changes slightly with respect to different sub-sampling mechanism in each trivial. The reason is that for $N=2000$, there are only $20$ columns selected to build up the estimators. When the number of samples increases, for the same sub-sampling ratio, the number of  selected columns also increases. Then, it can be found in Figure  \ref{arbitrary} (b) and (d) that the generalization capability of Nystr\"{o}m regularization is independent of the sub-sampling mechanism. All the numerical results above show that the proposed sequential sub-sampling strategy is a good choice in practice.

{\bf Simulation 3:\rm} In this simulation, we study the relation between the generalization capability and number of samples to verify Theorem \ref{Theorem:error-for-exp} and Theorem \ref{Theorem:error-for-alg}.
We build the Nystr\"{o}m regularization estimator on  data sets of size  varying in $\{2000,5000,10000,20000,50000\}$ and fix the number  of testing samples as 10. We consider   Nystr\"{o}m regularization with three  sub-sampling ratios: 0.002, 0.005 and 0.01.
The regularization parameter  $\lambda$ is selected from $[2\times 10^{-4}:2\times 10^{-4}:0.004]$ and $[1\times 10^{-4}:1\times 10^{-5}:2\times 10^{-4}]$ for $M_1$ and $M_2$, respectively. The experimental results can be found in Figure \ref{select_n1}.
\begin{figure}
\centering
\subfigure[ Relation on $M_1$]{\includegraphics[scale=0.4]{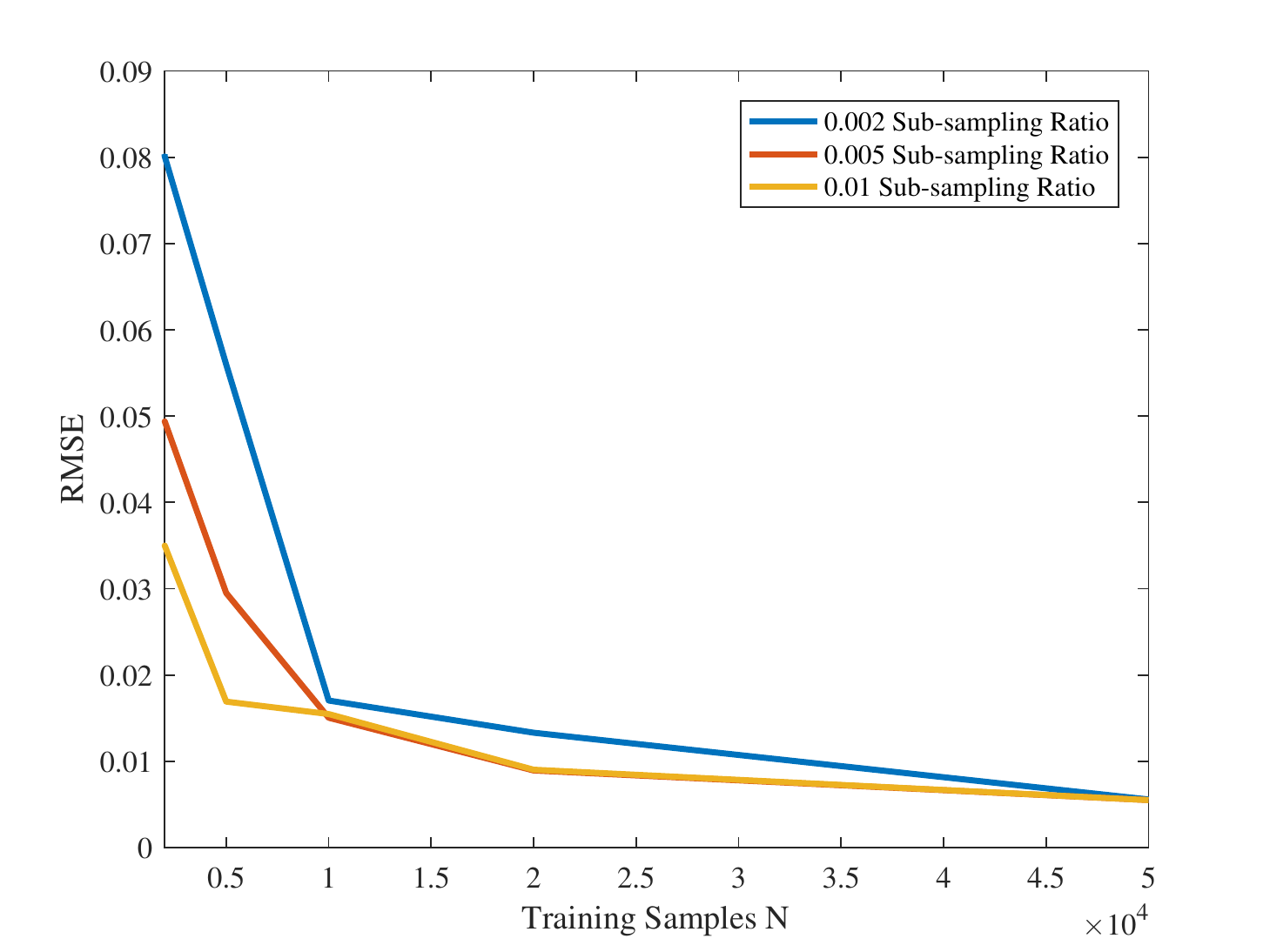}}
\subfigure[ Relation on $M_2$]{\includegraphics[scale=0.4]{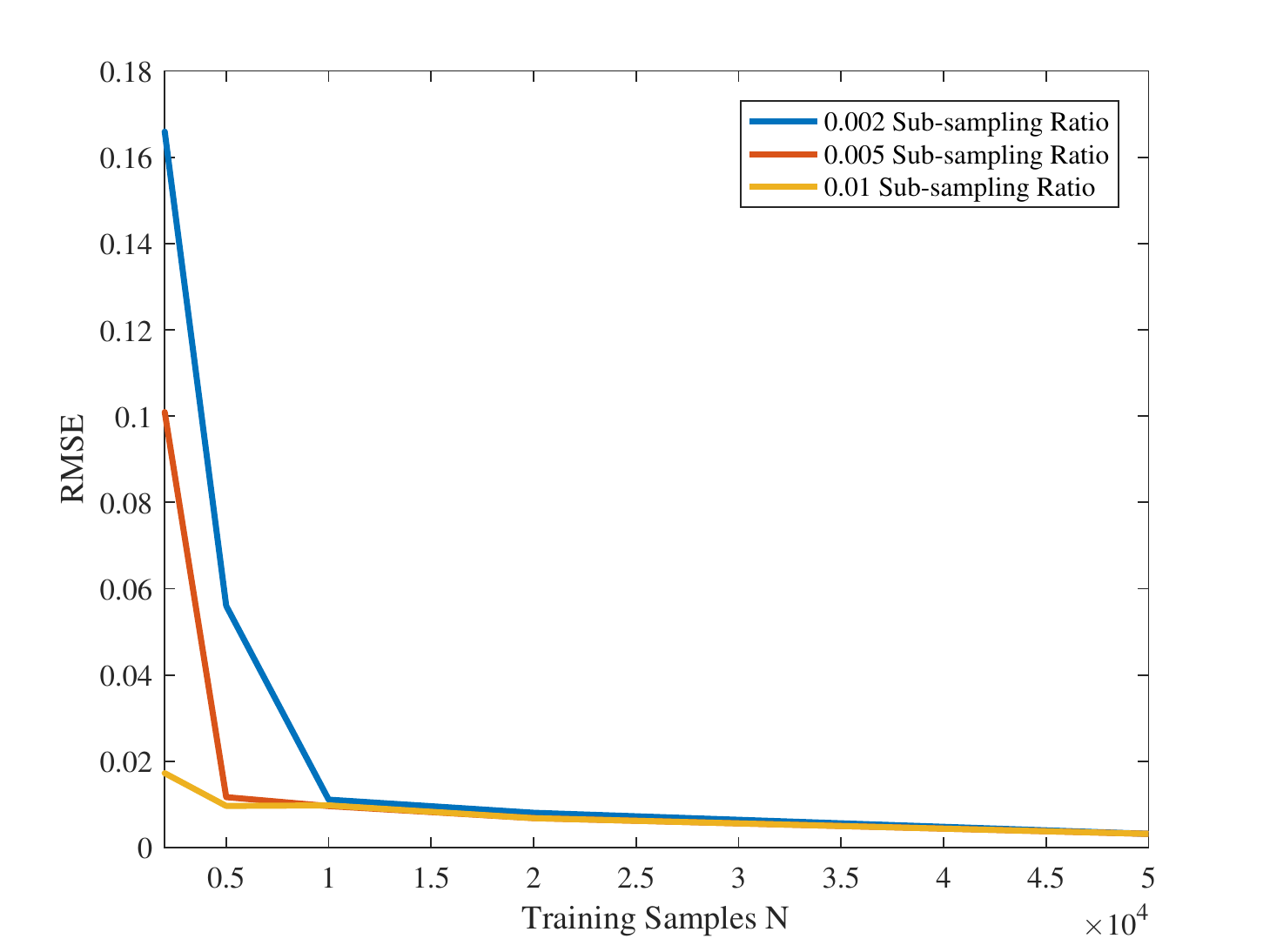}}
\caption{The relation between generalization error and number of samples}\label{select_n1}
\end{figure}

Figure \ref{select_n1}   exhibits two   findings for the proposed Nystr\"{o}m regularization. The one is that
RMSE decreases with respect to the number of samples. This partly verifies Theorem \ref{Theorem:error-for-exp} and Theorem \ref{Theorem:error-for-alg}, since our theoretical analysis is for the worst case. It is difficult for  KRR to get a similar curve as   Figure \ref{select_n1} since
KRR requires $\mathcal O(N^3)$  computational complexity and thus needs at least $\mathcal O(10^{14})$ floating computations, which is beyond the capability of nowadays computers.  The other is that Nystr\"{o}m regularization with different sub-sampling ratios perform similarly, provided the number of samples is larger than a specific value. Taking $M_2$ for example, three sub-sampling ratios perform the almost same when $N\geq 10000$. The main reason is that, for small size training data, too small sampling ratio results in extremely small hypothesis space and then degrades the generalization capability of KRR. Both findings verify our theoretical assertions and show the efficiency of the sequential sub-sampling scheme.

{\bf Simulation 4:\rm} In the previous simulations, the performance of Nystr\"{o}m regularization were tested on clean test data. However, in real-world time series forecasting, it is impossible to neglect the noise, making the machine learning  estimator  always lag  behind the real-world time series \citep{Chen1998}. As shown in Figure 1, a real-world time series can be divided into a deterministic part to quantify the relation between samples in successive time and a random part to describe the uncertainty of time series. Our theorems and simulations show  that  Nystr\"{o}m regularization is capable of discovering the deterministic part of time series, but similar as all existing machine learning approaches, it cannot catch the random part of time series. This makes our approach be not so good for time series with large random noise.  However, since the deterministic part of time series is  well approximated by the proposed Nystr\"{o}m regularization, we can use the derived estimator to be a noise-extractor to pursue the distribution of the random noise of time series. In this simulation, we focuses on the performance of Nystr\"{o}m regularization in extracting the random noise.

For this purpose,
we generated three time series according to (\ref{f_1}), with $\varepsilon_t$ be i.i.d. drawn from $\mathcal{B}(0.5)$, $\mathcal{U}[-0.2,0.2]$ and $\mathcal{N}(0,0.1^2)$, respectively. In this simulation, we set the number of training samples $N$ as 2000, $\lambda = 0.005$ and the sub-sampling ratio as 0.01. The number of test samples is 2000 (statistical bar), else 50. Our simulation results can be found in Figure \ref{Bernoulli}.
\begin{figure}[h]
\subfigure[without noise]
{\includegraphics[scale=0.25]{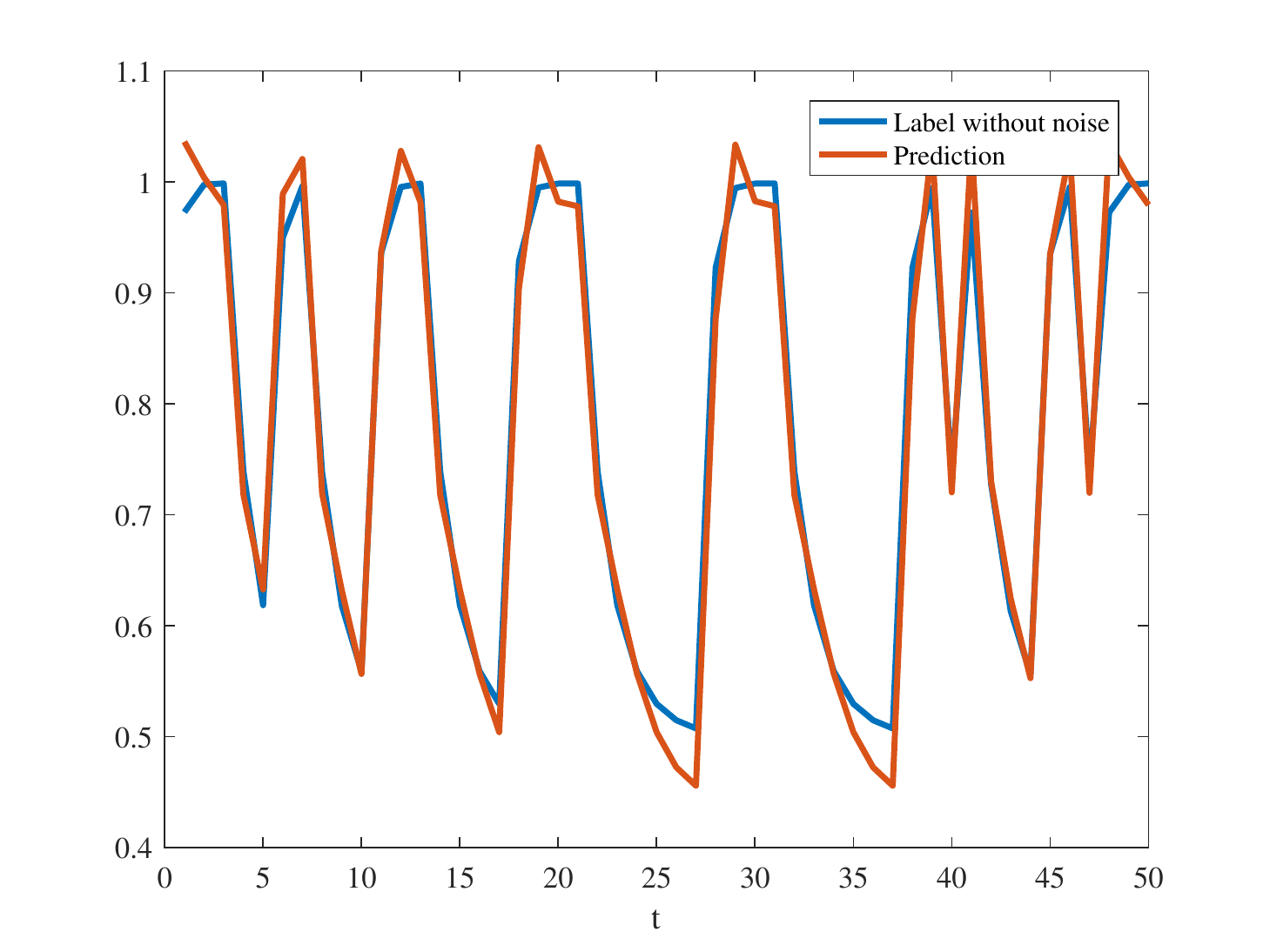}}
\subfigure[with bernoulli noise]{\includegraphics[scale=0.25]{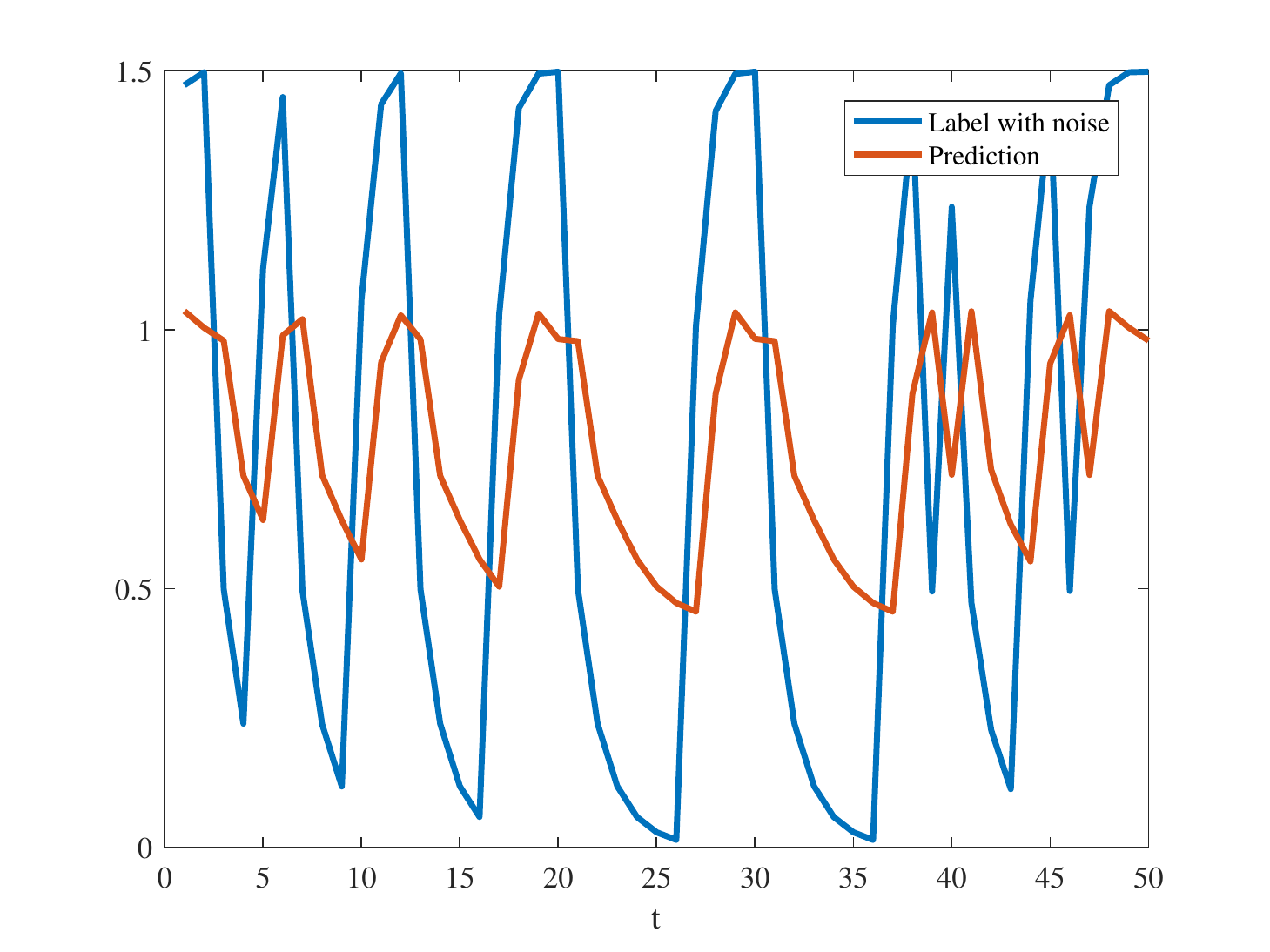}}
\subfigure[noise-exactor]{\includegraphics[scale=0.25]{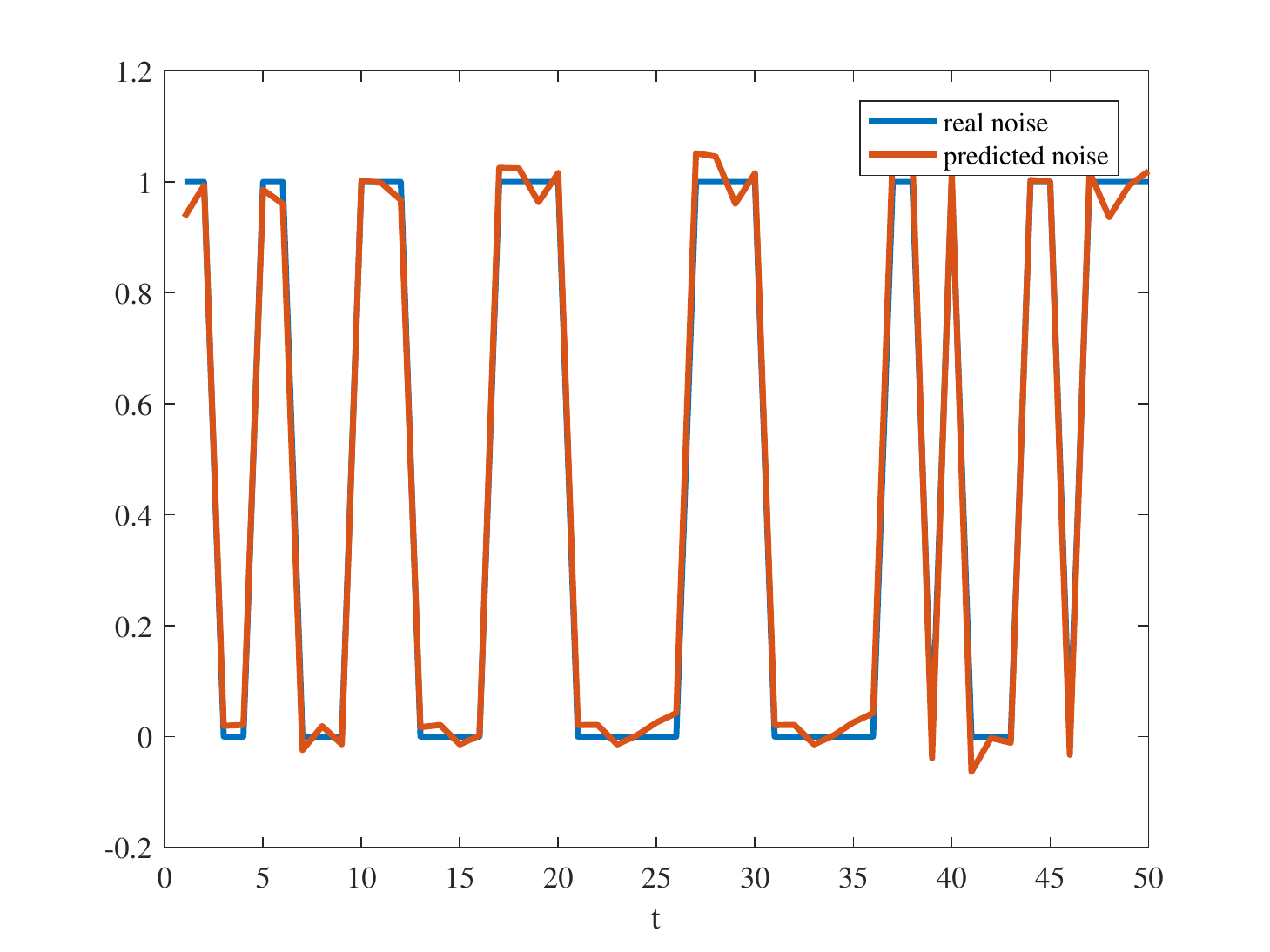}}
\subfigure[statistical bar]
{\includegraphics[scale=0.25]{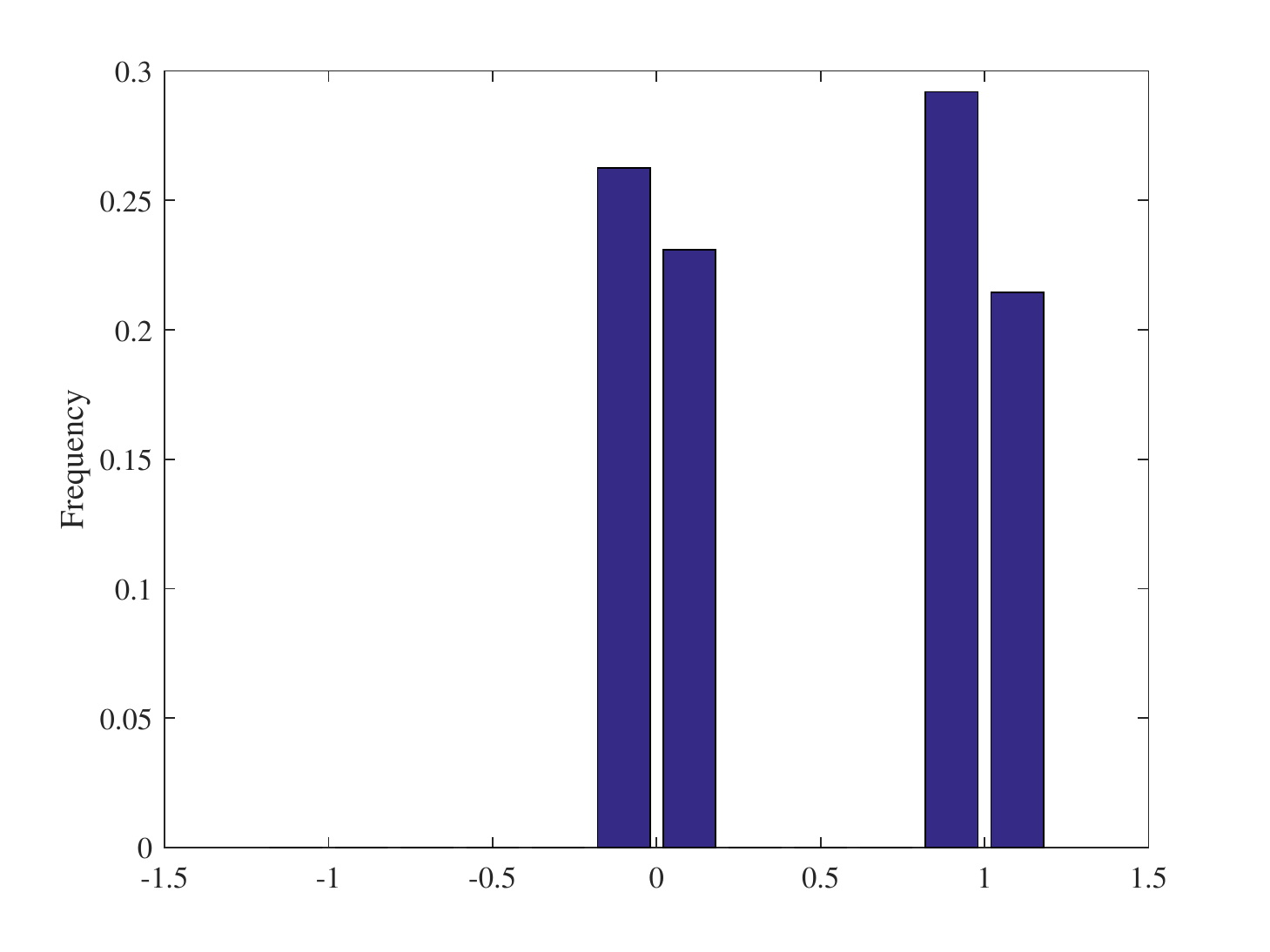}}
\subfigure[without  noise]{\includegraphics[scale=0.25]{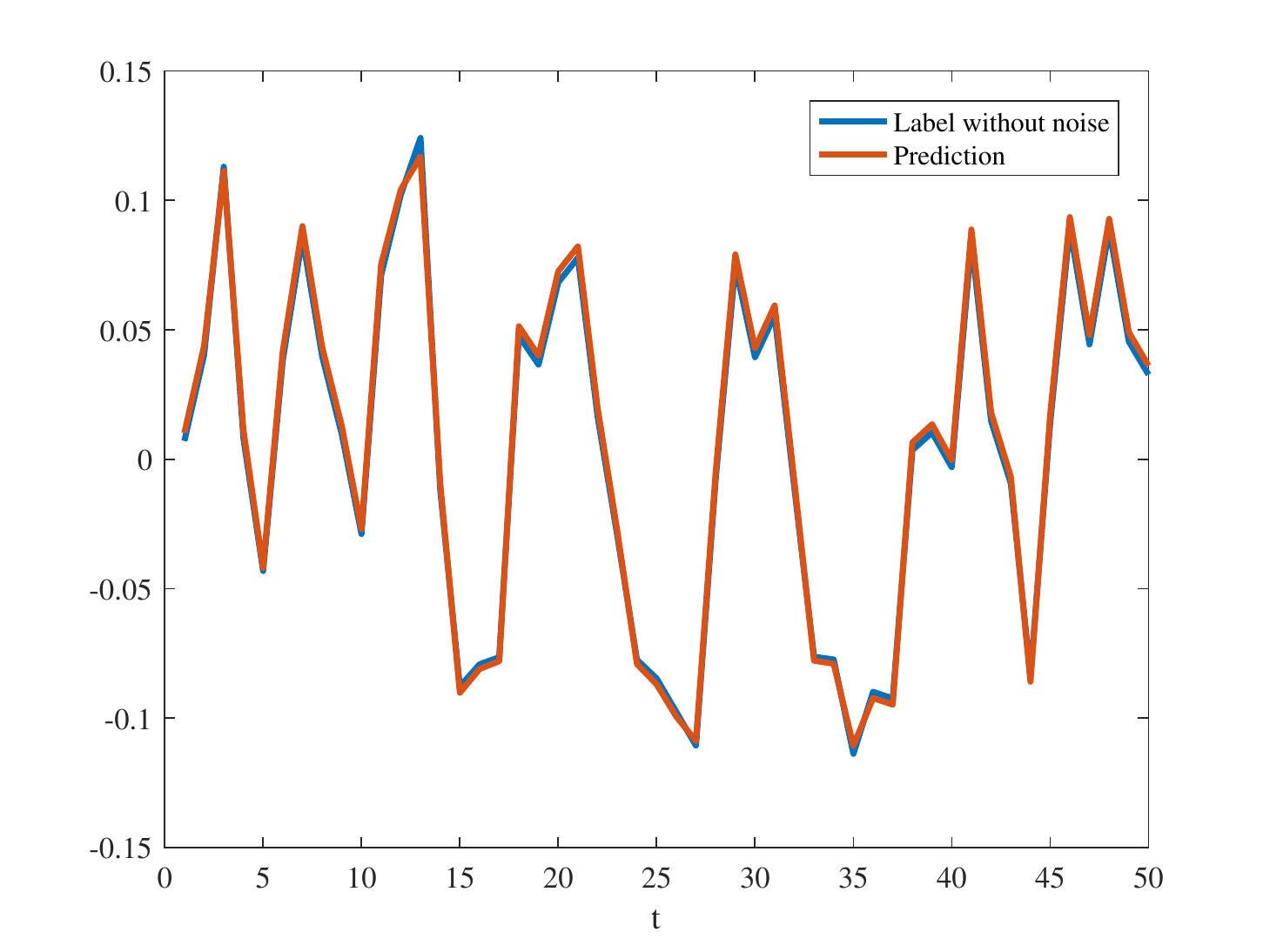}}
\subfigure[with uniform noise]{\includegraphics[scale=0.25]{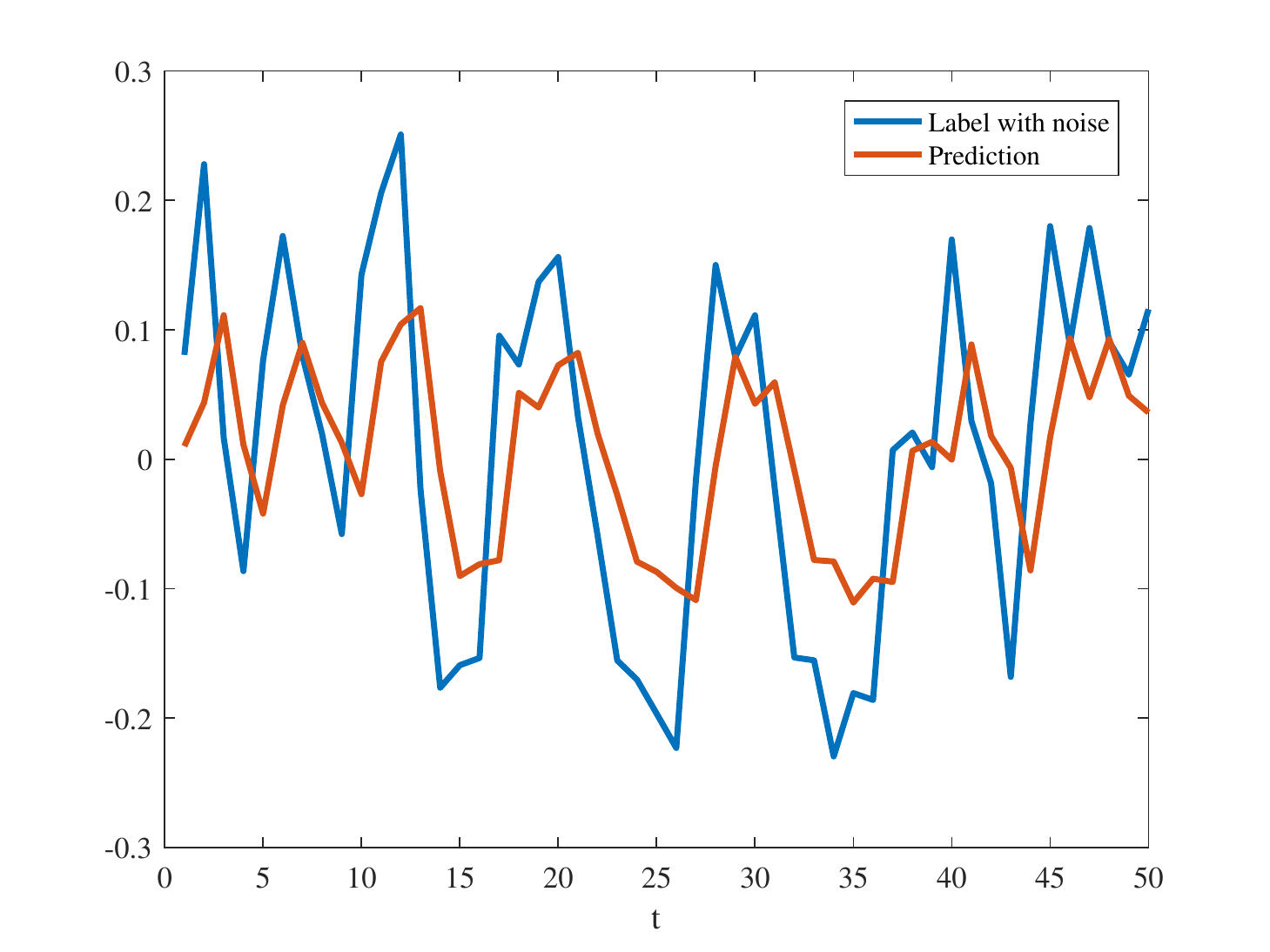}}
\subfigure[noise-exactor]{\includegraphics[scale=0.25]{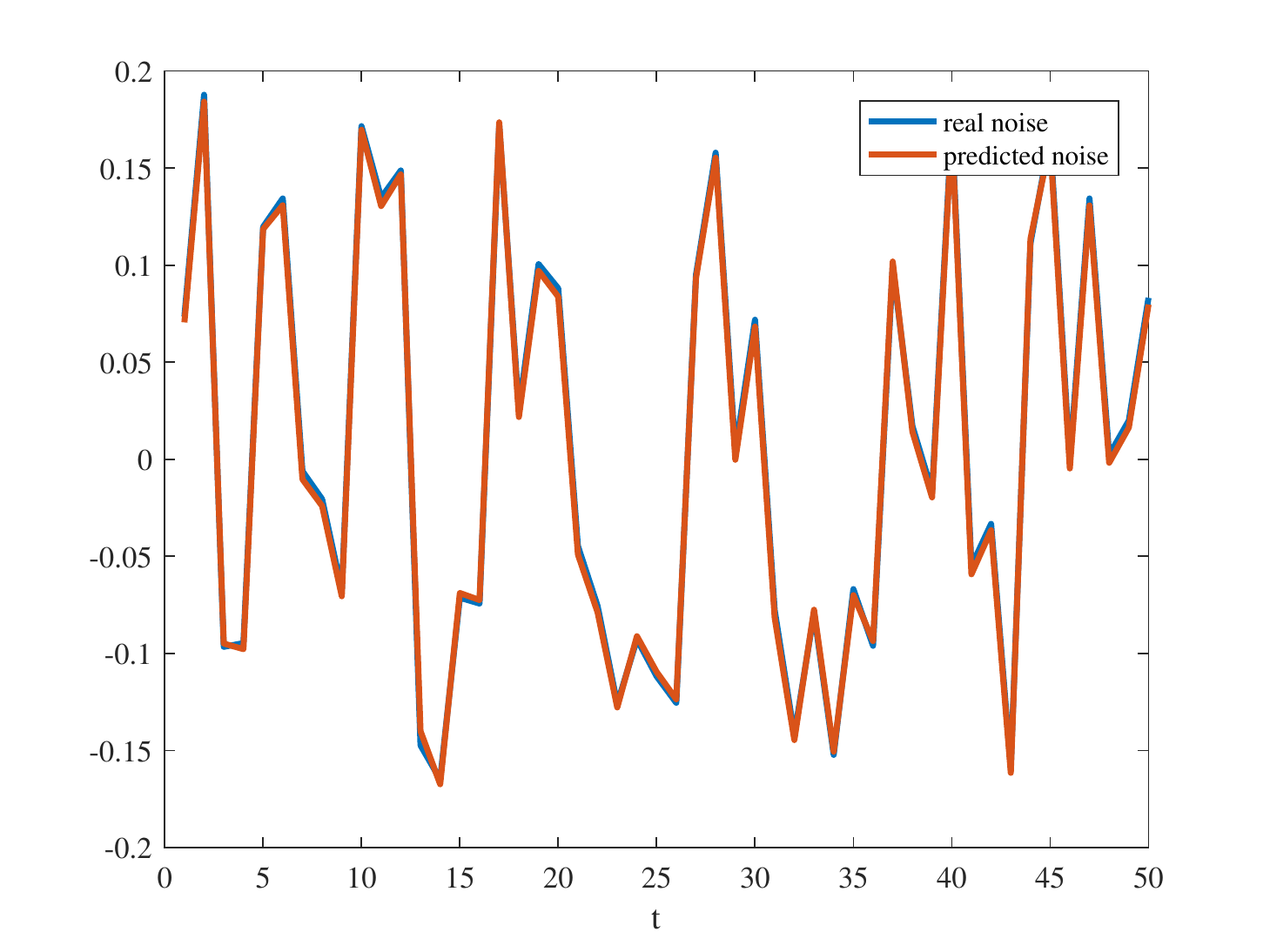}}
\subfigure[statistical bar]{\includegraphics[scale=0.25]{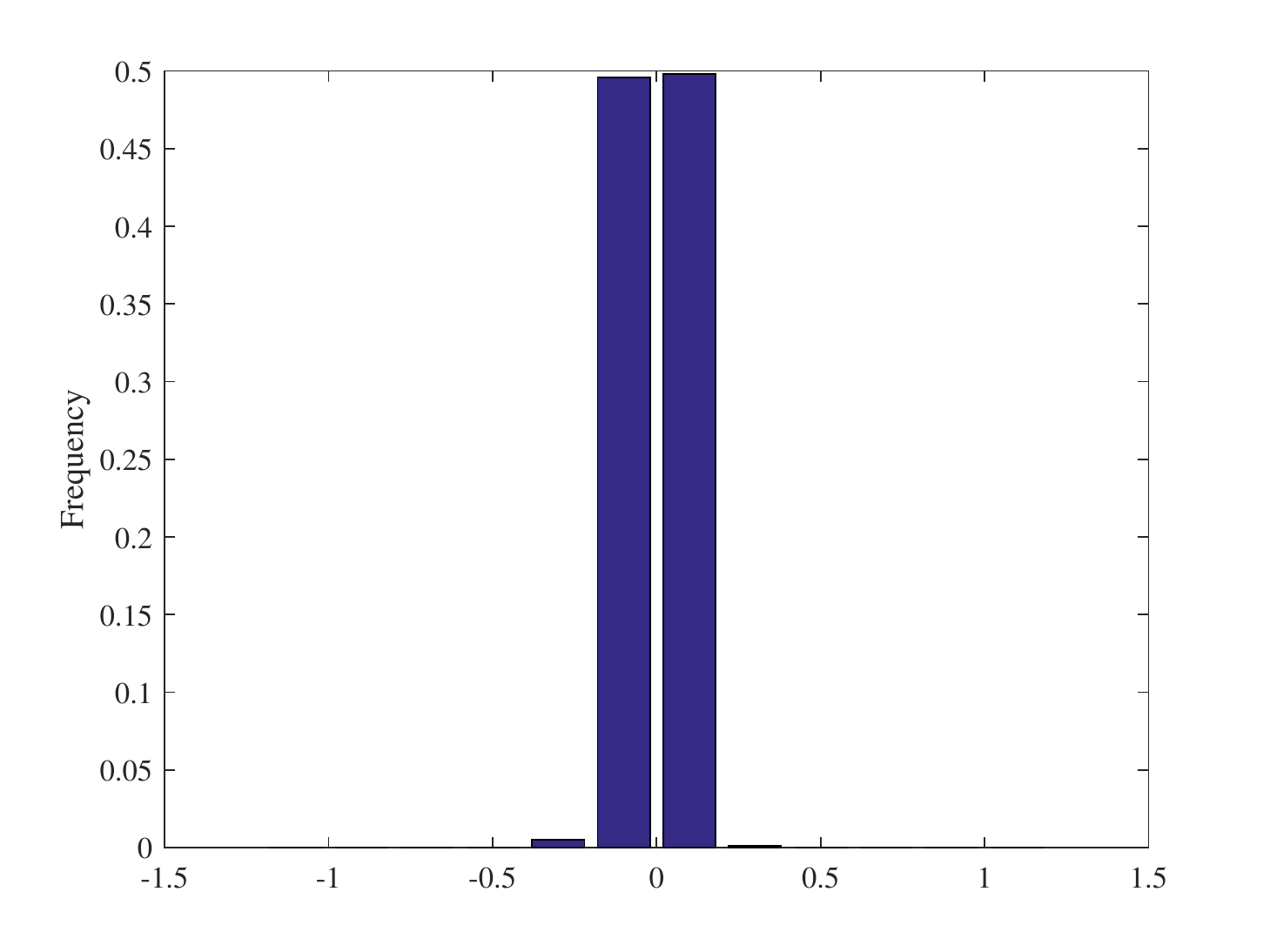}}
\subfigure[without noise]
{\includegraphics[scale=0.25]{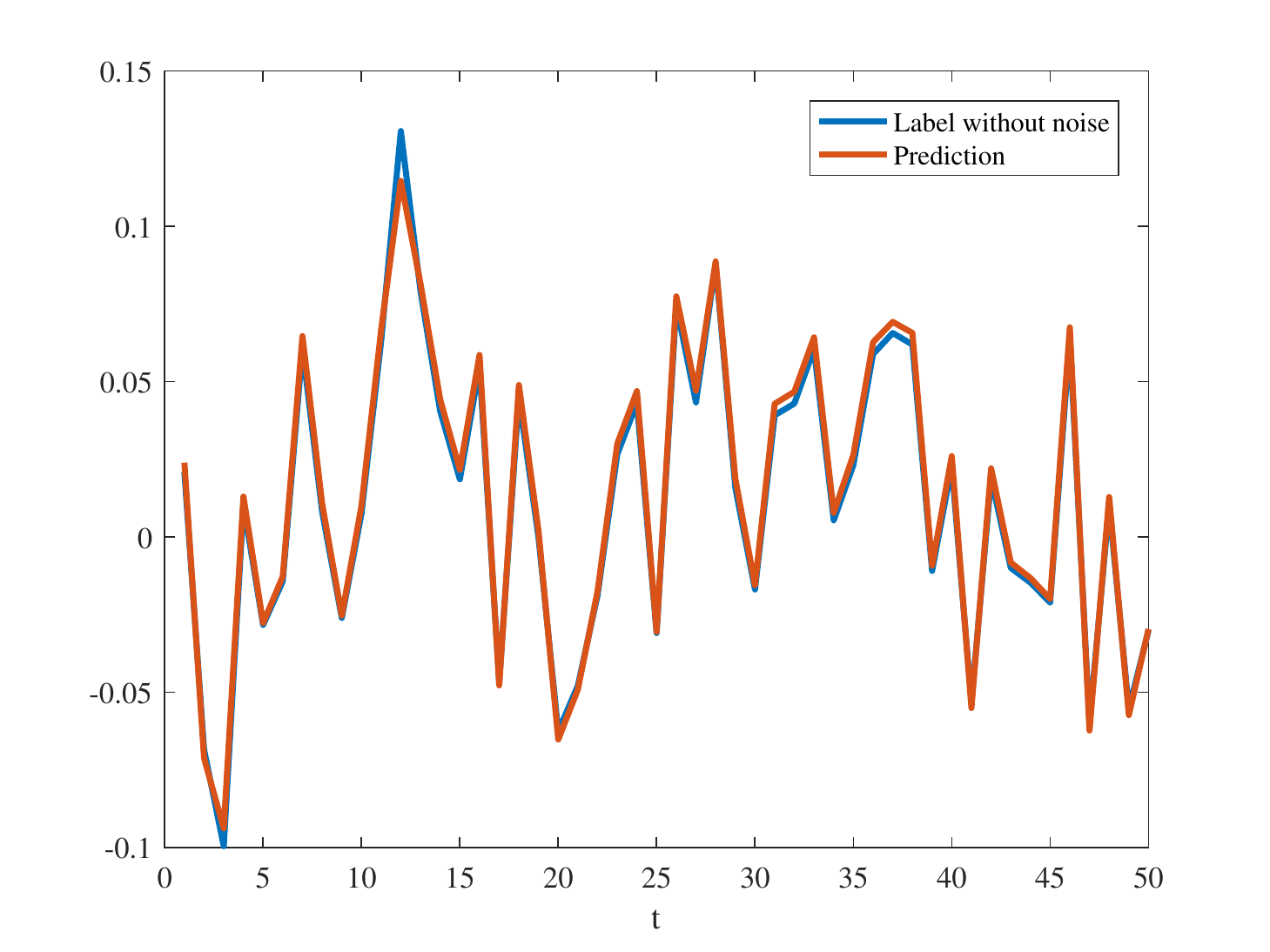}}
\subfigure[with gaussian noise]{\includegraphics[scale=0.25]{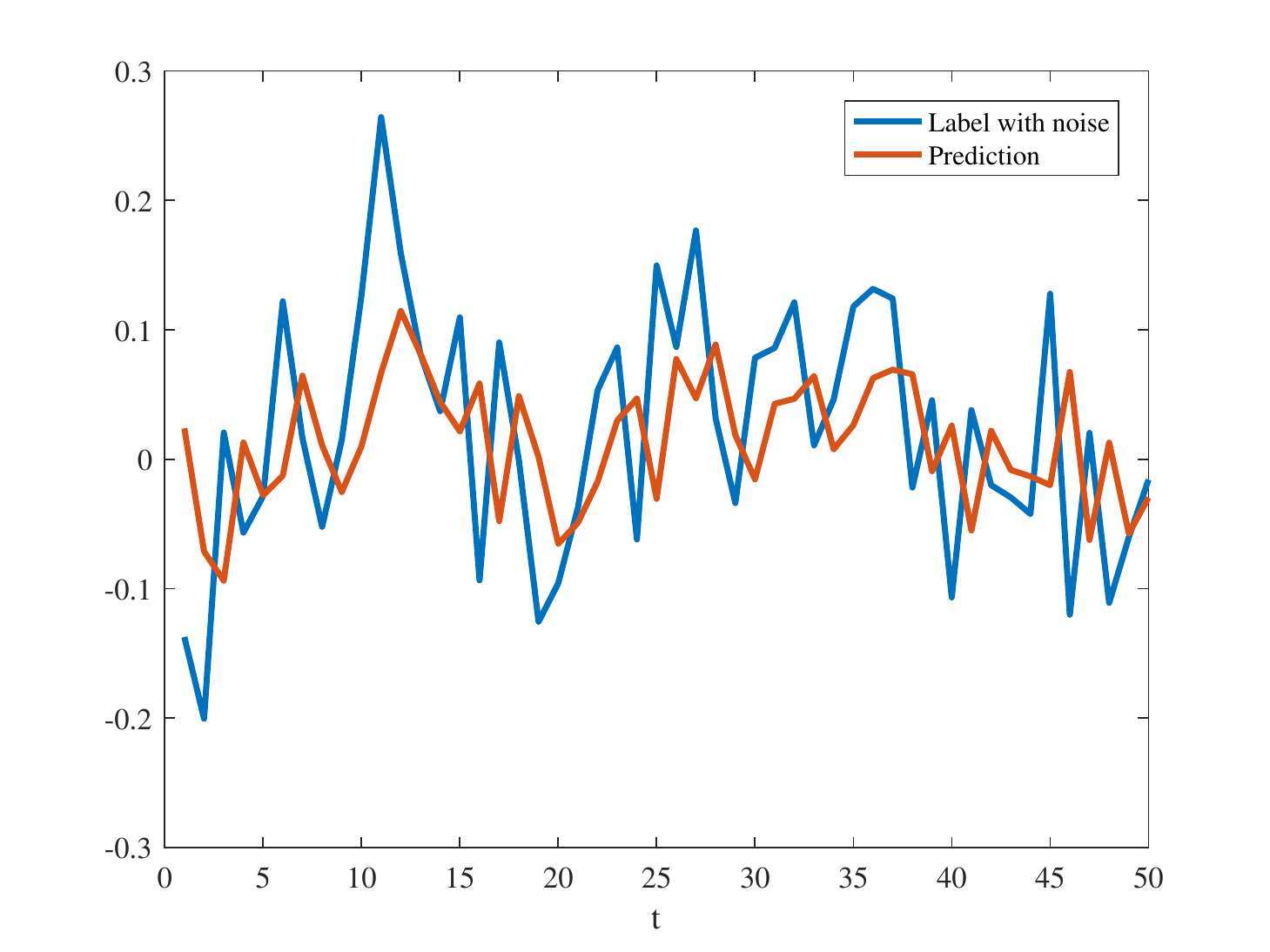}}
\subfigure[noise-exactor]{\includegraphics[scale=0.25]{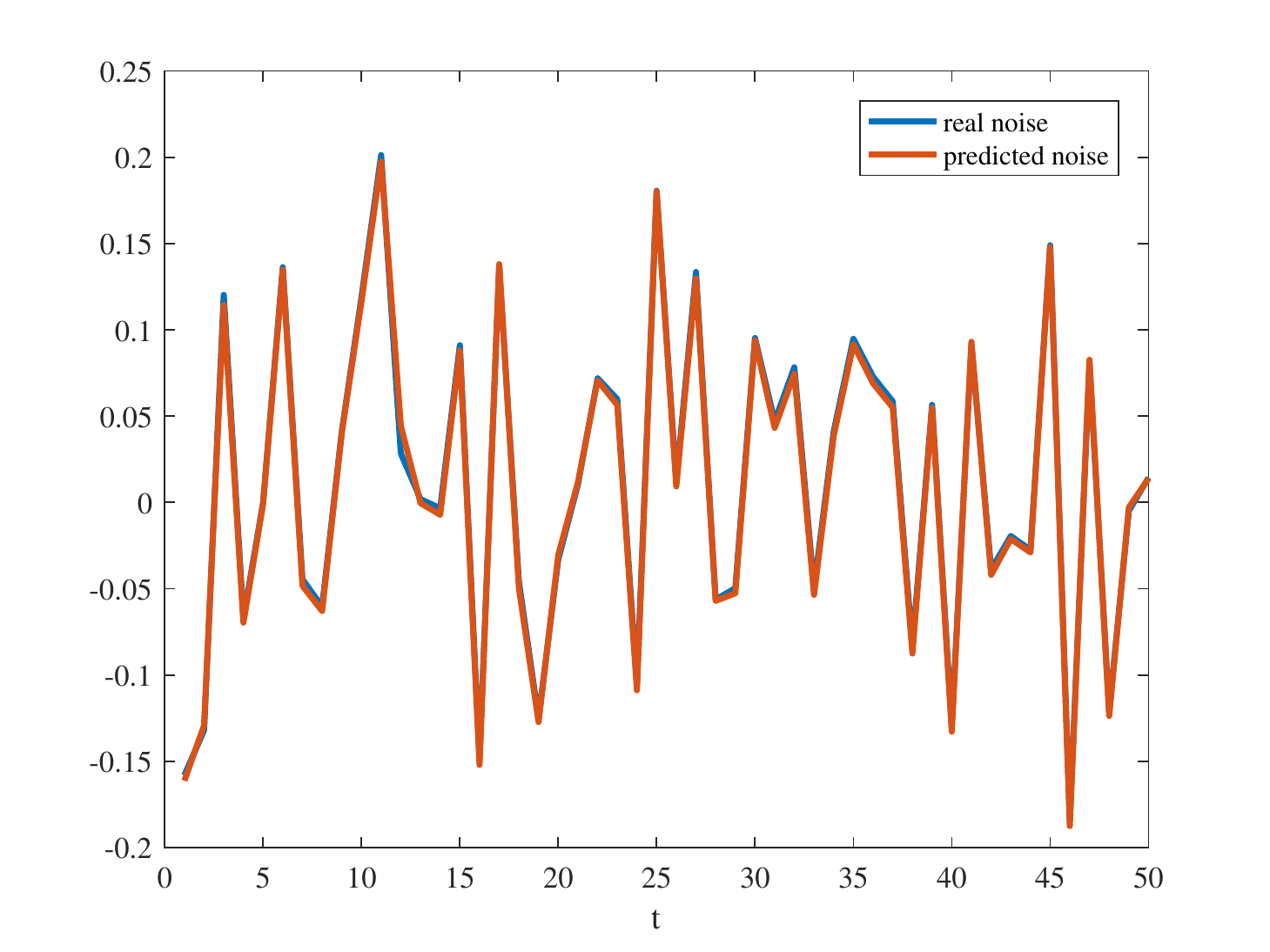}}
\subfigure[statistical bar]
{\includegraphics[scale=0.25]{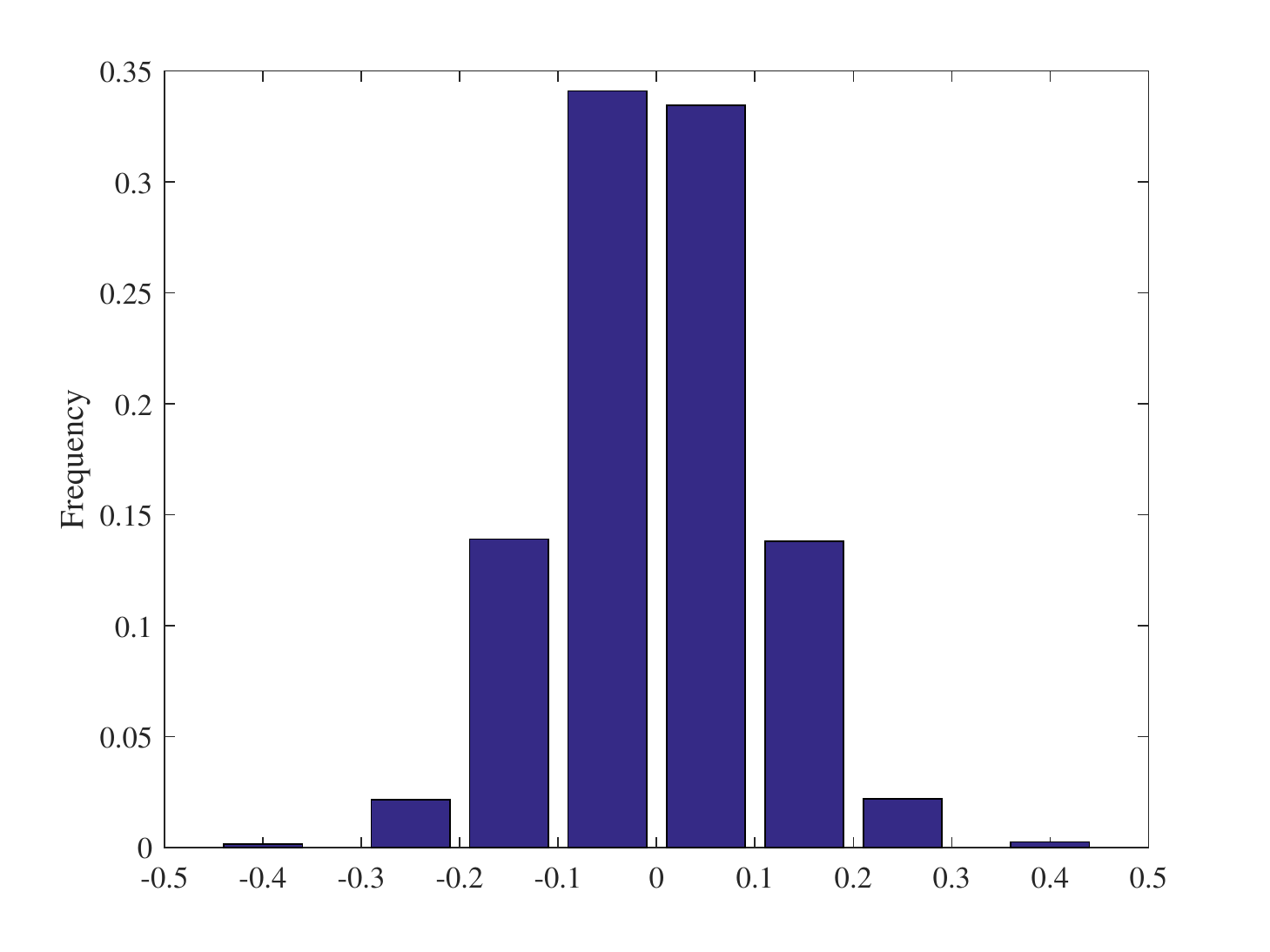}}
\caption{Noise extractor for Nystr\"{o}m regularization: (a)-(d) for  $\mathcal{B}(0.5)$; (e)-(h) for $\mathcal{U}[-0.2,0.2]$ and (i)-(l) for $\mathcal{N}(0,0.1^2)$.}\label{Bernoulli}
\end{figure}

It can be found in Figure \ref{Bernoulli} (a), (e), (i) that the proposed Nystr\"{o}m regularization with sub-sampling ratio $0.01$ can precisely catch the deterministic part of the time series, i.e., yielding an estimator that can approximate $x_{t}  =  0.5\sin(x_{t-1})$ very well, even though the training samples are contaminated by different noises. However, Figure \ref{Bernoulli}(b), (f), (j) show that the proposed algorithm is incapable of catching the random noise, making the derived estimator lag heavily behind the real time series.
In Figure \ref{Bernoulli} (c), (g), (k), we compare the real noise  with  $x_{t+1}-f_{D,D_j^*,\lambda}(x_t)$ to regard Nystr\"{o}m regularization as a noise-exactor. It is exhibited that using Nystr\"{o}m regularization, it is easy to mimic the random noise part of the time series.
From Figure \ref{Bernoulli} (d), (h), (l), we find that the distribution of the extracted noise via Nystr\"{o}m regularization is similar to the real noise.  All these findings yield the
the following two conclusions: 1) Nystr\"{o}m regularization with sequential sub-sampling can successfully capture the trend information (deterministic part) of time series. 
2) Nystr\"{o}m regularization with sequential sub-sampling has strong ability to extract the noise of data (not only the value but also the distribution).
Although the noise-extractor property of Nystr\"{o}m regularization seems trivial in toy simulation, it is extremely important in real-world time series forecasting, where the deterministic part and random noise part cannot be divided as those in the toy simulations.

\subsection{Real world applications}
\qquad In this part, we apply the proposed Nystr\"{o}m regularization with sequential sub-sampling to two real-world time series forecasting data, WTI data and BTC data, to show its learning performance.

\textbf{WTI data}:
In the industrial society, oil almost dominates the production of the whole industry. Oil and additional industries account for a high proportion in the national economy.  A deep understanding of the international crude oil price fluctuation mechanism and improving the accuracy of international crude oil price forecasts are of great significance to economic development, enterprise production operations and investment. WTI  (Western Texas Intermediate) Spot Prices from EIA U.S. (Energy Information Administration)
is suitable for refining gasoline, diesel, thermal fuel oil and aircraft fuel, etc., and can increase the output value of refineries. It is crude oil with a higher utilization rate. WTI is an important part of the international energy pricing system and has become the benchmark for global crude oil pricing. Meanwhile, WTI has become one of the two most market-oriented crude oils in the world.
Therefore, the analysis and prediction of WTI data is very important for finance and production.
The data (https://datahub.io/core/oil-prices) is daily recorded and from September 12, 2018 to August 28, 2020. Since it is difficult to check the smoothness of the regression function, we use a less-smoothed  kernel function
\begin{eqnarray*}
K(x,x')&=&1+min(x,x')
\end{eqnarray*}
in real-world data experiments.

\begin{figure}
\subfigure[ACF]{\includegraphics[scale=0.25]{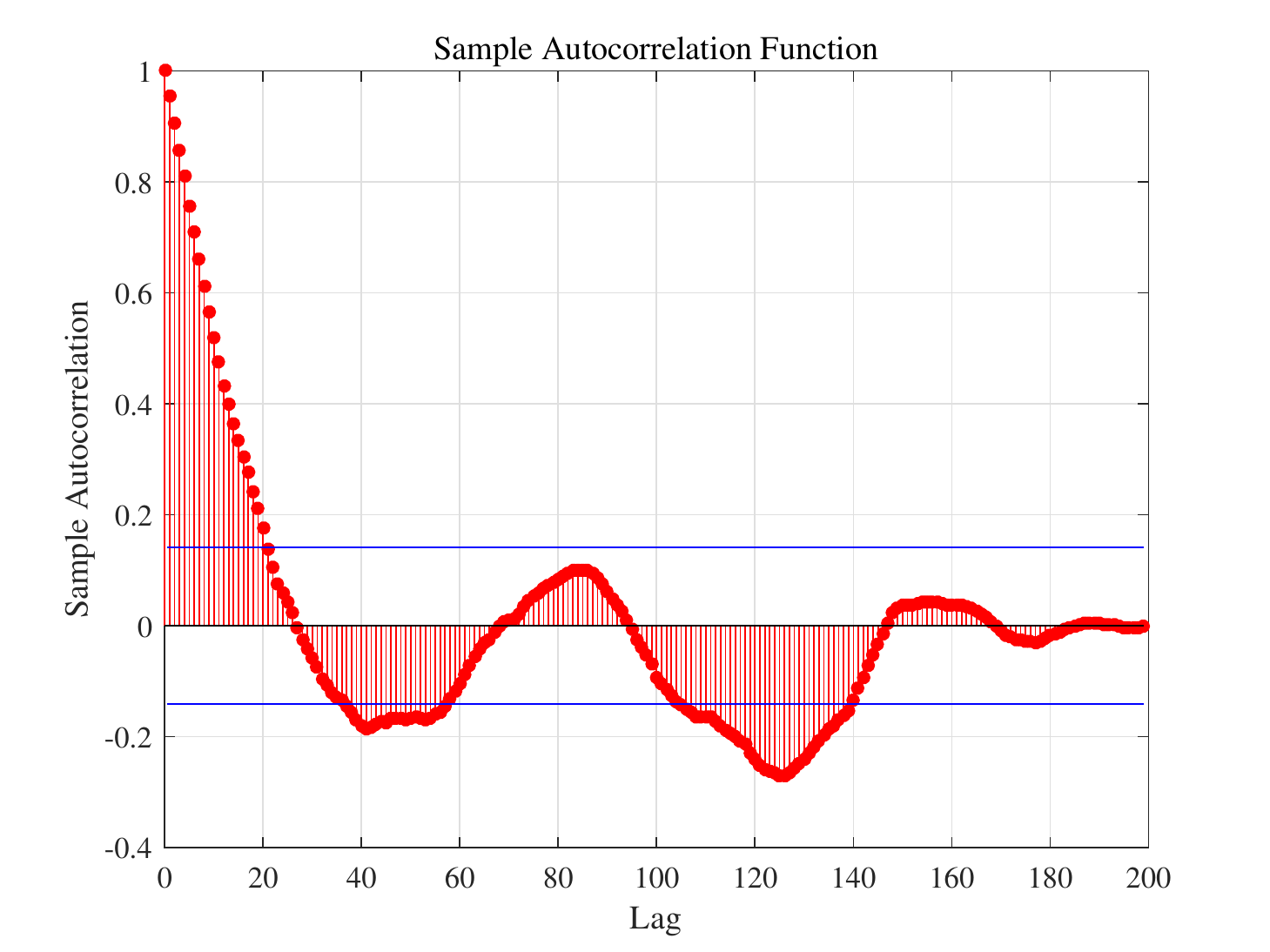}}
\subfigure[forecasting]{\includegraphics[scale=0.25]{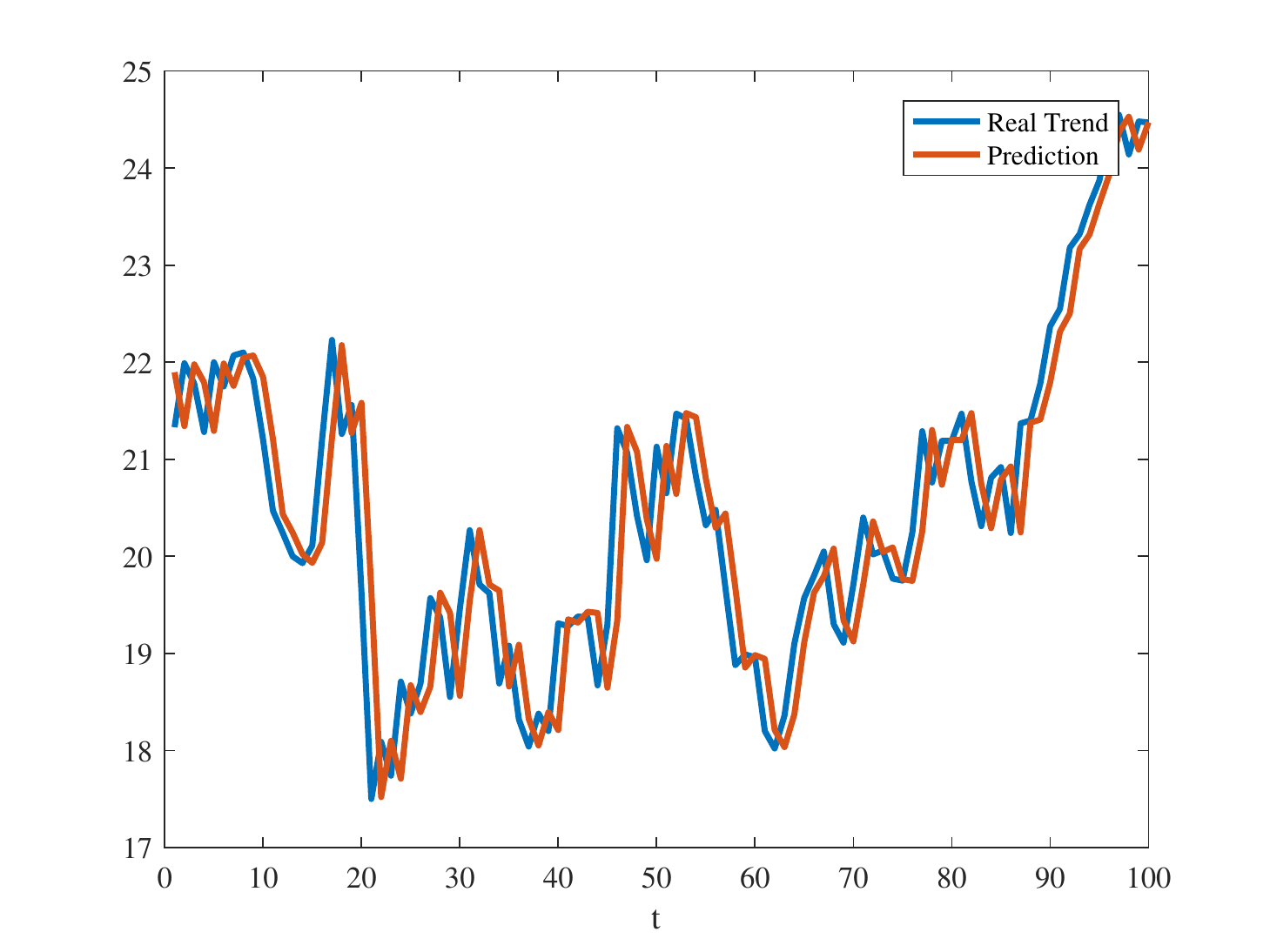}}
\subfigure[noise-extractor]{\includegraphics[scale=0.25]{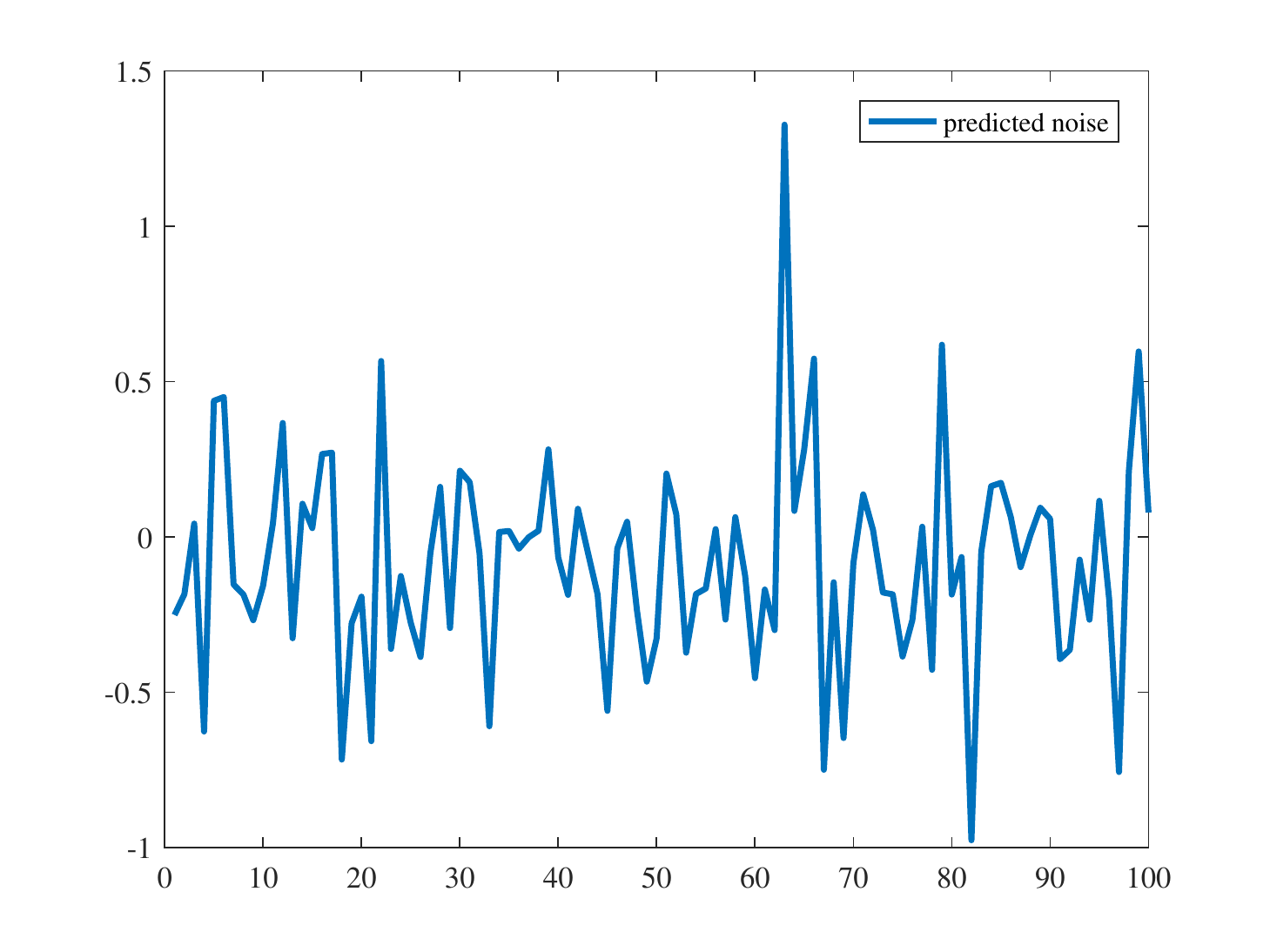}}
\subfigure[statistical bar]{\includegraphics[scale=0.25]{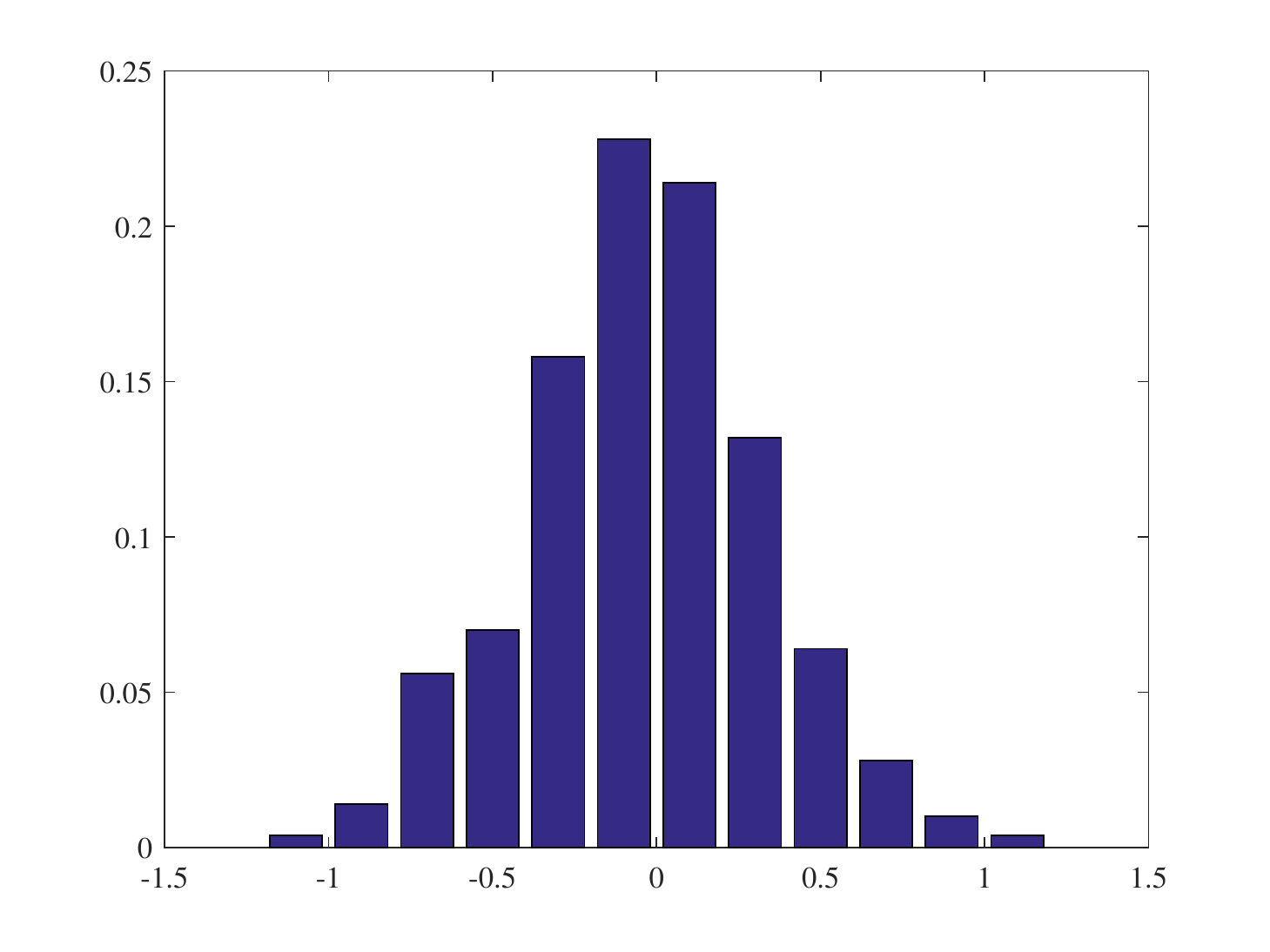}}
\\
\caption{Nystr\"{o}m regularization on WTI data}\label{WTI_fig}
\end{figure}

Before we demonstrating the availability  of Nystr\"{o}m regularization to WTI data, we present  the autocorrelation function (ACF) of the time series to verify their mixing property, just as  Figure \ref{WTI_fig} (a) purports to show. Although it is difficult to numerically compute the $\tau$-mixing coefficient of  WTI data, Figure \ref{WTI_fig} (a) implies that the dependence among time series decreases when the difference of time increases. We employ Nystr\"{o}m regularization on WTI data with the size of  samples and regularization parameters being set as:
\begin{itemize}
\item Forecasting: The number of training samples: 4000, test samples: 100, sampling ratio is 0.1, $\lambda=\frac{1}{1000}$.
\item Statistical bar: The number of training samples: 3000, test samples: 500,sampling ratio is 0.1, $\lambda=\frac{4}{3000}$.
\item Noise-extractor: The number of training samples: 3000, test samples: 100,sampling ratio is 0.1, $\lambda=\frac{4}{3000}$.
\end{itemize}
The experimental results are reported in Figure \ref{WTI_fig}.

There are three interesting findings illustrated by Figure \ref{WTI_fig}: 1)  It can be found in  Figure \ref{WTI_fig} (b) that the proposed Nystr\"{o}m regularization can roughly mimic the trend  of WTI data within a relatively small prediction error. However, due to the existence of noise, the predictor   lags slightly behind the real time series, making  the proposed algorithm a bit wise after the event; 2) As discussed in toy simulations, Nystr\"{o}m regularization can be used as a noise-extractor in practice. According to this property, we  extract the noise of WTI data in Figure \ref{WTI_fig} (b). Furthermore, Figure \ref{WTI_fig} (d) shows that the random  noise part of WTI data obeys the normal distribution; 3) Comparing Figure \ref{WTI_fig} (b) and Figure \ref{WTI_fig} (c), it is obvious that the random noise part is far smaller than the deterministic part, exhibiting a noise/signal ratio to be less than $1/15$, showing that exploiting the deterministic part of WTI data is important in practice. All these demonstrate  the excellent performance of Nystr\"{o}m regularization in WTI data forecasting.

%

\textbf{BITCOIN(BTC) data}: Bitcoin (BTC)  is a decentralized digital currency that can be transformed  from user to user on a peer-to-peer  BTC network by using the blockchain technique. After its release in 2009, its price increases from about 3\$ to 60000\$ in 2021. Such a   surge in price attracts more and more investors' attention. In particular, there are more  than 10 million active accounts in October, 2021. The problem is that, however, the price of BTC changes dramatically even for every minute, although the trends of price are roughly increasing. Therefore, it is highly desired to forecast the real trend of   price of BTC.

The BTC  data collected via ``https://www.kaggle.com/prasoonkottarathil/btcinusd'' record the price of BTC from September 17, 2014 to April 9, 2020 within each minute. Therefore, there are at least 500000 samples in the data set and the classical kernel methods and neural networks fail to tackle this massive time series. In this part, we focus on applying the proposed Nystr\"{o}m regularization on BTC data to mimic the trend of price and extract the noise. The data and parameters of Nystr\"{o}m regularization are described as follows:
\begin{itemize}
\item Forecasting: the number of training samples: 500000, test samples: 100, sampling ratio is 0.001,  $\lambda=\frac{1}{500000}$.
\item Noise-extractor: the number of training samples: 20000, test samples: 100, sampling ratio is 0.001, $\lambda=\frac{1}{20000}$.
\item   Statistical bar: the number of training samples: 20000, test samples: 3000, sampling ratio is 0.001, $\lambda=\frac{1}{20000}$.
\end{itemize}

\begin{figure}
\subfigure[ACF]{\includegraphics[scale=0.25]{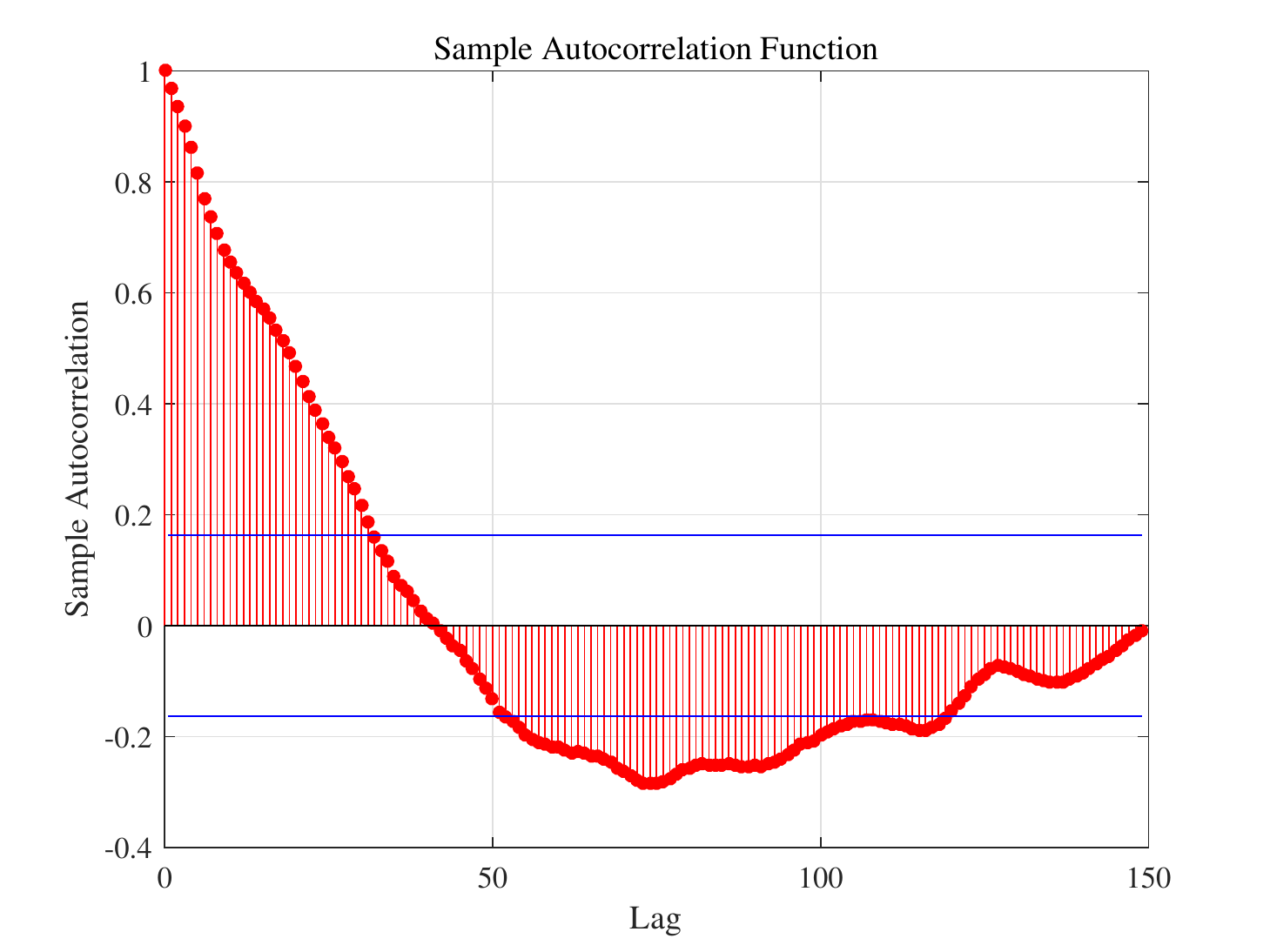}}
\subfigure[Forecasting]{\includegraphics[scale=0.25]{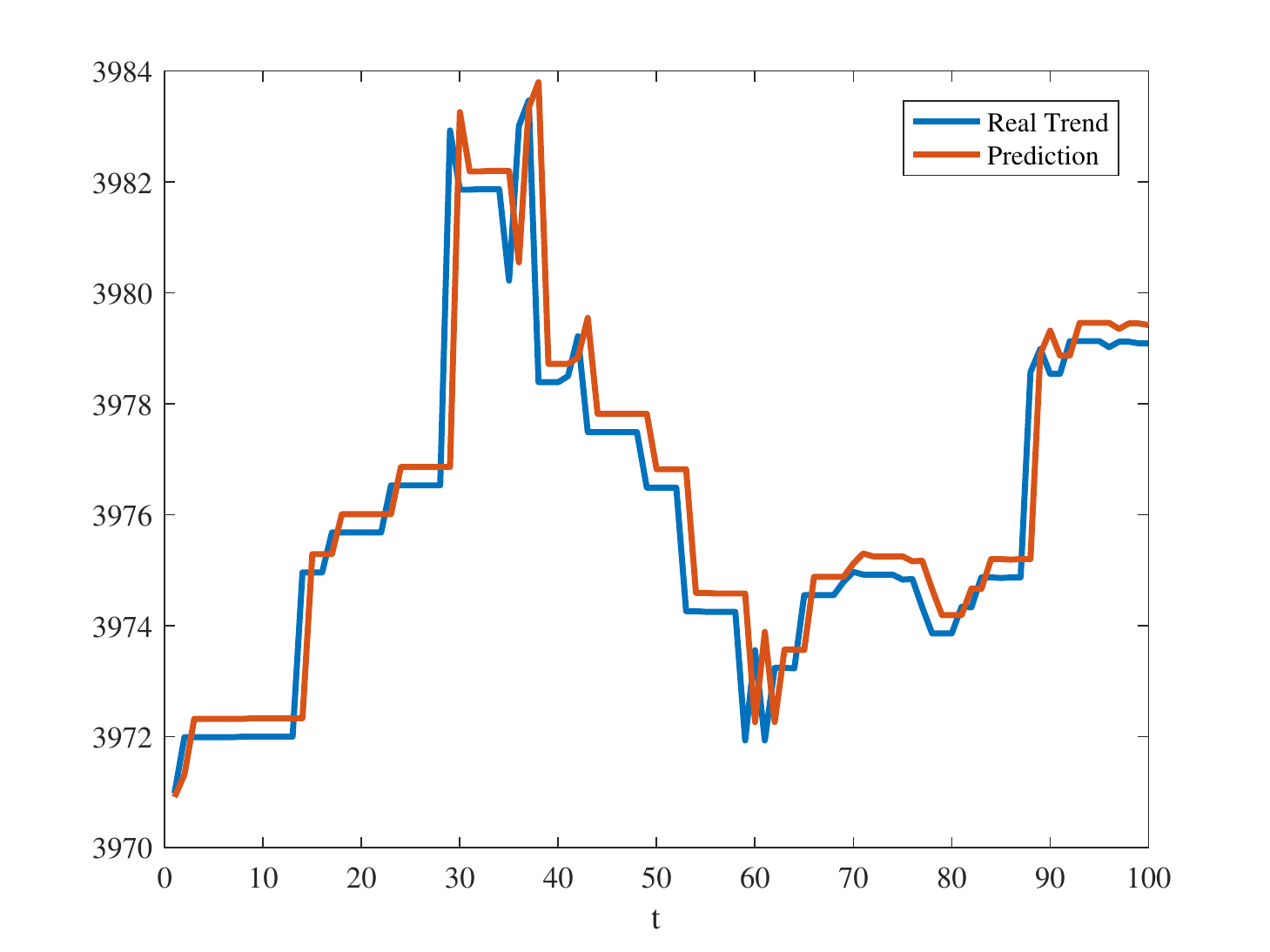}}
\subfigure[Noise-extractor]{\includegraphics[scale=0.25]{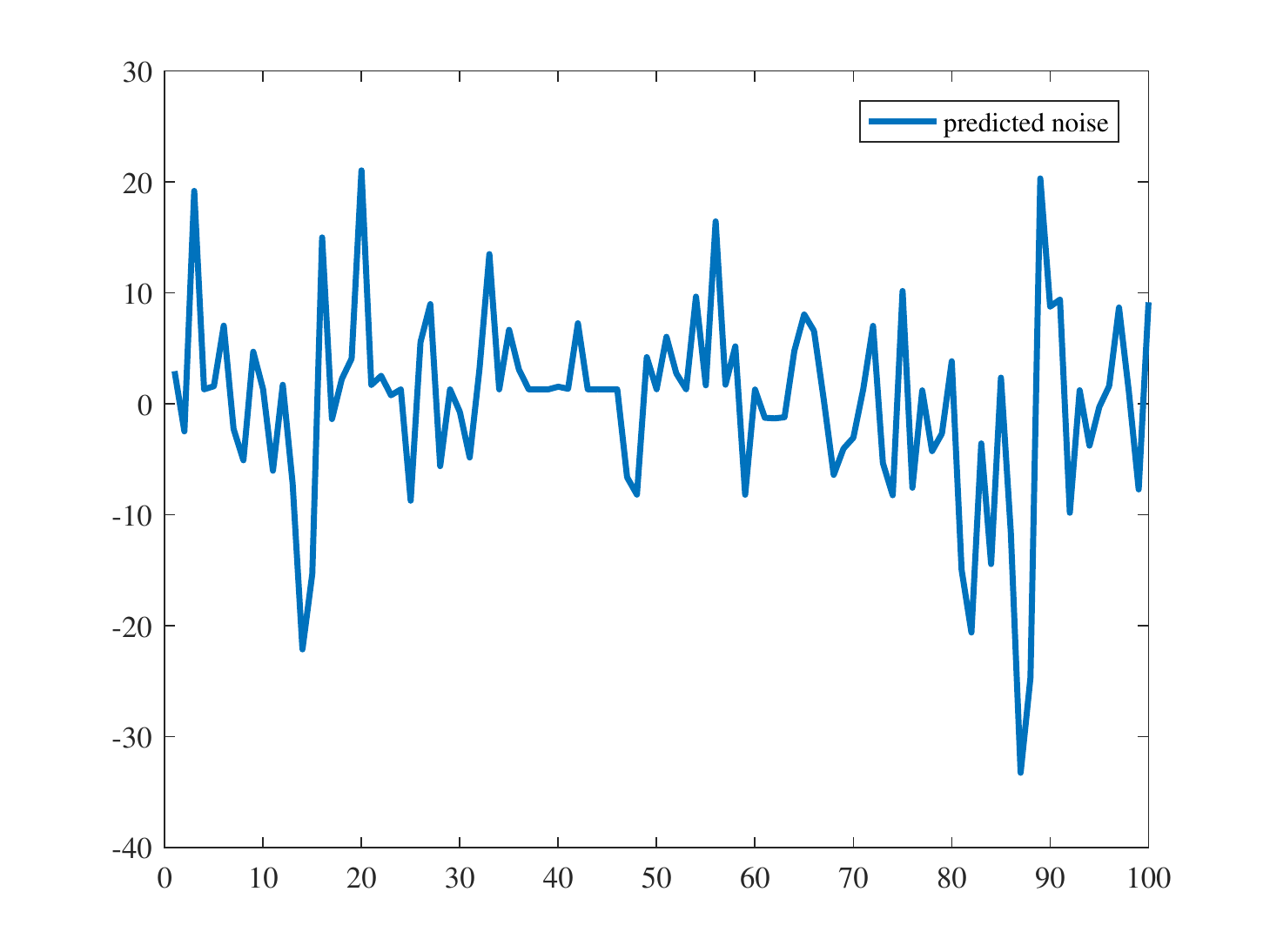}}
\subfigure[Statistical bar]{\includegraphics[scale=0.25]{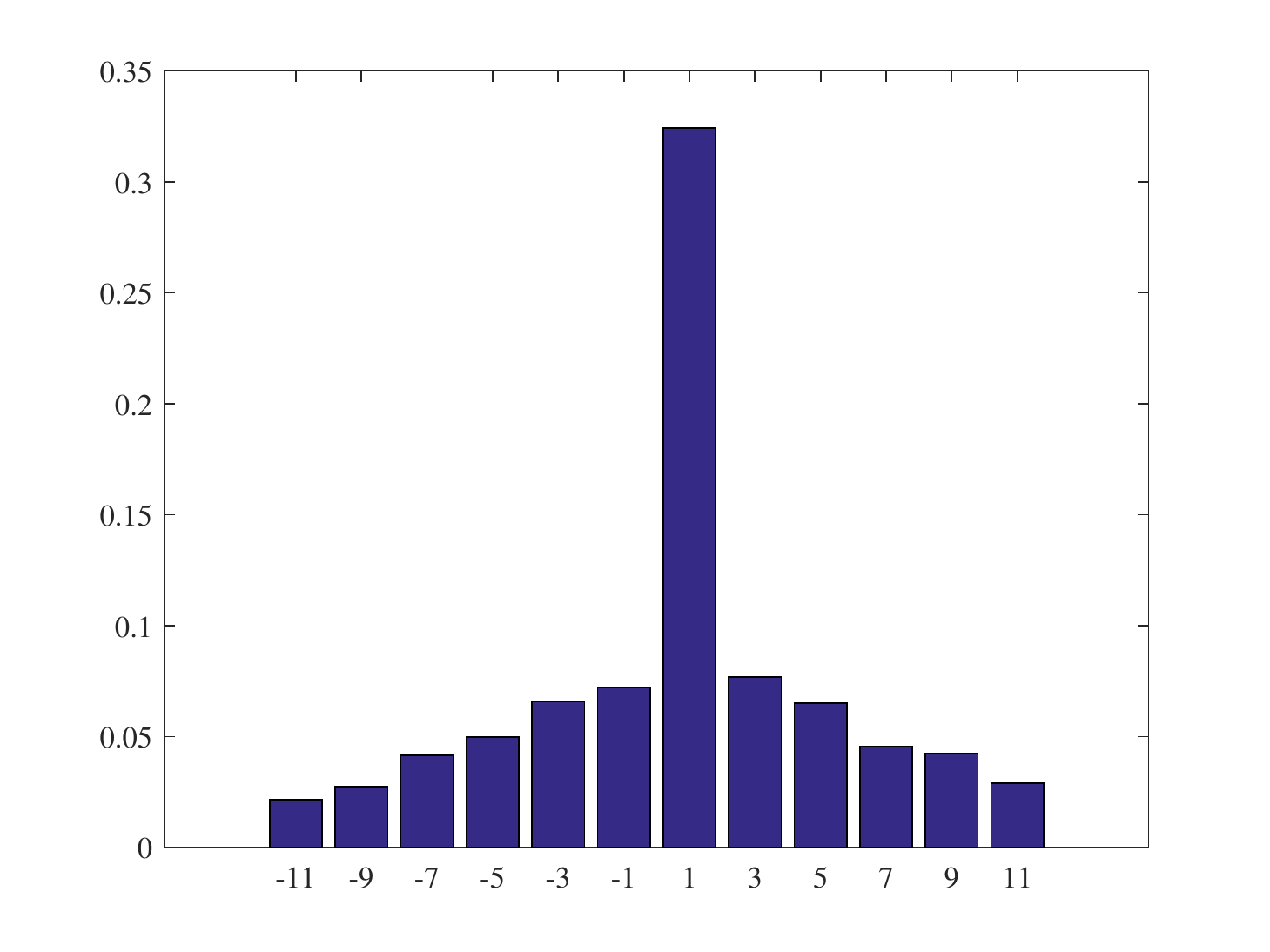}}
\caption{Nystr\"{o}m regularization for BTC min data}\label{BTC_min}
\end{figure}

The numerical  results are shown in Figure \ref{BTC_min}. As shown in Figure \ref{BTC_min} (a), the time series  concerning BTC price exhibits certain mixing property in the sense that the ACF curve decreases with respect to  the    difference of time. Like WTI data, it can be found in Figure \ref{BTC_min} (b) that the trends predicted by Nystr\"{o}m regularization can approximate the real trend within a relative small error, although there also exists a bit delay in prediction. As discussed above, the reason of delay is the existence of random noise that can be learned by BTC data. We then conduct our experiment on pursuing the distribution of such random noise.   Figure \ref{BTC_min} (c) and (d) show that the random noise part of BTC data behaves roughly as a normal distribution with standard deviation being smaller {than 30\$}.  In summary,
even with extremely small sub-sampling ratio (about 0.001), Nystr\"{o}m regularization can still predict the trend well and is capable of describing the distribution of noise for  Bitcoin price.
All these show the power of the suggested Nystr\"{o}m regularization with sequential sub-sampling in forecasting  time series and extracting its noise information.

\section{Proofs}\label{Sec.Proof}
We divide our proof into four steps: error decomposition, error estimates based on operator differences, estimate for operator differences and final  error analysis.

\subsection{Error decomposition}
We conduct our analysis in the popular integral operator approach developed in \citep{Smale2007,Rudi2015,Lin2017}.  Let $S_{D}:\mathcal H_K\rightarrow\mathbb R^{|D|}$ be the sampling
operator \citep{Smale2007} defined by
\begin{equation}\label{sampling-operator}
         S_{D}f:=(\langle f,K_{x}\rangle_K)_{(x,y)\in D}=(f(x))_{(x,y)\in D}.
\end{equation}
 Its scaled adjoint $S_D^T:\mathbb R^{|D|}\rightarrow
\mathcal H_K$  is
given by
$$
       S_{D}^T{\bf c}:=\frac1{|D|}\sum_{i=1}^{|D|}c_iK_{x_i},\qquad {\bf
       c}:=(c_1,c_2,\dots,c_{|D|})^T
       \in\mathbb
       R^{|D|}.
$$
Define
$$
         L_{K,D}f:=S_D^TS_Df=\frac1{|D|}\sum_{(x,y)\in
         D}f(x)K_x.
$$

In this proof only, we denote by $D_m$ an arbitrary sequential sub-sampling of $D$ with cardinality $m$.
Let $P_{D_m}$ be the projection from $\mathcal H_K$ to $\mathcal
H_{D_m}$, where $\mathcal H_{D_m}$ is defined by (\ref{hypothesis-space-sub}) with $D_j=D_m$.
Then for arbitrary $v>0$, there holds
\begin{equation}\label{projection-p-1}
      (I-P_{D_m})^v=I-P_{D_m}.
\end{equation}
Let $\mathcal S_D:\mathcal H_{D_m}\rightarrow \mathbb R^{|D|}$ be the sampling  operator defined by (\ref{sampling-operator}) such that the range of its adjoint  operator $\mathcal S^T_D$ is exactly $\mathcal H_{D_m}$. Let
$
       U\Sigma V^T
$
be the SVD of $\mathcal S_D$. Then we have
\begin{equation}\label{VD-property}
   V_{D_m}^TV_{D_m}=I,\qquad   V_{D_m}V_{D_m}^T=P_{D_m}.
\end{equation}

Write
\begin{equation}\label{spetral definetion}
     g_{D_m,\lambda}(L_{K,D}):=V_{D_m}(V_{D_m}^TL_{K,D}V_{D_m}+\lambda
     I)^{-1}V_{D_m}^T.
\end{equation}
Then it can be found in \citep{Rudi2015}
that
\begin{equation}\label{Nystrom operator}
       f_{D,D_m,\lambda}=g_{D_m,\lambda}(L_{K,D})S_D^Ty_D.
\end{equation}
Therefore, (\ref{Nystrom}) with $D_j=D_m$ is similar as the     spectral-type algorithms \citep{LoGerfo2008}. Under this circumstance, the property of $g_{D_m,\lambda}(L_{K,D})$ plays a crucial rule in bounding the generalization error of $f_{D,D_m,\lambda}$. For an arbitrary bounded linear operator $B$, it follows from  \eqref{VD-property} and \eqref{spetral definetion} that
\begin{eqnarray}\label{Important1}
     &&g_{D_m,\lambda}(L_{K,D})(L_{K,D}+\lambda
     I)V_{D_m}BV_{D_m}^T
      =
     V_{D_m}BV_{D_m}^T.
\end{eqnarray}
Inserting $B=(V^T_{D_m}L_{K,D}V_{D_m}+\lambda I)^{-1}$ into (\ref{Important1}), we obtain
\begin{eqnarray*}
      &&\|(L_{K,D}+\lambda
      I)^{1/2}g_{D_m,\lambda}(L_{K,D})(L_{K,D}+\lambda
      I)^{1/2}\|^2\\
      &=&
      \|(L_{K,D}+\lambda
      I)^{1/2}g_{D_m,\lambda}(L_{K,D})(L_{K,D}+\lambda I)
       {g_{D_m,\lambda}}(L_{K,D})
      (L_{K,D}+\lambda I)^{1/2}\|\\
      &=&
      \|(L_{K,D}+\lambda
      I)^{1/2}g_{D_m,\lambda}(L_{K,D})(L_{K,D}+\lambda I)^{1/2}\|,
\end{eqnarray*}
which yields
\begin{equation}\label{Important2}
    \|(L_{K,D}+\lambda
      I)^{1/2}g_{D_m,\lambda}(L_{K,D})(L_{K,D}+\lambda
      I)^{1/2}\|= 1,
\end{equation}
where $\|\cdot\|$ denotes the operator norm.

Define
$$
        f_{D,D_m,\lambda}^*:=g_{D_m,\lambda}(L_{K,D})L_{K,D}f_\rho
$$
be the noise-free version of $ f_{D,D_m,\lambda}$. The triangle inequality shows
\begin{equation}\label{error-d-1}
     \|f_{D,D_m,\lambda}-f_\rho\|_\rho
     \leq
     \|f_{D,D_m,\lambda}-f_{D,D_m,\lambda}^*\|_\rho
     +
     \|f_{D,D_m,\lambda}^*-f_\rho\|_\rho.
\end{equation}
Since $\|f_{D,D_m,\lambda}-f_{D,D_m,\lambda}^*\|_\rho$ describes the noise of samples, it  is named as the sample error in \citep{Rudi2015}.
From \eqref{spetral definetion}, we get
\begin{eqnarray*}
      \|f_{D,D_m,\lambda}^*-f_\rho\|_\rho
     &\leq& \|(g_{D_m,\lambda}(L_{K,D})L_{K,D}-I)(I-P_{D_m})f_\rho\|_\rho\\
     &+&\|(g_{D_m,\lambda}(L_{K,D})L_{K,D}-I)P_{D_m}f_\rho\|_\rho.
\end{eqnarray*}
It is easy to see that the second term in the righthand side of the above estimate is similar as the classical approximation error \citep{Guo2017}. The first term involves an additional term  $I-P_{D_m}$ to show the limitation of sub-sampling and thus is named as the computational error in \citep{Rudi2015}. Plugging the above inequality into \eqref{error-d-1}, we obtain
\begin{eqnarray}\label{Error decomposition}
   \|f_{D,D_m,\lambda}-f_\rho\|_\rho
   \leq
  \mathcal A(D,\lambda,m)+\mathcal S(D,\lambda,m)+\mathcal C(D,\lambda,m),
\end{eqnarray}
where
\begin{eqnarray*}
   \mathcal A(D,\lambda,m)&=&\|(g_{D_m,\lambda}(L_{K,D})L_{K,D}-I)P_{D_m}f_\rho\|_\rho,\\
   \mathcal S(D,\lambda,m)&=&\|f_{D,D_m,\lambda}-f_{D,D_m,\lambda}^*\|_\rho,\\ \mathcal C(D,\lambda,m)&=&\|(g_{D_m,\lambda}(L_{K,D})L_{K,D}-I)(I-P_{D_m})f_\rho\|_\rho
\end{eqnarray*}
are called the approximation error, sample error and computational error, respectively.

\subsection{Error estimates based on operator differences}
In  this part, we quantify $\|f_{D,D_m,\lambda}-f_\rho\|_\rho $ via differences between operators $L_{K,D}$ and $L_K$ and functions $S_D^Ty_D$ and $L_{K,D}f_\rho$. In particular, we use the products of operators
\begin{eqnarray*}
         \mathcal Q_{D,\lambda}&:=&\left\|(L_K + \lambda I) (L_{K,D}+\lambda
         I)^{-1}\right\|,\\
         \mathcal Q^*_{D,\lambda}&:=&\left\|(L_K + \lambda I)^{-1}(L_{K,D}+\lambda
         I) \right\|,
\end{eqnarray*}
and the difference of operators
$$
    \mathcal R_{D,\lambda}:=\left\|(L_K+\lambda
          I)^{-1/2}(L_{K,D}-L_K)\right\|
$$
to quantify the similarity of $L_K$ and $L_{K,D}$, while utilize
$$
         \mathcal P_{D,\lambda}:=
         \left\|(L_K+\lambda
          I)^{-1/2}(L_{K,D}f_\rho-S_D^Ty_D)\right\|_K
$$
to quantify the difference between $L_{K,D}f_\rho$ and $S_D^Ty_D$. Our main tools are the following two lemmas
that can be found in \cite[Proposition 3]{Rudi2015} and \cite[Proposition 6]{Rudi2015}, respectively.

\begin{lemma}\label{Lemma:Projection general}
Let $\mathcal H$, $\mathcal K$ , $\mathcal F$ be three separable
Hilbert spaces. Let $Z:\mathcal H\rightarrow\mathcal K$ be a bounded
linear operator and $P$ be a projection operator on $\mathcal H$
such that $\mbox{range}P=\overline{\mbox{range}Z^T}$. Then for any
bounded linear operator $F:\mathcal F\rightarrow\mathcal H$ and any
$\lambda>0$ we have
$$
     \|(I-P)F\|\leq\lambda^{1/2}\|(Z^TZ+\lambda I)^{-1/2}F\|.
$$
\end{lemma}

\begin{lemma}\label{Lemma:operator inequality general}
Let $\mathcal H,\mathcal K$ be two separable Hilbert spaces,
$A:\mathcal H\rightarrow\mathcal H$ be a positive linear operator,
$V:\mathcal H\rightarrow\mathcal K$ be a partial isometry and
$B:\mathcal K\rightarrow\mathcal K$ be a bounded operator. Then for
all $0\leq r^*,s^*\leq1/2$, there holds
$$
      \|A^{r^*}VBV^TA^{s^*}\|\leq\|(V^TAV)^{r^*}B(V^TAV)^{s^*}\|.
$$
\end{lemma}

With the help of  above lemmas and the important properties of $g_{D_m,\lambda}$ in \eqref{Important1} and \eqref{Important2}, we derive the following error estimate for $\|f_{D,D_m,\lambda}-f_\rho\|_\rho.$

\begin{proposition}\label{Proposition:Error decomposition detailed}
If (\ref{regularitycondition}) holds with $1/2\leq r\leq 1$, then we
have
\begin{eqnarray}\label{Error decomposition detailed}
     \|f_{D,D_m,\lambda}-f_\rho\|_\rho
    &\leq&
      \mathcal Q_{D,\lambda}\mathcal P_{D,\lambda}
      +
      \lambda^r\mathcal Q_{D,\lambda}^{r}\|h_\rho\|_\rho \nonumber\\
      &+&
      (\mathcal Q_{D,\lambda}^{\frac12}\mathcal (Q_{D,\lambda}^*)^{\frac12}+1)\lambda^{r}\mathcal
     Q_{D_m,\lambda}^{r}\|h_\rho\|_\rho.
\end{eqnarray}
\end{proposition}

\begin{proof}
According to \eqref{Error decomposition}, it suffices to bound $\mathcal A(D,\lambda,m), \mathcal S(D,\lambda,m),\mathcal C(D,\lambda,m)$ respectively. We at first use (\ref{Important1}) and (\ref{Important2}) to bound $\mathcal S(D,\lambda,m)$.
Due to the well known Codes inequality \citep{Bathia1997}
\begin{equation}\label{Codes inequality}
         \|A^u B^u\|\le\|AB\|^u, \qquad 0<u\leq 1
\end{equation}
for arbitrary positive operators $A,B$, we have
\begin{equation}\label{Codes-for-Q}
     \|(L_K+\lambda I)^{1/2}(L_{K,D}+\lambda I)^{-1/2}\|
     \leq\mathcal Q_{D,\lambda}^{1/2}.
\end{equation}
Then, it follows from   (\ref{Important2}) and $\|f\|_\rho=\|L_K^{1/2}f\|_K$ for any $f\in L_{\rho_X}^2$ that
\begin{eqnarray}\label{sample error estimate}
    &&\mathcal S(D,\lambda,m)
    =\|g_{D_m,\lambda}(L_{K,D})(S_D^Ty_D-L_{K,D}f_\rho)\|_\rho\nonumber\\
    &=&\|L_K^{1/2}g_{D_m,\lambda}(L_{K,D})(S_D^Ty_D-L_{K,D}f_\rho)\|_K\nonumber \\
    &\leq&
    \|(L_K+\lambda I)^{1/2}( {L_{K,D}}+\lambda I)^{-1/2}\|^2
    \|(L_{K,D}+\lambda I)^{1/2}g_{D_m,\lambda}(L_{K,D})(L_{K,D}+\lambda I)^{1/2}\|\nonumber\\
    &\times&
    \|(L_K+\lambda I)^{-1/2}
    (S_D^Ty_D-L_{K,D}f_\rho)\|_K\nonumber\\
    &\leq&
  {\mathcal Q_{D,\lambda}}\mathcal P_{D,\lambda}.
\end{eqnarray}
We then turn to bounding $\mathcal A(D,\lambda,m)$. It follows from   (\ref{Important1}) with $B=I$ that
$$
    P_{D_m}=g_{D_m,\lambda}(L_{K,D})(L_{K,D}+\lambda
    I)P_{D_m},
$$
which together with the definition of $\mathcal A(D,\lambda,m)$
$$
 \mathcal A(D,\lambda,m)=\lambda\| g_{D_m,\lambda}(L_{K,D}) P_{D_m} f_\rho\|_\rho.
$$
Since $L_{K,D}$ is a positive operator, we have $\|(V_{ {D_m}}^T(L_{K,D}+\lambda
     I)V_{D_m})^{r-1}\|\leq\lambda^{r-1}$ for $r\leq 1$.
Noting further (\ref{VD-property}) and (\ref{Codes inequality}), we get from {Lemma}
\ref{Lemma:operator inequality general} with $A=(L_{K,D}+\lambda
I)$, $V=V_{D_m}$, $B=(V_{D_m}^TL_{K,D}V_{D_m}+\lambda I)^{-1}$,
$r^*=1/2$ and $s^*=r-1/2$ that
\begin{eqnarray}\label{Approximation error estimate}
     &&\mathcal A(D,\lambda,m)
      \leq
     \lambda \|L_K^{1/2}g_{ {D_m},\lambda}(L_{K,D})V_{D_m}V_{D_m}^TL_K^{r-1/2}\|\|h_\rho\|_\rho \nonumber\\
     &\leq&
     \lambda\mathcal Q_{D,\lambda}^{r}\|h_\rho\|_\rho
     \|(L_{K,D}+\lambda I)^{1/2}g_{ {D_m},\lambda}(L_{K,D})V_{D_m}V_{D_m}^T(L_{K,D}+\lambda I)^{r-1/2}\| \nonumber\\
     &\leq&
     \lambda\mathcal Q_{D,\lambda}^{r}\big\|(V_{D_m}^T(L_{K,D}+\lambda
     I)V_{D_m})^{1/2}(V_{ {D_m}}^T(L_{K,D}+\lambda
     I)V_{D_m})^{-1}\nonumber\\
     &&(V_{D_m}^T(L_{K,D}+\lambda
     I)V_{D_m})^{r-1/2}\big\|\|h_\rho\|_\rho\nonumber\\
     &\leq&
      \lambda\mathcal Q_{D,\lambda}^{r}\|(V_{ {D_m}}^T(L_{K,D}+\lambda
     I)V_{D_m})^{r-1}\|\|h_\rho\|_\rho\nonumber\\
     &\leq&
     \lambda^r\mathcal Q_{D,\lambda}^{r}\|h_\rho\|_\rho.
\end{eqnarray}
Finally, we aim at bounding $\mathcal C(D,\lambda,m)$.
Due to Lemma \ref{Lemma:Projection general} and (\ref{Codes-for-Q}), we have
\begin{eqnarray*}
        &&\|(I-P_{D_m})(L_K+\lambda
        I)^{1/2}\|\\
        &\leq&
        \lambda^{1/2}\|(L_{K,D_m}+\lambda I)^{-1/2}(L_{K}+\lambda
        I)^{1/2}\|\leq \lambda^{1/2}\mathcal Q^{1/2}_{D_m,\lambda}.
\end{eqnarray*}
Then, it follows from (\ref{Codes inequality}), (\ref{Important2}) and \eqref{projection-p-1} with $\tau=2r$
that
\begin{eqnarray}\label{Computation error estimate}
     &&\mathcal C(D,\lambda,m)
      \leq
     \|L_K^{1/2}g_{D_m,\lambda}(L_{K,D})L_{K,D}(I-P_{D_m})L_K^{r-1/2}\|\|h_\rho\|_\rho\nonumber\\
     &+&\|L_K^{1/2}(I-P_{D_m})L_K^{r-1/2}\|\|h_\rho\|_\rho\nonumber\\
     &\leq&
     \mathcal Q^{1/2}_{D,\lambda}(\mathcal Q_{D,\lambda}^*)^{1/2}\|(L_{K,D}+\lambda I)^{1/2}g_{D_m,\lambda}(L_{K,D})(L_{K,D}+\lambda
     I)^{1/2}\|\nonumber\\
     &\times&\|(L_{K}+\lambda
     I)^{1/2}(I-P_{D_m})^{2r}L_K^{r-1/2}\|
      \|h_\rho\|_\rho\nonumber\\
      &+&
      \|(L_{K}+\lambda
     I)^{1/2}(I-P_{D_m})^{2r}L_K^{r-1/2}\|
      \|h_\rho\|_\rho \nonumber\\
     &\leq&
       \big(\mathcal Q^{1/2}_{D,\lambda}(\mathcal Q_{D,\lambda}^*)^{1/2}+1\big)\|h_\rho\|_\rho
       \|(L_K+\lambda I)^{1/2}(I-P_{D_m})\|\|(I-P_{D_m})^{2r-1}
     L_K^{r-1/2}\|\nonumber\\
     &\leq&
     \big(\mathcal Q^{1/2}_{D,\lambda}(\mathcal Q_{D,\lambda}^*)^{1/2}+1\big)\|h_\rho\|_\rho\lambda^{r}\mathcal
     Q_{D_m,\lambda}^{r}.
\end{eqnarray}
Inserting (\ref{Approximation error estimate}), (\ref{sample error
estimate}) and (\ref{Computation error estimate}) into (\ref{Error
decomposition}), we get \eqref{Error decomposition detailed}.
This finishes the proof of Proposition \ref{Proposition:Error
decomposition detailed}.
\end{proof}

\subsection{Bounds for operator differences}
In this part, we focus on deriving tight bounds for $\mathcal Q_{D,\lambda}$, $\mathcal Q_{D,\lambda}^*$ and $\mathcal P_{D,\lambda}$ when $D$ is a $\tau$-mixing sequences. Our main tool is the Bernstein-type inequality for Banach-valued sums in \citep{Blanchard2019} and the second-order decomposition for operator differences in \citep{Lin2017,Guo2017}. Under Assumptions \ref{Assumption:kernel} and \ref{Assumption:mixing}, define
\begin{equation}\label{Effective-sample}
     n_\gamma:=\left\{\begin{array}{cc} \frac{n}{ 2(1\vee \log(c_1n))^{1/\gamma_0}}, & \mbox{if}\  D \ \mbox{satisfies} \eqref{def-Galpha},\\
       (c_2\lambda\mathcal N(\lambda))^\frac1{2\gamma_1+1}n^{\frac{2\gamma_1}{2\gamma_1+1}}, &
      \mbox{if} \ D \ \mbox{satisfies} \eqref{def-Galpha11},
      \end{array}
      \right.
\end{equation}
where
\begin{eqnarray*}
 c_1&:=&c_0^*b_0\max\{\mathcal K, 3(1+\mathcal KM)/(2M)\},\\
 c_2&:=&2^{-\frac{2\gamma_1}{2\gamma_1+1}}\left(\min\left\{2\sqrt{2}M/(1+\mathcal KM),1/(2\mathcal K)\right\}/c_0^*\right)^{\frac{2}{2\gamma_1+1}}.
\end{eqnarray*}
The following lemma can be found in \cite[Lemma 4.1]{Blanchard2019}.

\begin{lemma}\label{Lemma:operator-differences}
Let $0<\delta\leq1/2$. Under Assumptions 1 and 2, then
\begin{eqnarray}
     \mathcal P_{D,\lambda}&\leq& 21(\sqrt{2}M+2M)\mathcal B(n_\gamma,\lambda)\log\frac{2}{\delta},\label{operator-diff-1} \\
     \mathcal R_{D,\lambda} &\leq&
     42 \mathcal B(n_\gamma,\lambda)\log\frac{2}{\delta},\label{operator-diff-2}
\end{eqnarray}
hold simultaneously  with confidence $1-\delta$, where
\begin{equation}\label{Def.B}
    \mathcal B(n_\gamma,\lambda):=\frac{\sqrt{\mathcal N(\lambda)}}{\sqrt{n_\gamma}}+\frac1{n_\gamma\sqrt{\lambda}}.
\end{equation}
\end{lemma}

Then, we use Lemma \ref{Lemma:operator-differences} and approaches in \citep{Guo2017} to bound operator products $\mathcal Q_{D,\lambda}$ and $\mathcal Q_{D,\lambda}^*$.

\begin{lemma}\label{Lemma:operator-products}
Let $0<\delta\leq1/2$. Under Assumptions 1 and 2, then
\begin{eqnarray}
     \mathcal Q_{D,\lambda}&\leq&   {3528} \lambda^{-1}  \mathcal B^2(n_\gamma,\lambda)\log^2\frac{2}{\delta}+2,\label{bound-Q-1} \\
     \mathcal Q^*_{D,\lambda} &\leq& 42 \lambda^{-1/2}\mathcal B(n_\gamma,\lambda)\log\frac{2}{\delta}+1
      ,\label{bound-Q-2}
\end{eqnarray}
hold with confidence $1-\delta$.
\end{lemma}

\begin{proof}
For invertible positive operators $A,B$, we have
\begin{eqnarray}
    A^{-1}B&=&(A^{-1}-B^{-1})B+I
    = A^{-1}(B-A)+I \label{first-order}\\
    &=&(A^{-1}-B^{-1})(B-A)+B^{-1}(B-A)+I   \nonumber\\
    &=&A^{-1}(B-A)B^{-1}(B-A)+B^{-1}(B-A)+I. \label{second-order}
\end{eqnarray}
We first use (\ref{first-order}) with $A=(L_K+\lambda I)$ and $B=(L_{K,D}+\lambda I)$ to bound $\mathcal Q_{D,\lambda}^*$.
It follows from \eqref{operator-diff-2} that with confidence $1-\delta$, that holds
\begin{eqnarray*}
  &&\mathcal Q_{D,\lambda}^*
  =
  \|(L_K+\lambda I)^{-1}(L_{K,D}-L_K)\| +1
  \leq
  \lambda^{-1/2}\mathcal R_{D,\lambda}+1\\
  &\leq& 42 \lambda^{-1/2}\mathcal B(n_\gamma,\lambda)\log\frac{2}{\delta}+1.
\end{eqnarray*}
Therefore (\ref{bound-Q-2}) holds. Then, due to \eqref{second-order} with
$A=(L_{K,D}+\lambda I)$ and $B=(L_K+\lambda I)$, we have from  \eqref{operator-diff-2} again that with confidence $1-\delta$, there holds
\begin{eqnarray*}
    &&\mathcal Q_{D,\lambda}
    \leq
    \|(L_{K,D}+\lambda I)^{-1}(L_K-L_{K,D})(L_K+\lambda I)^{-1}(L_K-L_{K,D})\|\\
    &+&
    \|(L_K+\lambda I)^{-1}(L_K-L_{K,D})\|+1\\
    &\leq&
    \lambda^{-1}\|(L_K+\lambda I)^{-1/2}(L_K-L_{K,D})\|^2
    +\lambda^{-1/2}\|(L_K+\lambda I)^{-1/2}(L_K-L_{K,D})\|
    +1\\
    &\leq&
    2(\lambda^{-1}\mathcal R_{D,\lambda}^2+1)
    \leq
    {3528} \lambda^{-1}  \mathcal B^2(n_\gamma,\lambda)\log^2\frac{2}{\delta}+2.
\end{eqnarray*}
This completes the proof of Lemma \ref{Lemma:operator-products}.
\end{proof}

\subsection{Error analysis}

Based on Lemma \ref{Lemma:operator-differences}, Lemma \ref{Lemma:operator-products} and Proposition \ref{Proposition:Error decomposition detailed}, we are in a position to prove main results in Section \ref{Sec.Mainresult}. To this end,  we present a more general theorem as follows.

\begin{theorem}\label{Theorem:boundwith-effective}
Let $0<\delta\leq 1/2$. Under Assumptions \ref{Assumption:kernel}-\ref{Assumption:regularity} with $\frac12\leq r\leq1$, for any $j\in[1,n-m+1]$, with confidence $1-\delta$,  there holds
\begin{eqnarray}\label{Error-analysis-1}
    \|f_{D,D_m,\lambda}-f_\rho\|_\rho
    \leq
     \bar{C}(\lambda^{-1}  \mathcal B^2(n_\gamma,\lambda)+1)
      \left(\mathcal B(n_\gamma,\lambda)+    \mathcal B^{2r }(m_\gamma,\lambda)+\lambda^r\right) {\log^4\frac{2}{\delta}},
\end{eqnarray}
where
\begin{equation}\label{Effective-sample-for-m}
     m_\gamma:=\left\{\begin{array}{cc} \frac{m}{ 2(1\vee \log(c_1m))^{1/\gamma_0}}, & \mbox{if}\  D \ \mbox{satisfies} \eqref{def-Galpha},\\
      (c_2\lambda\mathcal N(\lambda))^\frac1{2\gamma_1+1}m^{\frac{2\gamma_1}{2\gamma_1+1}}, &
      \mbox{if} \ D \ \mbox{satisfies} \eqref{def-Galpha11},
      \end{array}
      \right.
\end{equation}
 and $\bar{C}$ is a constant independent of $m,n,\lambda,j$ or $\delta$.
\end{theorem}

\begin{proof}
Due to Lemma \ref{Lemma:operator-differences}  and Lemma \ref{Lemma:operator-products}, we have from $(a+b)^u\leq 2^u(a^u+b^u)$ for $a,b,u\geq0$   that
\begin{eqnarray*}
    \mathcal P_{D,\lambda}\mathcal Q_{D,\lambda}
     &\leq& C_1  \mathcal B(n_\gamma,\lambda)
    ( {42^2}\lambda^{-1}  \mathcal B^2(n_\gamma,\lambda)+1)\log^3\frac{2}{\delta},\\
    \mathcal Q_{D,\lambda}^{\frac12}\mathcal (Q_{D,\lambda}^*)^{\frac12}
    &\leq&
    2( {42} {\lambda^{-1/2} }\mathcal B(n_\gamma,\lambda)+1)^{3/2}\log^2\frac2\delta,\\
    \mathcal Q^r_{D,\lambda}  &\leq&
    ( {3528}^r \lambda^{-r}  \mathcal B^{2r }(n_\gamma,\lambda)\log^{2r}\frac{2}{\delta}+2^r) {2^r},\\
    \mathcal Q^r_{D_m,\lambda}  &\leq&
    ( {3528}^r \lambda^{-r}  \mathcal B^{2r }(m_\gamma,\lambda)\log^{2r}\frac{2}{\delta}+2^r) {2^r},
\end{eqnarray*}
where
$C_1:=42(\sqrt{2}M+2M)$. Plugging the above estimates into \eqref{Error decomposition detailed}, we have
\begin{eqnarray*}
     &&\|f_{D,D_m,\lambda}-f_\rho\|_\rho
     \leq
      C_1  \mathcal B(n_\gamma,\lambda)
    ( {42^2}\lambda^{-1}  \mathcal B^2(n_\gamma,\lambda)+1)\log^3\frac{2}{\delta}\\
      &+&
      \lambda^r\|h_\rho\|_\rho\left( {3528}^r \lambda^{-r}  \mathcal B^{2r }(n_\gamma,\lambda)\log^{2r}\frac{2}{\delta}+2^r\right) {2^r},\\
      &+&
      \lambda^r \|h_\rho\|_\rho  \left( {3528}^r \lambda^{-r}  \mathcal B^{2r }(m_\gamma,\lambda)\log^{2r}\frac{2}{\delta}+2^r\right) {2^r}\\
      &\times&
      \left(2( {42} {\lambda^{-1/2}} \mathcal B(n_\gamma,\lambda)+1)^{3/2}\log^2\frac2\delta+1\right).
\end{eqnarray*}
Due to \eqref{Def.B}, we have $\mathcal B(\ell,\lambda)$ is monotonously decreasing with respect to $\ell$. Therefore, (\ref{Effective-sample}) and \eqref{Effective-sample-for-m}  yield
$\mathcal B(n_\gamma,\lambda)\leq \mathcal B(m_\gamma,\lambda)$. Inserting it into the above estimate, we obtain from $1/2\leq r\leq 1$ that
\begin{eqnarray*}
   &&\|f_{D,D_m,\lambda}-f_\rho\|_\rho
      \leq
      C_1  \mathcal B(n_\gamma,\lambda)
    ( {42^2}\lambda^{-1}  \mathcal B^2(n_\gamma,\lambda)+1)\log^3\frac{2}{\delta}\\
    &+&
    \lambda^r \|h_\rho\|_\rho  \left( {3528}^r \lambda^{-r}  \mathcal B^{2r }(m_\gamma,\lambda)\log^{2r}\frac{2}{\delta}+2^r\right)
      \left(2( {42}\mathcal B(n_\gamma,\lambda) {\lambda^{-1/2}+1})^{3/2}\log^2\frac2\delta {+2}\right){2^r}\\
      &\leq&
      C_2( {\mathcal B(n_\gamma,\lambda)\lambda^{-1/2}+1)^{3/2}}+1)
      \left(\mathcal B(n_\gamma,\lambda)+    \mathcal B^{2r }(m_\gamma,\lambda)+\lambda^r\right)  {\log^4}\frac{2}{\delta},
\end{eqnarray*}
where $C_2:=\max\{ {1764}C_1, {7056^r\times 545}\|h_\rho\|_\rho\}.$ This completes the proof of Theorem \ref{Theorem:boundwith-effective} with $\bar{C}=C_2$.
\end{proof}

With the help of Theorem \ref{Theorem:boundwith-effective}, we can prove our main results easily.

\begin{proof}[Proof of Theorem \ref{Theorem:error-for-exp}]
Since $\lambda=n_\gamma^{-\frac{1}{2r+s}}$, it follows  from (\ref{Def.B})
and Assumption 4 with $0<s\leq 1$ and $1/2\leq r\leq 1$ that
\begin{eqnarray*}
    \mathcal B(n_\gamma,\lambda)\leq\frac{\sqrt{\mathcal N(\lambda)}}{\sqrt{n_\gamma}}+\frac1{n_\gamma\sqrt{\lambda}}
    \leq(\sqrt{C_0}+1)n_\gamma^{-\frac{r}{2r+s}}
\end{eqnarray*}
and
\begin{eqnarray*}
    \mathcal B(m_\gamma,\lambda)\leq\frac{\sqrt{\mathcal N(\lambda)}}{\sqrt{m_\gamma}}+\frac1{m_\gamma\sqrt{\lambda}}
    \leq \sqrt{C_0} n_\gamma^\frac{s}{4r+2s}m_\gamma^{-\frac12}+m_\gamma^{-1}n_\gamma^\frac{1}{4r+2s}.
\end{eqnarray*}
But Assumption 2 holds with $\tau_j$ satisfying (\ref{def-Galpha}), which together with (\ref{Effective-sample}) and (\ref{Effective-sample-for-m})
yields
\begin{eqnarray}\label{Bound-Bn-exp}
   \mathcal B(n_\gamma,\lambda)\leq C_3n^{-\frac{r}{2r+s} }( {1+\log (n)})^{\frac{r}{(2r+s)\gamma_0}},
\end{eqnarray}
where
$C_3:=(\sqrt{C_0}+1)(2+2\log c_1)^{r/(2r+s)}$. Furthermore, (\ref{bound-on-m}) implies
$m_\gamma\geq C_4n_\gamma^{\frac{s+1}{2r+s}}$ with $C_4:=(2+2\log c_1)^\frac{ {1-2r}}{(2r+s)\gamma_0}.$ Then,
we have
\begin{eqnarray}\label{Bound-Bm-exp}
   \mathcal B^{2r}(m_\gamma,\lambda)\leq C_5n^{-\frac{r}{2r+s} }( {1+ \log (n)})^{\frac{r}{(2r+s)\gamma_0}},
\end{eqnarray}
where $C_5:=C_4^{-r}C_3$.
Since $r>1/2$, there is a constant $C_7>0$ depending only on $r,s,\gamma_0$ and $C_5$ such that
$$
        n^{ {(1-2r)/2(2r+s)}}( {1+ \log (n)})^{ {\frac{2r-1}{2(2r+s)\gamma_0}}}\leq C_7.
$$
Then,
plugging (\ref{Bound-Bn-exp}) and (\ref{Bound-Bm-exp}) into \eqref{Error-analysis-1}, we obtain that with confidence $1-\delta$, there holds
$$
    \|f_{D,D_j,\lambda}-f_\rho\|_\rho
     \leq
     C^* n^{-\frac{r}{2r+s} }( {1+ \log (n)})^{\frac{r}{(2r+s)\gamma_0}} {\log^4\frac3\delta},
$$
where $C^*:=\bar{C}  {(C_3C_7+1)^{\frac{3}{2}}}(C_3+C_5+1).$ This completes the proof of Theorem \ref{Theorem:error-for-exp}.
\end{proof}

\begin{proof}[Proof of Theorem \ref{Theorem:error-for-alg}]
Due to Assumption 4, (\ref{def-Galpha11}), \eqref{Effective-sample} and \eqref{Effective-sample-for-m} yield
\begin{eqnarray*}
   n_\gamma  \leq
   C_8 \lambda^\frac{1-s}{2\gamma_1+1}n^{\frac{2\gamma_1}{2\gamma_1+1}}, \qquad \mbox{and}\
   m_\gamma  \leq  C_8 \lambda^\frac{1-s}{2\gamma_1+1}m^{\frac{2\gamma_1}{2\gamma_1+1}},
\end{eqnarray*}
where $C_8:=(c_2C_0)^{{\frac{1}{2\gamma_1+1}}}$.
Then (\ref{Def.B}), (\ref{bound-on-m-alg}) and $\lambda=n^{-\frac{2\gamma_1}{2\gamma_1(2r+s)+2r+1}}$ shows
\begin{eqnarray*}
    &&\mathcal B(n_\gamma,\lambda)
     \leq C_0{^{1/2}}C_8^{-1/2}\lambda^{-\frac{2s\gamma_1+1}{4\gamma_1+2}}n^{-\frac{\gamma_1}{2\gamma_1+1}}
    +C_8^{-1}{\lambda^{-\frac{2\gamma_1-2s+3}{4\gamma_1+2}}}n^{-\frac{2\gamma_1}{2\gamma_1+1}} \\
    &\leq&
    (C_0{^{1/2}} C_8^{-1/2}+C_8^{-1})n^{-\frac{2\gamma_1r}{2\gamma_1(2r+s)+2r+1}},
\end{eqnarray*}
and
\begin{eqnarray*}
    &&\mathcal B(m_\gamma,\lambda)
     \leq
    C_0{^{1/2}}C_8^{-1/2}\lambda^{-\frac{2s\gamma_1+1}{4\gamma_1+2}}m^{-\frac{\gamma_1}{2\gamma_1+1}}
    +C_8^{-1} {\lambda^{-\frac{2\gamma_1-2s+3}{4\gamma_1+2}}}m^{-\frac{2\gamma_1}{2\gamma_1+1}}\\
    &\leq & (C_0{^{1/2}}C_8^{-1/2}+C_8^{-1})n^{-\frac{\gamma_1}{2\gamma_1(2r+s)+2r+1}}.
\end{eqnarray*}
Inserting the above two estimates into \eqref{Error-analysis-1}, we obtain that
$$
    \|f_{D,D_j,\lambda}-f_\rho\|_\rho
     \leq
     \hat{C} n^{-\frac{2\gamma_1r}{2\gamma_1(2r+s)+2r+1}} {{\log^4\frac2\delta}}
$$
holds with confidence $1-\delta$, where $\hat{C}:=2C_2(C_0{^{1/2}}C_8^{-1/2}+C_8^{-1}+1)^3.$ This completes the proof of Theorem \ref{Theorem:error-for-alg}.
\end{proof}

\section*{Acknowledge} {The work   is supported partially by   the  National Key R\&D Program of China (No.2020YFA0713900) and
   the National Natural Science Foundation of China
(Nos.618761332,11971374).

\end{document}